\documentclass[a4paper]{article}

\usepackage{amsmath}
\usepackage{amssymb}
\usepackage{amsthm}
\usepackage[english]{babel}
\usepackage{caption} 
\usepackage{color}
\usepackage{enumitem} 
\usepackage{framed}
\usepackage{graphicx}
\usepackage{hyperref}
\usepackage{natbib} 
\usepackage{stmaryrd} 
\usepackage{subcaption} 

\usepackage{algpseudocode}
\usepackage{algorithm}

\renewcommand{\cite}[1]{\citep{#1}} 
\newcommand{\Qrefs}{\cite{watkins_1989,watkins-dayan_1992}}

\newcommand{\ramp}{Value-Ramp}

\numberwithin{table}{section}
\numberwithin{figure}{section}
\numberwithin{equation}{section}

\newcommand{\proofapp}[1]{\textit{(Proof in Appendix~\ref{#1}.)}\qed}
\newcommand{\proofmain}[1]{The proof is given in Section~\ref{#1}.}
\newcommand{\propproof}[1]{Proof of Property~\ref{#1}}

\theoremstyle{definition}
\newtheorem{definition}{Definition}[section]
\newtheorem{theorem}{Theorem}[section]
\newtheorem{property}[theorem]{Property}
\newtheorem{example}[theorem]{Example}
\newtheorem{lemma}[theorem]{Lemma}
\newtheorem{corollary}[theorem]{Corollary}
\newtheorem{remark}[theorem]{Remark}
\newtheorem{claim}[theorem]{Claim}
\newtheorem*{theoremexplore}{Exploration}
\newtheorem*{theoremgreedy}{Greediness}

\newcommand{\fname}[1]{\text{\textit{#1}}} 
\renewcommand{\qed}{\hfill$\square$} 
\newcommand{\X}[2]{#1_{#2}}

\newcommand{\nat}{\mathbb{N}}
\newcommand{\ints}{\mathbb{Z}}
\newcommand{\set}[1]{\left\{#1\right\}}
\newcommand{\ssize}[1]{\left|#1\right|}
\newcommand{\length}[1]{\left|#1\right|}
\newcommand{\powset}[1]{\mathcal{P}(#1)}
\newcommand{\dom}[1]{\fname{dom}\left(#1\right)} 

\newcommand{\V}{V} 
\newcommand{\VX}[1]{\X{\V}{#1}}
\newcommand{\get}[1]{[#1]} 
\newcommand{\step}{K}
\newcommand{\pref}[2]{\fname{pref}(#1,#2)}
\newcommand{\rl}{\Delta_{\step}}

\newcommand{\jump}[1]{{}\xrightarrow{#1}}
\newcommand{\cnf}{c}
\newcommand{\cnftup}{(\st,\V)}

\newcommand{\cnftupX}[1]{(\stX{#1},\VX{#1})}
\newcommand{\transX}[1]{\cnftupX{#1}\jump{\actX{#1},\,\stX{#1+1}}\cnftupX{#1+1}}

\newcommand{\run}{\mathcal{X}}

\newcommand{\clamp}[1]{\left\llbracket#1\right\rrbracket}

\newcommand{\edge}[1]{\xrightarrow{#1}}

\newcommand{\actions}{A}
\newcommand{\act}{a}
\newcommand{\actX}[1]{\X{\act}{#1}}

\newcommand{\task}{T}
\newcommand{\tasktup}{(\states,\startstates,\actions,\tr,\R)}
\newcommand{\states}{S}
\newcommand{\startstates}{S^\mathrm{start}}
\newcommand{\tr}{\delta} 

\newcommand{\st}{s}
\newcommand{\stX}[1]{\X{\st}{#1}}

\newcommand{\R}{R} 

\newcommand{\ceil}[1]{\fname{ceiling}(#1)}

\newcommand{\Det}[3]{\tr(#1,#2)=\set{#3}} 

\newcommand{\DC}{DC}

\newcommand{\pval}[1]{\fname{val}(#1)} 
\newcommand{\optval}[1]{\fname{opt-val}(#1)}
\newcommand{\explore}[1]{\fname{explore}(#1)}

\newcommand{\M}{M}
\newcommand{\plen}[1]{\fname{len}(#1)}

\newcommand{\sub}[1]{\fname{sub}(#1)} 
\newcommand{\alt}[1]{#1^*}

\newcommand{\viol}[1]{\fname{viol}(#1)}
\newcommand{\violmax}[1]{\fname{viol-max}(#1)}

\newcommand{\rewards}[1]{\fname{rewards}(#1)}
\newcommand{\goals}[1]{\fname{goals}(#1)}
\newcommand{\reducelayer}[1]{L_{#1}(\task)}
\newcommand{\reduce}[1]{\fname{reduce}(#1)}
\newcommand{\nonreduce}[1]{\fname{non-reduce}(#1)}

\newcommand{\RR}{RR}

\newcommand{\highest}[1]{\fname{highest}(#1)}

\newcommand{\layer}[1]{\fname{layer}(#1)}

\newcommand{\stratlayer}[1]{z_{#1}(\V)}
\newcommand{\stratlayerX}[2]{z_{#2}(#1)} 
\newcommand{\strategy}[1]{\fname{strategy}(#1)}
\newcommand{\fixp}[1]{\fname{fixp}(#1)}

\newcommand{\Glayer}[1]{\G_{#1}}
\newcommand{\G}{g^{(\st)}}
\newcommand{\go}[1]{\fname{go}^{(#1)}}

\newcommand{\restart}[1]{\fname{restart}(#1)}
\newcommand{\expect}[3]{\fname{expect}(#1,#2,#3)}

\newcommand{\nextname}{\beta}
\newcommand{\next}[1]{\nextname(#1)}

\DeclareMathOperator{\cause}{\triangleright}

\newcommand{\cycle}{\mathcal{C}}

\begin{document}

\title{Learning with \ramp}

\author{
    Tom~J.~Ameloot%
    \thanks{This work was made while T.J.~Ameloot was a postdoctoral fellow of the Research Foundation -- Flanders (FWO), at Hasselt University and the transnational University of Limburg (Belgium).}
    \and
    Jan~Van~den~Bussche
}
\date{}

\maketitle{}

\begin{abstract}    
    We study a learning principle based on the intuition of forming ramps. The agent tries to follow an increasing sequence of values until the agent meets a peak of reward. The resulting \ramp\ algorithm is natural, easy to configure, and has a robust implementation with natural numbers.
\end{abstract}

\setcounter{tocdepth}{2}
\tableofcontents

\section{Introduction}
\label{sec:intro}

In reinforcement learning, techniques such as temporal difference learning~\cite{sutton_1988} are used to model biological learning mechanisms~\cite{potjans_2011,schultz_2013,schultz_2015}. In that context, \citet{fremaux_2013} have simulated neuron-based agents acting on various tasks, where the firing frequency of some neurons represents the value of encountered states. \citet{fremaux_2013} observed that the simulated value neurons behave in a ramp-like manner: the firing frequency of a value neuron steadily increases as the agent approaches reward.
Moreover, \citet{fremaux_2013} discuss an interesting link between their simulations and the behavior of real ``ramp'' neurons studied by~\citet{van-der-meer_2011}.
Therefore, we believe that the ramp intuition deserves further analysis, to better understand its potential use as a learning principle.
        
Our aim in this paper is to study the value-ramp principle from a general reinforcement learning perspective. Thereto, we formalize the intuition with a concrete algorithm, called \ramp. 
As in Q-learning~\Qrefs, we compute a value $\V(\st,\act)$ for each state-action pair $(\st,\act)$. The state value $\V\get\st$ is the maximum across the actions, i.e., $\V\get\st=\max\set{\V(\st,\act)\mid\act\text{ is an action}}$.
Letting $\R(\st,\act)$ be a nonnegative reward quantity obtained when performing $\act$ in $\st$, and letting $\st'$ be the successor state,  \ramp\ updates $\V(\st,\act)$ as follows:
\begin{equation}
    \V'(\st,\act) := \max(0, \V(\st,\act) + d),
\end{equation}
where
\[
    d = \max(\V\get{\st'},\R(\st,\act))-\step-\V\get\st,
\]
and $\step\geq1$ is a fixed step size. Rewards are assumed to be nonnegative, and values are constrained to be nonnegative.%
    \footnote{The nonnegative range is inspired by biological learning models where the (positive) reward spectrum has a dedicated representation mechanism, leaving room for a dual mechanism to represent the aversive spectrum~\cite{schultz_2013,hennigan_2015}. Representing the aversive spectrum is left as an item for further work (see Section~\ref{sec:conclusion}).}
We assume throughout this paper that values are natural numbers; natural numbers are adequate for our study. As a benefit, natural numbers can be implemented compactly and robustly on a computer.
    
Intuitively, value is reward expectation, or closeness to reward.
On a path, \ramp\ propagates encountered reward quantities (and values) backwards in time, where each step subtracts $\step$.    
The effect is that when following the path forwards, we see increasingly larger values, and there is a reward peak at the end.  
By choosing actions to maximize value, the agent can follow an increasing ramp of values. After learning, the agent experience may appear as in Figure~\ref{fig:ramp-intro}.

\begin{figure}[H]
    \begin{center}
        \includegraphics[width=\textwidth]{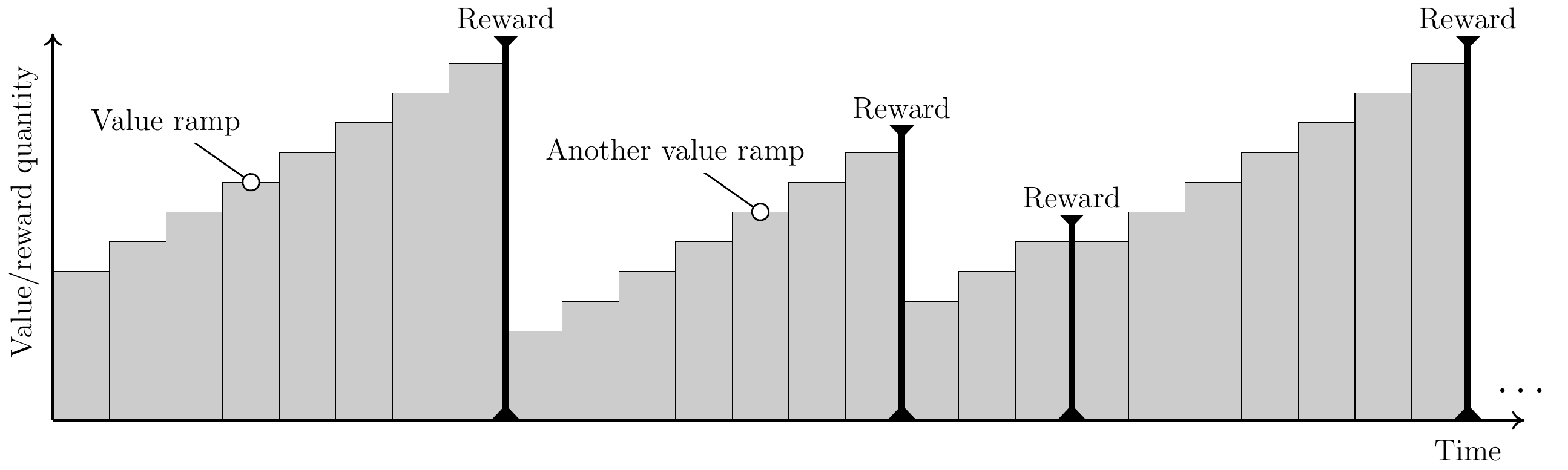}
    \end{center}
    \caption{Illustration of the ramp-like value experience of the agent.}
    \label{fig:ramp-intro}    
\end{figure}
        
To find concrete insights about the resulting agent behavior, our approach is to formally study \ramp\ on well-defined tasks. This approach can be likened to devising specific experiments in which the observed agent behavior is described. An important difference, though, is that we formally prove the observations. Although one may expect real tasks to be more complex than those investigated here, it still appears beneficial to have concrete insights about well-defined circumstances.
    Possibly, real-world behaviors can be understood as a mixture of formally described behaviors.
Below, we summarize the insights of our study in an informal manner.
\begin{theoremexplore}[Theorem~\ref{theo:explore} and Theorem~\ref{theo:sprint}] 
    When exploring, the agent wields a global viewpoint where it can reason about different reward magnitudes in the task. The agent could compute a height-map of values.
    To elaborate that intuition, we have considered deterministic tasks, where the agent always sees the same successor state for each state-action pair. We additionally assume that any state can be reached from any other state. For each state $\st$, reward quantities become less important when they are remote from $\st$.
    We show the following: by repeatedly trying all state-action pairs (in exploratory fashion), the agent learns for each state which rewards have the best quantity-versus-distance trade-off.
    Subsequently, by choosing actions to maximize value, the agent continuously moves to the highest reward as soon as possible. This insight shows the potential use of \ramp\ as a behavior optimizer.
\end{theoremexplore}

\begin{theoremgreedy}[Theorem~\ref{theo:greedy}]
    When constantly choosing actions to maximize reward, i.e., in a greedy approach, the agent has a local viewpoint restricted to measuring progress along a path. Here, the agent does not care about all rewards, just about finding one reward. This is useful for navigational tasks, even in abstract state spaces.
    To elaborate that intuition, we have considered nondeterministic tasks, where the successor state resulting from a state-action pair could vary.
    We make the relaxing assumption that the state space can be viewed as a stack of layers, where the bottom layer contains reward, and where the states at each layer can move robustly into a deeper layer (but without knowing the precise successor).
    We show the following: by constantly choosing actions to maximize value, the agent eventually learns to completely avoid cycles without reward. Phrased differently, eventually, whenever the agent walks in a cycle, the cycle is broken by reward (no matter how small the cycle is).
    This insight shows that \ramp\ can keep navigating to reward, even in tasks with a degree of unpredictability.
\end{theoremgreedy}       

The above insights apply to many tasks, ranging from navigation on maps to finding rewarding strategies in abstract state spaces.
Interestingly, \ramp\ appears easy to configure. The theorems work for any step size $\step\geq 1$, but in practice one could simply take $\step=1$.
We also introduce a parameter $\epsilon$ to control the degree of exploration, which is common practice in reinforcement learning. 
Other approaches in reinforcement learning often have multiple parameters, e.g., in Q-learning~\Qrefs\ one has the learning rate and the reward discounting factor.
 
In summary, the \ramp\ algorithm has useful characteristics: (1) it is conceptually simple, (2) it is easy to configure, and (3) it has a stable implementation based on natural numbers.
Additionally, insights discussed in this paper suggest that the algorithm might be versatile.

\paragraph*{Outline}
This paper is organized as follows.
Section~\ref{sec:alg} formalizes the \ramp\ algorithm and tasks. Next, Section~\ref{sec:explore} contains the insights about exploration and optimization on deterministic tasks. Section~\ref{sec:greedy} contains the insight about greedy learning on nondeterministic tasks. We conclude with items for further work in Section~\ref{sec:conclusion}.

\section{\ramp\ algorithm}
\label{sec:alg}

In this section, we introduce the \ramp\ algorithm in a general reinforcement learning setting~\cite{sutton-barto_1998}.  In subsequent sections, we analyze the behavior of the algorithm on two classes of applications, one based on continued exploration (Section~\ref{sec:explore}) and the other based on greedy path following (Section~\ref{sec:greedy}).

\subsection{Basic definitions}
Suppose we have a finite set $\states$ of states and a finite set $\actions$ of actions.
We write \[
    \st \edge{\act} \st'
\]
to denote that we can reach state $\st'$ by applying action $\act$ to state $\st$.

As in Q-learning~\Qrefs, we assign a numerical value to each pair $(\st,\act)\in\states\times\actions$. This setup reflects the intuition that states by themselves do not necessarily have meaning, but rather it is the intention, or action, in the state that matters.
In the present paper, natural numbers are sufficient for representing values. Hence, a value function $\V$ is of the form
\[
    \V:\states\times\actions\to\nat.
\]
We define the value of a state $\st$, denoted $\V\get\st$, as the maximum of the values over the actions:
\[
    \V\get\st = \max\set{\V(\st,\act)\mid \act\in\actions}.
\]
The set of actions preferred by $\st$ in $\V$, denoted $\pref\st\V$, contains the actions with the highest value in $\st$:
\[
    \pref\st\V = \set{\act\in\actions\mid\V(\st,\act)=\V\get\st}.
\]
Note that always $\pref\st\V\neq\emptyset$. 

For each pair $(\st,\act)\in\states\times\actions$, we have an immediate reward quantity $\R(\st,\act)\in\nat$ to say how good action $\act$ is in state $\st$.
The reward is given externally to the agent, whereas a value function forms the internal belief system of the agent about expectations (of reward).

As convenience notation, for each integer $x\in\ints$, we define a clamping operation
\[
\clamp x = \begin{cases}
            x & \text{if }x \geq 0;\\
            0 & \text{if }x < 0.
            \end{cases}
\] For any two integers $x$ and $y$, note that $x\leq y$ implies $\clamp{x}\leq\clamp{y}$. In the proofs we also frequently use the equality $\max(\clamp{x},\clamp{y})=\clamp{\max(x,y)}$.%
    \footnote{If $x\leq y$ then $\clamp{x}\leq\clamp{y}$; subsequently, $\max(\clamp{x},\clamp{y})=\clamp{y}=\clamp{\max(x,y)}$. The other case is symmetrical.}

\subsection{Desired properties}\label{sub:desired}
Suppose we have a path
\[
    \stX{1}
    \edge{\actX 1}
    \ldots
    \edge{\actX{n-1}}
    \stX n
    \edge{\actX n}
    \stX{n+1}.
\]
We fix some step size $\step\in\nat$ with $\step\geq1$.
Now, if the agent would repeatedly visit the above path, our intention of the \ramp\ algorithm is to find a value function $\V$ with the following properties: for each $i\in\set{1,\ldots,n}$,
\begin{itemize}
    \item for each $j\in\set{i,\ldots,n}$, we have $\V(\stX i,\actX i) \geq \clamp{\R(\stX j,\actX j) - (j-i+1)\step}$; and,

    \item $\V(\stX i,\actX i) \leq \clamp{\max(\V\get{\stX{i+1}}, \R(\stX i,\actX i)) - \step}$.
\end{itemize} 
The first property says that value should reflect reward expectation, taking into account the time until reward, as illustrated in Figure~\ref{fig:ramp-pillar}.%
    \footnote{If $j=i$ then we subtract $\step$ once from $\R(\stX i,\actX i)$, to indicate that action $\actX i$ should first be executed at state $\stX i$ before the reward is given.}
Rewards can be surpassed by value expectations for larger rewards, as illustrated in Figure~\ref{fig:ramp-propagate}. 
The first property suggests a mechanism for increasing value, to propagate reward information backwards in time. 

The second property says that local maxima can be sustained only by reward; other local maxima should be gradually removed, as illustrated in Figure~\ref{fig:ramp-violation}. We might refer to local value maxima without reward as violating values.    
The second property suggests a mechanism for subtracting value, in order to maintain the ramp shape. 
    
\begin{remark}[Note on violating values]
    In absence of reward, the only way to prevent nonzero values from being labeled as violating, is to have an infinite ramp of increasingly larger values. But there are only a finite number of states, so eventually the ramp should meet true reward.
    \qed
\end{remark}

\begin{figure}
\begin{center}
    \includegraphics[width=0.6\textwidth]{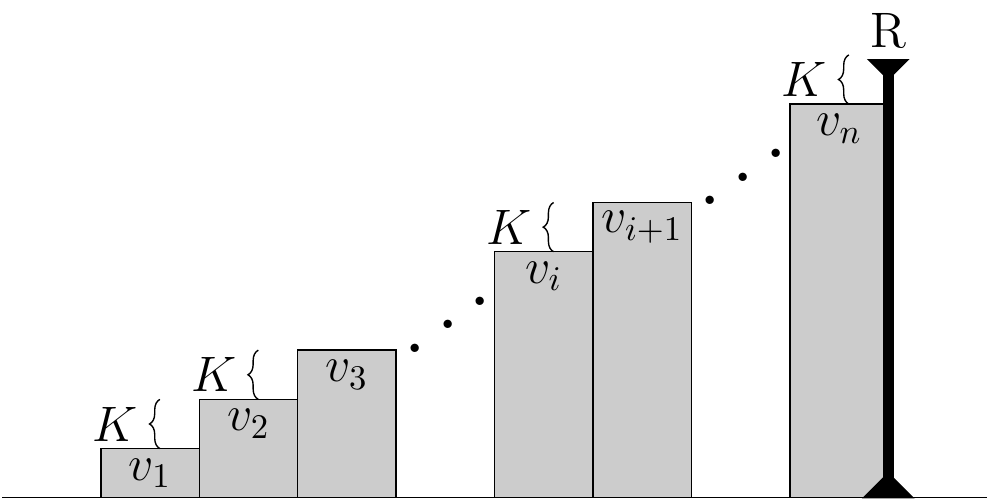}
\end{center}
\caption{Depiction of a value ramp. The agent is following a path of $n$ states, whose values $v_1,\ldots,v_n$ (in the encountered value functions) form a ramp. For graphical simplicity, there is only one nonzero reward, at the end. The height of the reward ``pillar'' reflects the actual reward quantity.}
\label{fig:ramp-pillar}
\end{figure}

\begin{figure}
\begin{center}
    \includegraphics[width=\textwidth]{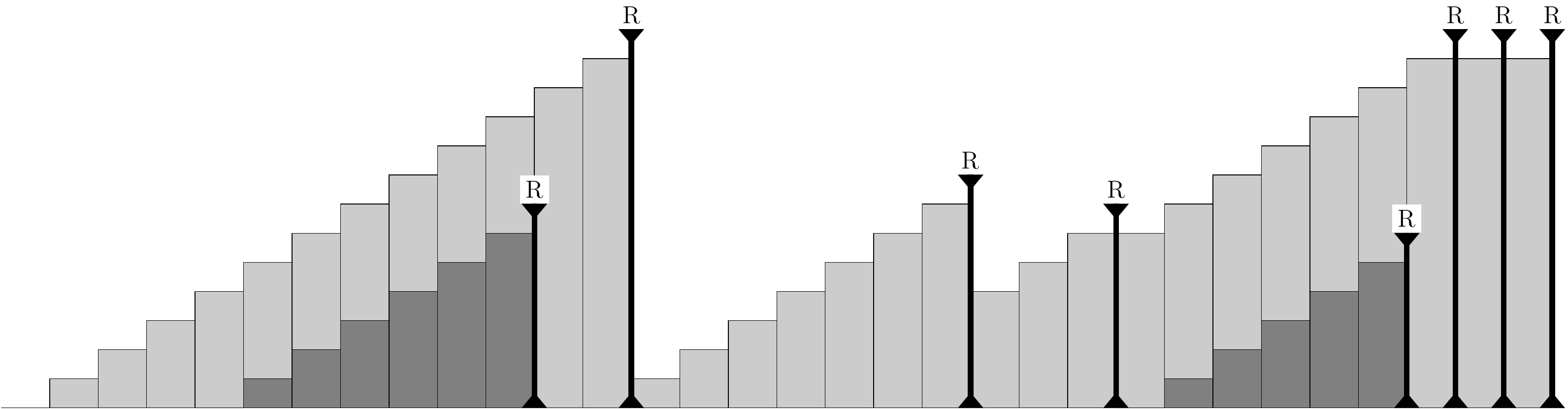}
\end{center}
\caption{Rewards can be surpassed by higher value expectations. We have given a darker shade to the ramps of surpassed rewards. Also, note that in principle the same reward quantity could be repeated in subsequent time steps (suggested at the end of this figure).}
\label{fig:ramp-propagate}
\end{figure}

\begin{figure}
\begin{center}
    \includegraphics[width=0.8\textwidth]{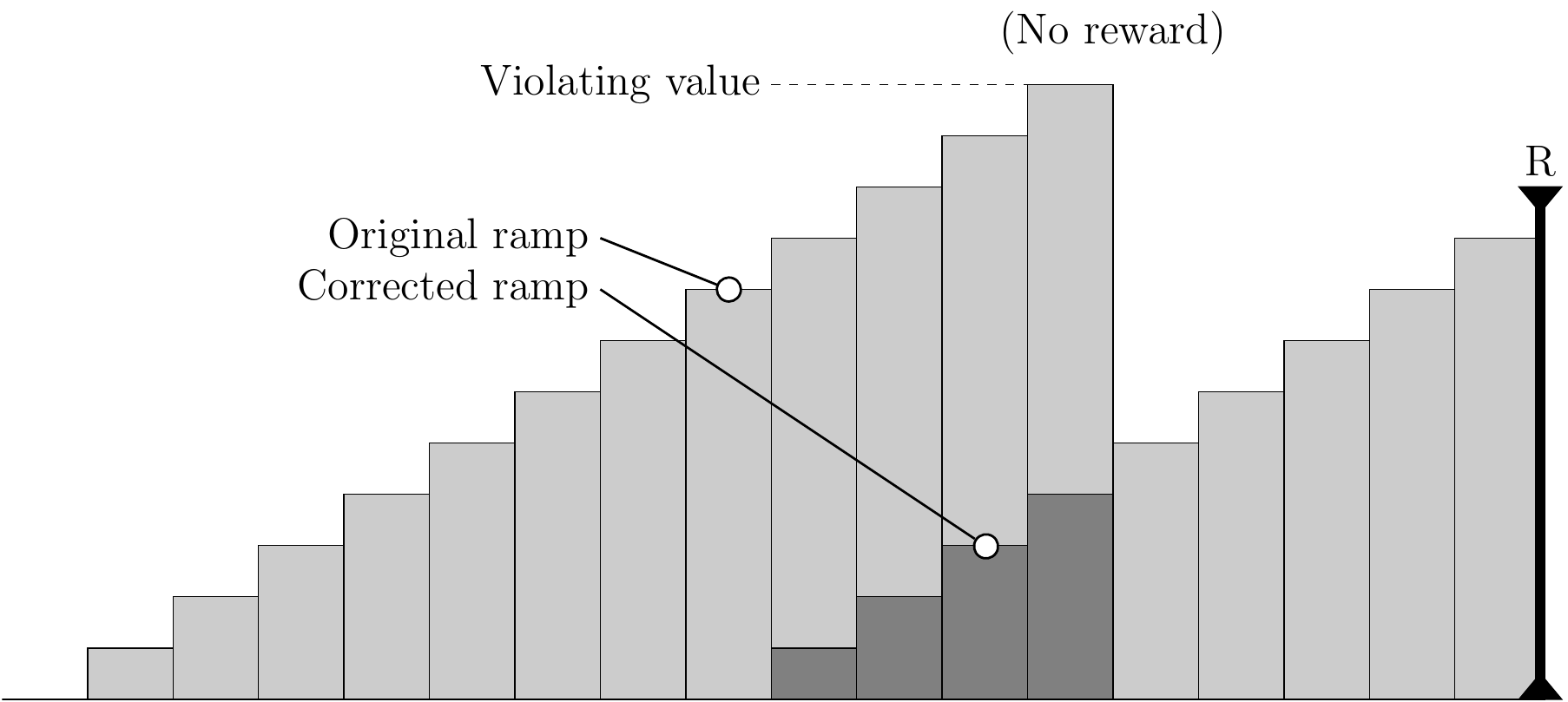}
\end{center}
\caption{To obtain ramps in alignment with reward, we should remove local maxima without reward, called violations.}
\label{fig:ramp-violation}
\end{figure}

\begin{remark}[Outlook]
    A value ramp reflects some ideal value function that, likely, can only be obtained under the right circumstances. 
    Moreover, it might be difficult to describe properties that are both interesting and sufficiently general, because there are widely different kinds of tasks (or environments) upon which an agent could operate. 
    For these reasons, in Section~\ref{sec:explore} and Section~\ref{sec:greedy} we provide more detailed insights for specific classes of tasks. This could provide an initial foundation for understanding the value ramp principle.
    \qed
\end{remark}

\subsection{Concrete algorithm}\label{sub:alg}
A learning rule is a function that produces a value change (as an integer) when given a triple $(v,v',r)\in\nat\times\nat\times\nat$, where $v$ is the value of the current state, $v'$ is the value of the next state, and $r$ is the reward quantity observed during the transition from the current state to the next state.
The desired properties from Section~\ref{sub:desired} inspire a concrete learning rule $\rl$, where $\step\geq 1$, defined for each triple $(v,v',r)$ as
\begin{equation}\label{eq:ramp}
    \rl(v,v',r) = \max(v',r)-\step-v.
\end{equation}
The proposed value change could be either strictly positive, zero, or strictly negative.

\begin{remark}[Usage of state value]
    Recall that state value is defined as the maximum over state-action values. There are multiple reasons for using state values to compute the update in Equation~\eqref{eq:ramp} instead of using bare state-action values. First, using state value appears more biologically plausible, since the brain likely assigns one (global) value to each observed state~\cite{potjans_2011,fremaux_2013}. When moving from state to state, the global value could be the aggregate of detailed state-action values. 
    Second, if we try to use bare state-action values instead, there is no natural mechanism to select which action value to associate with an observed successor state; causing us to default to the highest action value for instance.
    \qed
\end{remark}

Based on Equation~\eqref{eq:ramp}, we now formalize how the agent updates value through experience. A configuration of the system is a pair $(\st,\V)$ saying that we are in state $\st$ and that the current value function is $\V$.
\begin{definition}[Transition]\label{def:trans}
A transition is a quintuple $(\stX 1,\VX 1,\actX 1,\stX 2,\VX 2)$, where $\cnftupX 1$ and $\cnftupX 2$ are configurations; $\stX 2$ is reached from $\stX 1$ through action $\actX 1$; and, $\VX 2$ is defined, for each $(\st,\act)\in\states\times\actions$, as
\[
    \VX 2(\st,\act) = 
        \begin{cases}
            \clamp{\VX1(\stX 1,\actX 1) + \rl\big(\VX 1\get{\stX 1}, \VX 1\get{\stX 2},\R(\stX 1,\actX 1)\big)} & \text{if }(\st,\act)=(\stX 1,\actX 1);\\
            
            \VX 1(\st,\act) & \text{otherwise}.
        \end{cases}
\]
\end{definition}
\noindent We emphasize that during the transition, the value of the successor state $\stX 2$ is based on the old value function $\VX 1$.
We also write the transition as 
\[
    \cnftupX 1
    \jump{\actX 1,\, \stX 2}
    \cnftupX 2.
\]

Algorithm~\ref{alg:update} gives pseudocode for performing transitions.
The full \ramp\ algorithm, shown in Algorithm~\ref{alg:full}, repeatedly generates transitions.    
Note that there is a probability $\epsilon$ at each time step of choosing from all actions, instead of choosing from the actions with highest value.    
If $\epsilon=0$ then the algorithm follows the best known path to reward, without further exploration; in that case, we say that the algorithm is greedy.

\begin{algorithm}[H]
\captionsetup{singlelinecheck=off}
\caption[\ramp\ update function]{
 \ramp\ update function.\\    
    Input: 
    
    \quad$\bullet$ $(\st,\V)$: current configuration
    
    \quad$\bullet$ $\act$: performed action
    
    \quad$\bullet$ $\st'$: observed successor state
}
\label{alg:update}
\begin{algorithmic}[1]
    \Function{Update}{$\st,\V,\act,\st'$}       
    \State $d$ := $\max(\V\get{\st'},\R(\st,\act)) - \step - \V\get\st$
    \State Define $\V'$ as $\V$ but set $\V'(\st,\act):=\clamp{\V(\st,\act)+d}$
    \State\Return $\V'$
    \EndFunction
\end{algorithmic}
\end{algorithm}

\begin{algorithm}[h]
\captionsetup{singlelinecheck=off}
\caption[\ramp\ algorithm]{
    \ramp\ algorithm.\\    
    Input:
    
    \quad$\bullet$ $\VX 1$: initial value function (random)
    
    \quad$\bullet$ $\stX 1$: initial start state
    
    \quad$\bullet$ $\step$: step size with $\step\geq 1$
    
    \quad$\bullet$ $\epsilon$: probability in $[0,1]$
}
\label{alg:full}
\begin{algorithmic}[1]    
    \Procedure{\ramp}{}    
    \State $\V$ := $\VX 1$
    \State $\st$ := $\stX 1$  
    
    \Repeat                
        \State $\act$ := choose from $\pref\st\V$
        \State With probability $\epsilon$, do $\act$ := choose from $\actions$
        \State $\st'$ := some state resulting from $(\st,\act)$
        \State $\V'$ := Update($\st,\V,\act,\st'$)\Comment{See Algorithm~\ref{alg:update}}
        
        \State $\st$ := $\st'$
        \State $\V$ := $\V'$
    \Until{Interrupt}
    \EndProcedure
\end{algorithmic}
\end{algorithm}

\begin{remark}[Natural numbers]\label{remark:natural-numbers}
    Since we use a discrete time framework,
    natural numbers are a perfect fit for representing the steps of a ramp.
    Practical implementations of natural numbers are robust under addition and subtraction.
    Also, as is commonly known, a string of $n$ bits can represent any natural number in the range $\set{0,\ldots,2^n-1}$. Modest storage requirements can therefore accommodate huge values. That might be useful for learning (very) long paths in navigation problems (see Section~\ref{sec:greedy}).
    
    Many approaches in reinforcement learning are based on rational numbers~\cite{sutton-barto_1998}. Approximation errors arise when rational numbers are implemented as floating point numbers, inspiring the development of new digital number formats~\cite{gustafson_2015}.  By using natural numbers, \ramp\ avoids approximation errors.
    \qed
\end{remark}

\begin{remark}[Parameters]    
    A first parameter of \ramp\ is the step size $\step$. We develop the formal insights for a general $\step\geq 1$ (e.g.\ Theorem~\ref{theo:explore} and Theorem~\ref{theo:greedy}). In practice, it might be useful to simply set $\step=1$, because then rewards generate longer ramps, allowing the agent to learn longer strategies to reward.    
    Second, the exploration probability $\epsilon$ in Algorithm~\ref{alg:full} is a standard principle in reinforcement learning~\cite{sutton-barto_1998}.    
    
    \ramp\ has no other parameters besides $\step$ and $\epsilon$.
    In comparison, the general framework of reinforcement learning introduces an $\alpha$ and $\gamma$ parameter~\cite{sutton-barto_1998}. This applies in particular to Q-learning~\Qrefs, which has famous applications~\cite{mnih_2015}. Parameter $\alpha$ can be understood as the learning rate. Parameter $\gamma$, representing reward-discounting, is slightly less intuitive and could require detailed knowledge of the task domain in order to produce desired agent behavior~\cite{schwartz_1993}.
    
    \ramp\ replaces the $\alpha$ parameter by a fast value update mechanism that immediately establishes a (local) ramp shape on encountered states. An item for further work is to slow down the value update in the context of biological plausibility (see Section~\ref{sec:conclusion}).
    
    \ramp\ dismisses the $\gamma$ parameter by directly using reward quantities to define the height of ramps. For each reward, the ramp shape establishes a natural trade-off between the quantity of a reward and the time to get there. An item for further work is to investigate in more detail the relationship between reward discounting and the value ramp principle (see Section~\ref{sec:conclusion}).
    \qed
\end{remark}

\begin{remark}[Fixing $\step$]
    All definitions and results hold for any $\step\geq 1$. But for notational simplicity, we choose not to mention the symbol ``$\step$'' explicitly in the notations. We assume that throughout the rest of the paper, some particular $\step\geq 1$ is fixed.
    \qed
\end{remark}

\subsection{Tasks formalized}
\label{sub:task}

We want to describe the effect of \ramp\ on tasks. 
Formally, a task $\task$ is a tuple $\tasktup$, where
\begin{itemize}
    \item $\states$ is a finite, nonempty, set of states;
    \item $\startstates\subseteq\states$ is a set of start states;
    \item $\actions$ is a finite, nonempty, set of actions;
    \item $\tr:\states\times\actions\to\powset\states$ is the transition function, where $\tr(\st,\act)\neq\emptyset$ for each $(\st,\act)\in\states\times\actions$;%
        \footnote{For a set $X$, the symbol $\powset{X}$ denotes the powerset of $X$, which is the set of all subsets of $X$.}
    and,
    \item $\R:\states\times\actions\to\nat$ is the reward function.
\end{itemize}
For any $(\st,\act,\st')\in\states\times\actions\times\states$, we write $\st\edge{\act}\st'$ if $\st'\in\tr(\st,\act)$.

A run of \ramp\ on $\task$ is an infinite sequence of transitions, where the target configuration of each transition is the source configuration of the next transition,
\[
    \cnftupX 1
        \jump{\actX 1,\,\stX 2}
    \cnftupX 2
        \jump{\actX 2,\,\stX 3}
    \cnftupX 3
        \jump{\actX 3,\,\stX 4}
    \ldots,
\]
where $\stX 1\in\startstates$, and for each $i\geq 1$ we have $\stX{i+1}\in\tr(\stX i,\actX i)$.
For each $i\geq 1$, we recall that $\VX{i+1}$ is uniquely determined by $(\stX i,\VX i,\actX i,\stX{i+1})$ (see Section~\ref{sub:alg}).
We allow $\VX 1$ to be a random value function.
We emphasize that the successor state of each transition is restricted by function $\tr$.

The following lemma is a general observation that we will use frequently in proofs:
\begin{lemma}\label{lem:finite-cnf}
    For any task, in any infinite transition sequence, there are only a finite number of possible configurations.
\end{lemma}
\begin{proof}
    Let $\tasktup$ be the task. For a value function $\V$ we define $\ceil\V=\max(m_1,m_2)$ where
    \begin{align*}
        m_1 &= \max\set{
        \R(\st,\act)\mid(\st,\act)\in\states\times\actions}\text{; and,}\\
        m_2 &= \max\set{\V(\st,\act)\mid(\st,\act)\in\states\times\actions}.
    \end{align*}
    Intuitively, $\ceil\V$ is the highest quantity accessible by the agent; this quantity is either defined by reward or by the value function itself.
    For each transition 
    \[
        \cnftup
            \jump{\act,\st'}
        (\st',\V'),
    \]
    we can show that $\ceil\V\geq\ceil{\V'}$ (see Appendix~\ref{lem:finite-cnf--PROOF}).
    By transitivity, for every infinite transition sequence, the ceiling quantity of the first value function is an upper bound on the ceiling quantity of all subsequent value functions. So, the infinite transition sequence has a finite number of value functions because (1) there is an upper bound on the values, (2) value functions are composed of natural numbers, and (3) there are a finite number of states and actions. Therefore there are a finite number of configurations.
\end{proof}

\begin{remark}[Perception and finiteness]
    The task structure represents how the agent perceives its environment. The agent perception is in general the result of various processing steps applied to sensory information. Agent perception is not the focus of this paper. Although the environment in which the agent resides could have infinitely many states, we assume that the agent has a limited conceptual framework consisting of finitely many states. We still allow many states though.
    The finiteness of the state space is important for the convergence proofs of this paper; more precisely, the assumption is used in the general Lemma~\ref{lem:finite-cnf}. \qed
\end{remark}

\subsubsection{Kinds of run: exploring versus greedy}
\label{sub:runs}

Hereafter, we restrict attention to two kinds of run. 

\paragraph*{Exploring}
First, we say that a run is exploring if the following holds: if a configuration $\cnftup$ occurs infinitely often in the run, then for each $\act\in\actions$ and each $\st'\in\tr(\st,\act)$, there are infinitely many transitions
\[
    \transX i,
\]
where $(\stX i,\VX i,\actX i,\stX{i+1})=(\st,\V,\act,\st')$.
Intuitively, an exploring run contains a fairness assumption to ensure that the system explores infinitely often those options that are infinitely often available. 

\paragraph*{Greedy}
Second, we say that a run is greedy if the following holds:
\begin{enumerate}
    \item each transition 
        $\cnftup
            \jump{\act,\,\st'}
        (\st',\V')$ in the run satisfies $\act\in\pref\st\V$; and,
        
    \item if a configuration $\cnftup$ occurs infinitely often in the run, then for each $\act\in\pref\st\V$ and each $\st'\in\tr(\st,\act)$, there are infinitely many transitions
    \[
        \transX i,
    \]
    where $(\stX i,\VX i,\actX i,\stX{i+1})=(\st,\V,\act,\st')$.
\end{enumerate}
In a greedy run, we always select a preferred action, but the system can not reliably choose only one action from equally-preferred actions; moreover, as a fairness assumption, the system can not indefinitely postpone witnessing a certain successor state.

\begin{remark}[Relationship with Algorithm~\ref{alg:full}]
    In Algorithm~\ref{alg:full}, we generate exploring runs by setting $\epsilon> 0$. We will not use the specific $\epsilon$ value to delineate strict subclasses of exploring runs whose exploration rate satisfies $\epsilon$.
    In Algorithm~\ref{alg:full}, we generate greedy runs by setting $\epsilon=0$.
    
    When running Algorithm~\ref{alg:full} on a task $\task=\tasktup$, we assume that if the same state-action pair $(\st,\act)$ is executed infinitely often then each successor state in $\tr(\st,\act)$ is infinitely often the result of $(\st,\act)$.
    \qed
\end{remark}

\section{Exploration on deterministic tasks}
\label{sec:explore}

In a first study, we would like to show optimal value estimation of \ramp\ on at least some (well-behaved) class of tasks. Thereto we consider tasks that are both deterministic and connected, abbreviated \DC.
In Section~\ref{sub:opt} we show that exploring runs learn optimal values on \DC\ tasks. In Section~\ref{sub:sprint}, we subsequently show that when the agent uses the optimal values to select actions, the agent follows so-called optimal paths. In Section~\ref{sub:shortest}, we apply the results to shortest path following.

\subsection{Optimal value estimation}\label{sub:opt}
We first define a few auxiliary concepts.
Let $\task=\tasktup$ be a task. To improve readability, we omit symbol $\task$ from the notations below where possible; it will be clear from the context which task is meant.

\paragraph*{\DC\ tasks}
We say that $\task$ is deterministic if $\ssize{\tr(\st,\act)}=1$ for each $(\st,\act)\in\states\times\actions$. 
Next, we say that $\task$ is connected if for each $(\st,\st')\in\states\times\states$, there is a path 
\[
    \stX{1}
    \edge{\actX 1}
    \ldots
    \edge{\actX{n-1}}
    \stX n
\]
with $\stX 1=\st$ and $\stX{n}=\st'$. Connectedness means that for each state we can go to any other state.
We say that a task is \DC\ if the task is both deterministic and connected.

\paragraph*{Consistency}
On deterministic tasks, it will be interesting to observe eventual stability of the value function. 
In that context, we say that a value function $\V$ is consistent if it satisfies: 
    $\forall\st\in\states$, $\forall\act\in\pref\st\V$, denoting $\Det\st\act{\st'}$,
        \[
        \V\get\st = \clamp{\max(\V\get{\st'},\R(\st,\act)) - \step}.
        \]%
Intuitively, this means that the agent knows exactly what value to expect when following preferred actions. We will see below in the context of Corollary~\ref{corol:fixed} that consistency eventually halts the learning process on \DC\ tasks.

\paragraph*{Optimal value}
Next, we define a notion related to shortest paths. 
Let $\st$ be a state. An action-path for $\st$ is a sequence 
\[
    p=
    (\stX{1},\actX 1),
    \ldots,
    (\stX{n-1},\actX{n-1}),
    (\stX n,\actX n),
\]
of state-action pairs, where $\stX 1=\st$, and $\stX{i+1}\in\tr(\stX i,\actX i)$ for each $i\in\set{1,\ldots,n-1}$. We allow $n=1$.
We define the value of $p$, denoted $\pval p$, as
\begin{equation}
    \pval p = \max\set{\clamp{\R(\stX i,\actX i)-i\step}\mid i\in\set{1,\ldots,n}}.\label{eq:pval}
\end{equation}
This value expresses a trade-off between time and reward amplitude. For example, high rewards could become less important than lower rewards if the time distance is too long.
The concept is illustrated in Figure~\ref{fig:opt-val}.
Note that always $\pval p\geq 0$ due to the clamping operation.

\begin{figure}
    \begin{center}
        \includegraphics[width=0.9\textwidth]{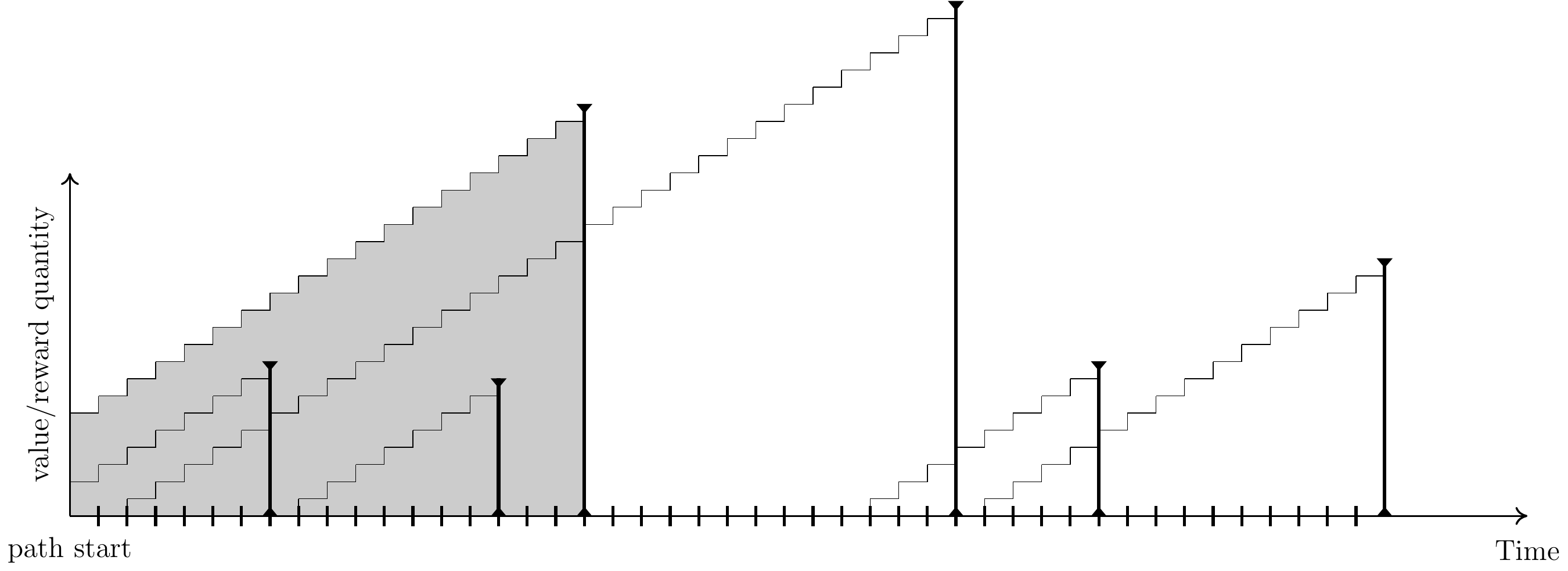}
    \end{center}    
    \caption{Illustration of action-path value. Each tick mark on the horizontal axis represents a state-action pair of the action-path. To keep the figure simple, only some of the state-action pairs have a strictly positive reward, leading to a value-ramp. We have shaded the value-ramp of the state-action pair that determines the path value. The other state-action pairs project less reward expectation towards the beginning of the path.}
    \label{fig:opt-val}
\end{figure}

If the action-path $p$ contains a cycle of states then there is always an action-path $p'$ without such cycles and with $\pval{p'}\geq\pval{p}$. To see this, we can do the following steps to transform $p$ into a cycle-free action-path without decreasing the value:
\begin{enumerate}
    \item We select some $i\in\set{1,\ldots,n}$ with $\pval p=\clamp{\R(\stX i,\actX i)-i\step}$.
    
    \item\label{enu:prune-suffix} We remove all pairs $(\stX j,\actX j)$ with $j > i$.
    
    \item In the remaining path, we systematically replace all cycles $(\st,\act),\ldots,(\st,\act')$ (with repeated state $\st$) by the single step $(\st,\act')$.
    Note that pair $(\stX i,\actX i)$ is preserved because this pair comes last, as caused by step~\ref{enu:prune-suffix}. As a result, the reward quantity $\R(\stX i,\actX i)$ can only come closer to the beginning of the path.
\end{enumerate}
Let $\explore\st$ be the set of all cycle-free action-paths starting at state $\st$.
We define the optimal value of $\st$, denoted $\optval\st$, as
\[
    \optval\st = \max\set{\pval p\mid p\in\explore\st},
\]
i.e., the optimal value is the largest value across the (cycle-free) action-paths. 
The case $\optval\st=0$ occurs when all reward is too remote for $\st$.

We say that a value function $\V$ is optimal if it satisfies: $\forall\st\in\states$, 
\[
\V\get\st = \optval\st.
\] 

We are now ready to state the optimization result:
\begin{theorem}[Optimization]\label{theo:explore}
    For each \DC\ task, 
    in each exploring run, 
    eventually every value function is both optimal and consistent.
\end{theorem}
\noindent \proofmain{theo:explore--PROOF}
The following corollary provides an additional insight about the learning process on \DC\ tasks:
\begin{corollary}\label{corol:fixed}
    For each \DC\ task, in each exploring run, eventually the value function is no longer changed, i.e., there is a fixpoint on the value function.        
\end{corollary}
\noindent \proofmain{corol:fixed--PROOF} 
\begin{example}[Example simulation]\label{ex:dc-sim}
    To illustrate Theorem~\ref{theo:explore}, we have simulated the \ramp\ algorithm on a 2D grid world that is both deterministic and connected (\DC).
    Each cell $(x,y)\in\nat\times\nat$ inside the boundaries of the map is a distinct state. 
    There is a fixed start cell. At each cell, there are five deterministic actions available to the agent: left, right, up, down, and finish. The agent can not move through wall cells, serving as obstacles. Some cells are marked as goal cells. By performing the finish action in a goal cell $g$, the agent receives a fixed reward quantity associated with goal cell $g$, and the agent is subsequently sent back to the fixed start cell. In a non-goal cell, the finish action neither gives reward and neither moves the agent to another cell.
    The agent learns the values of all cell-action pairs. Figure~\ref{fig:dc-sim} shows for three different maps how value is  propagated from the goal cells across the map. Eventually, the cell values visibly stabilize; one could imagine that this is the point after which the value function (1) is optimal and consistent (Theorem~\ref{theo:explore}) and (2) no longer changes (Corollary~\ref{corol:fixed}).\qed
\end{example}

\begin{figure}
\newcommand{\spaceline}{ & & }
\newcommand{\wOne}{0.35\textwidth}
\begin{center}
\begin{tabular}{|c|c|c|}
    \hline
    (A) setup & (B) midway & (C) final\\    
    \hline

    \includegraphics[width=\wOne]{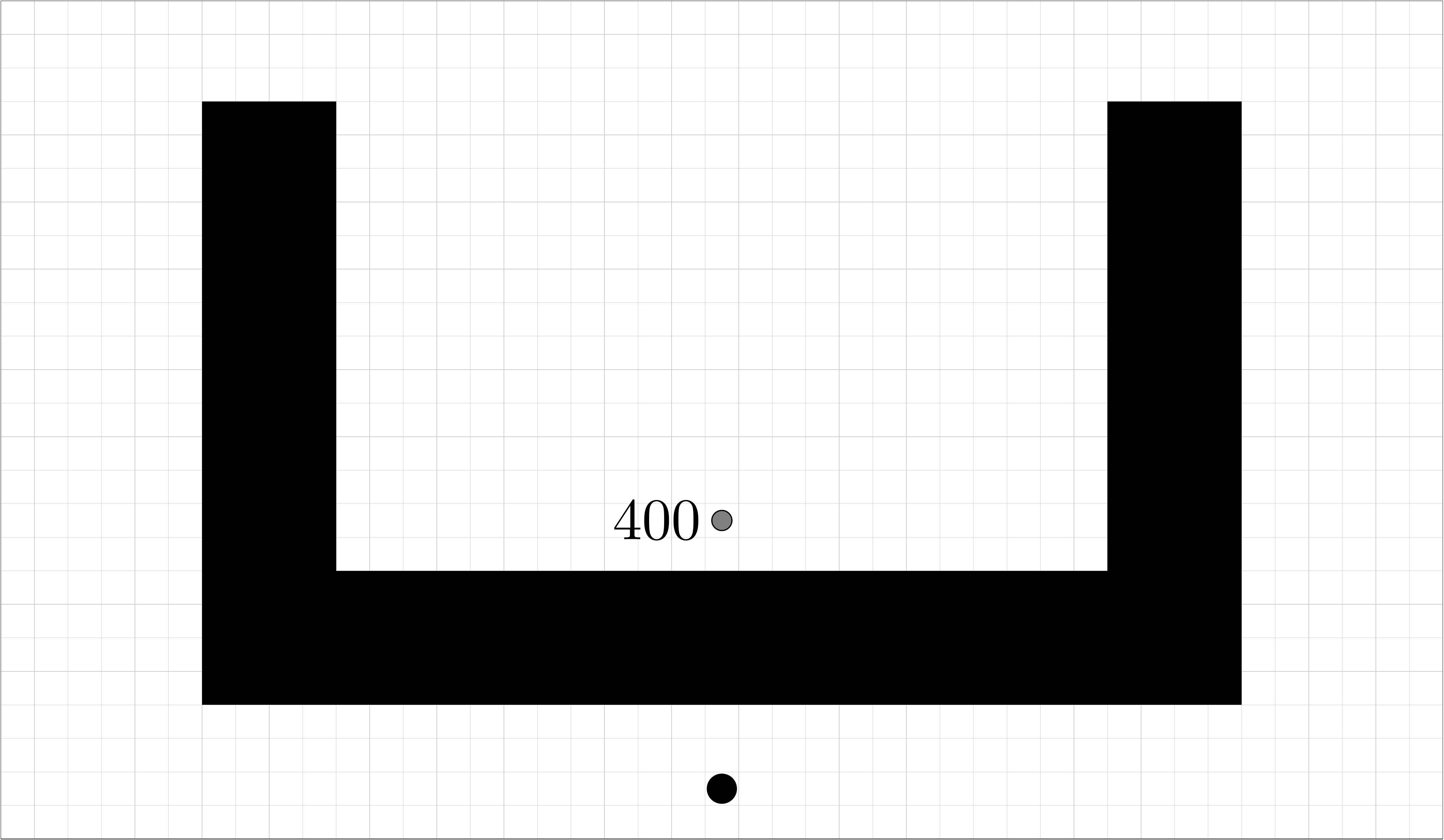}&
    \includegraphics[width=\wOne]{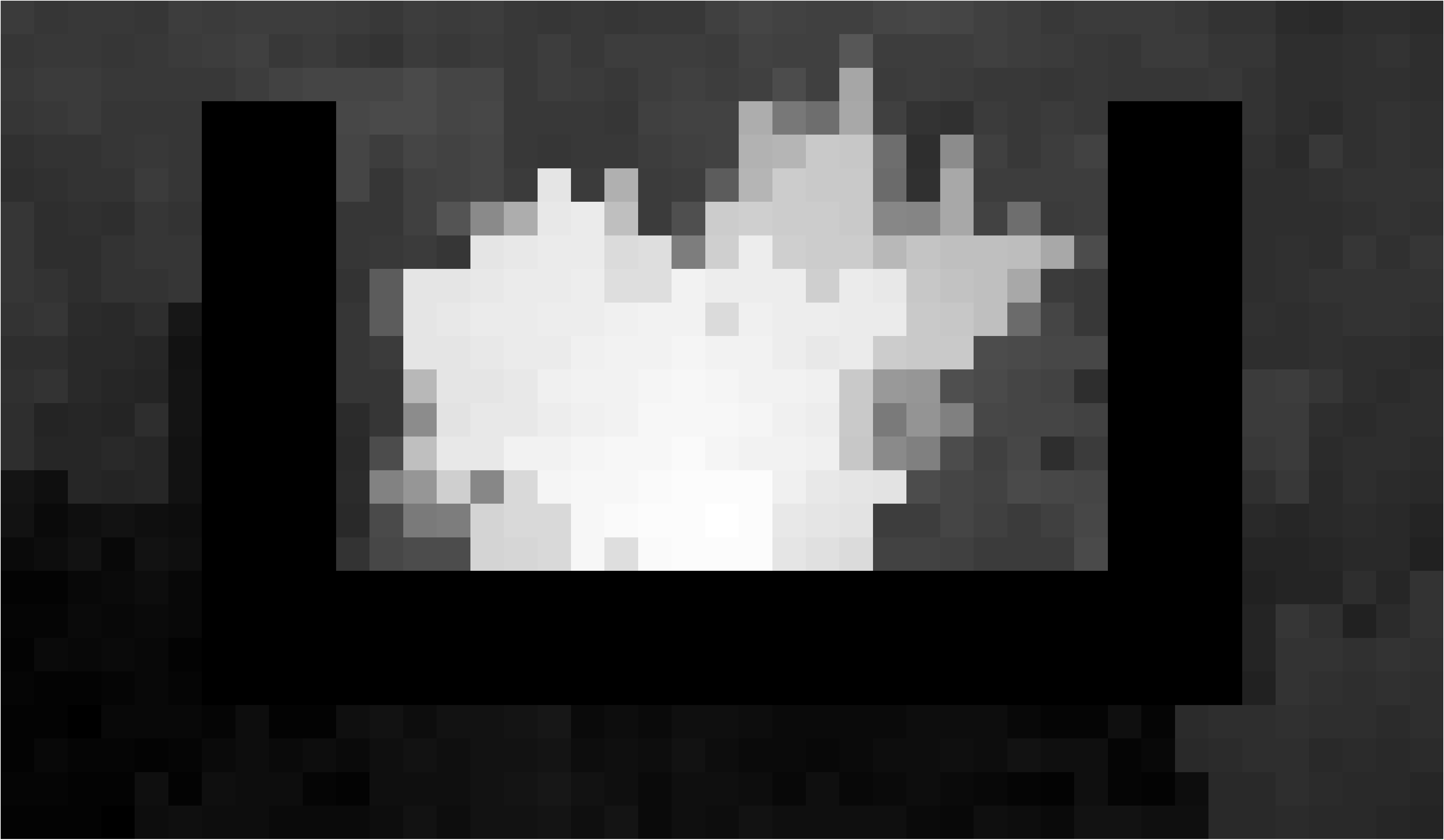} &
    \includegraphics[width=\wOne]{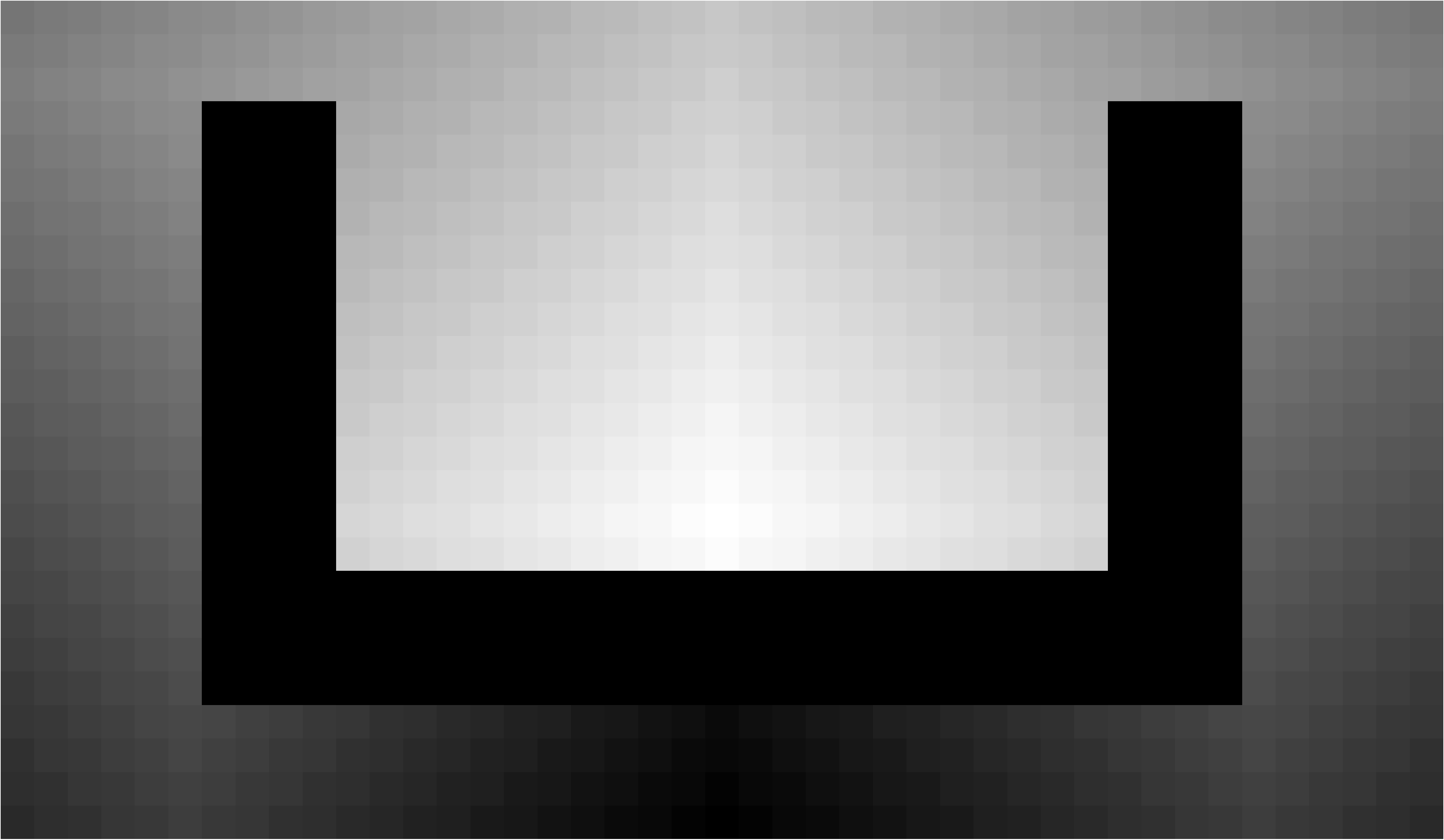} \\
    
    $\step=2$. Initial values: $[0,200]$ & & \\
    
    \hline    

    \includegraphics[width=\wOne]{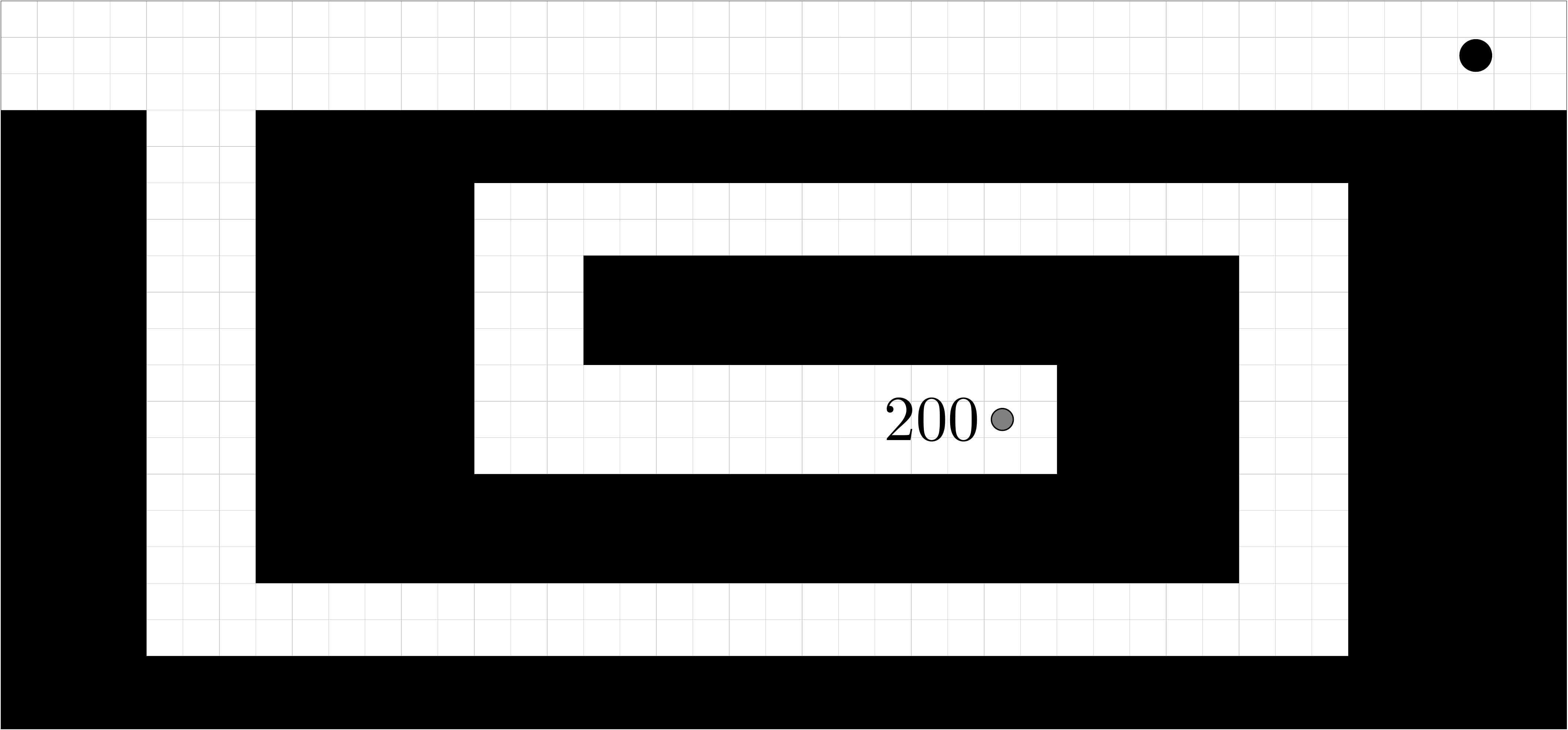} &    
    \includegraphics[width=\wOne]{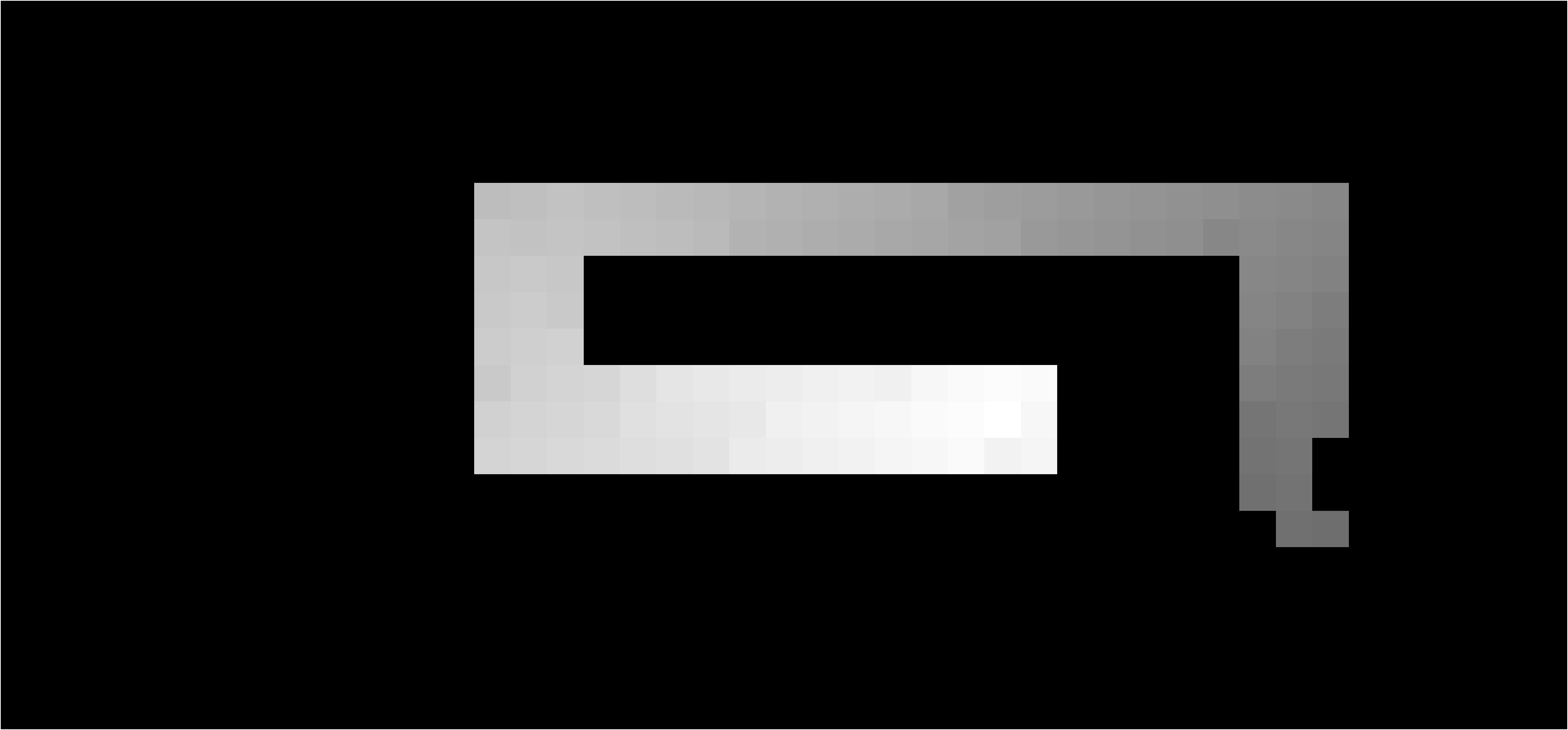} &
    \includegraphics[width=\wOne]{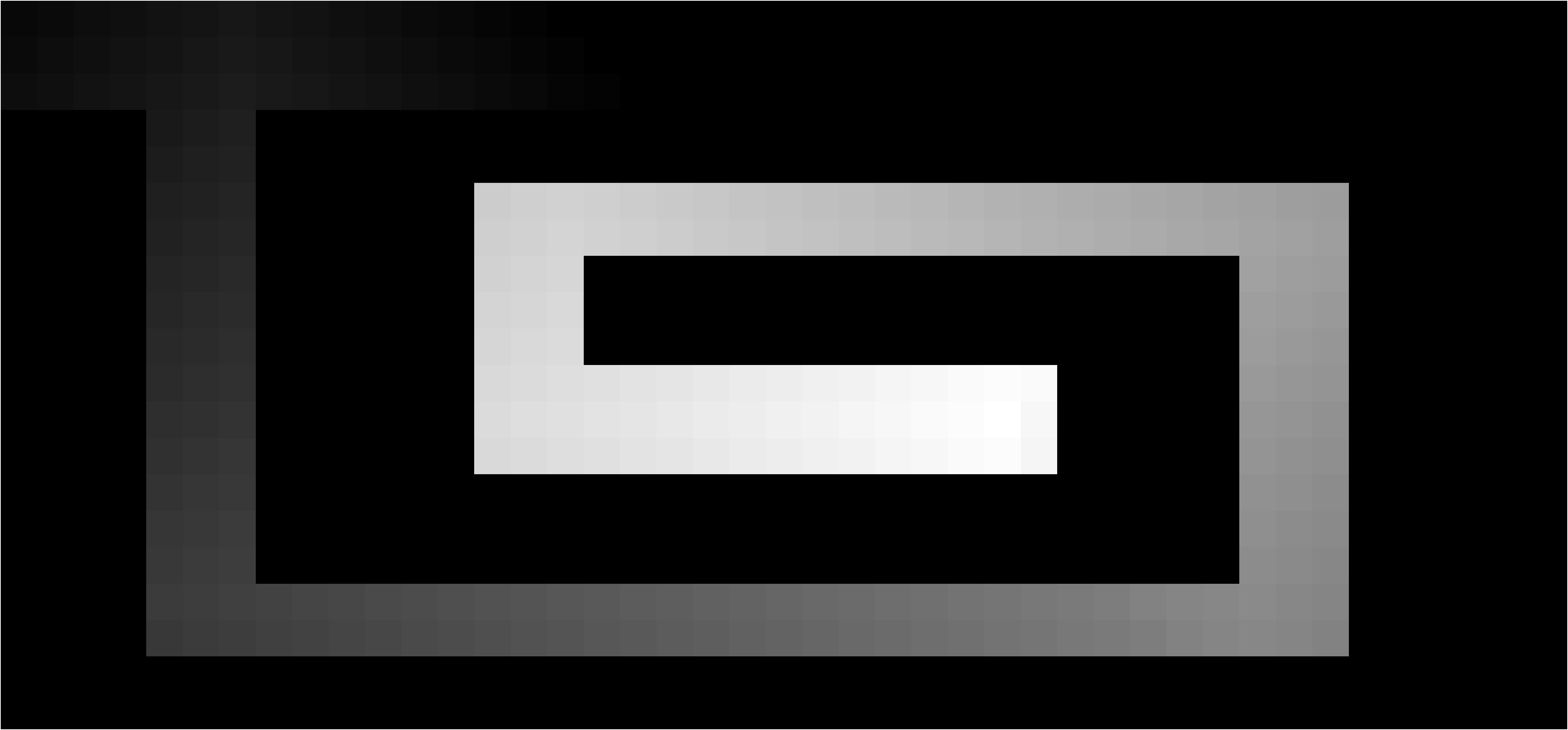} \\
    
    $\step=2$. Initial values: zero & & \\
    
    \hline
    
    \includegraphics[width=\wOne]{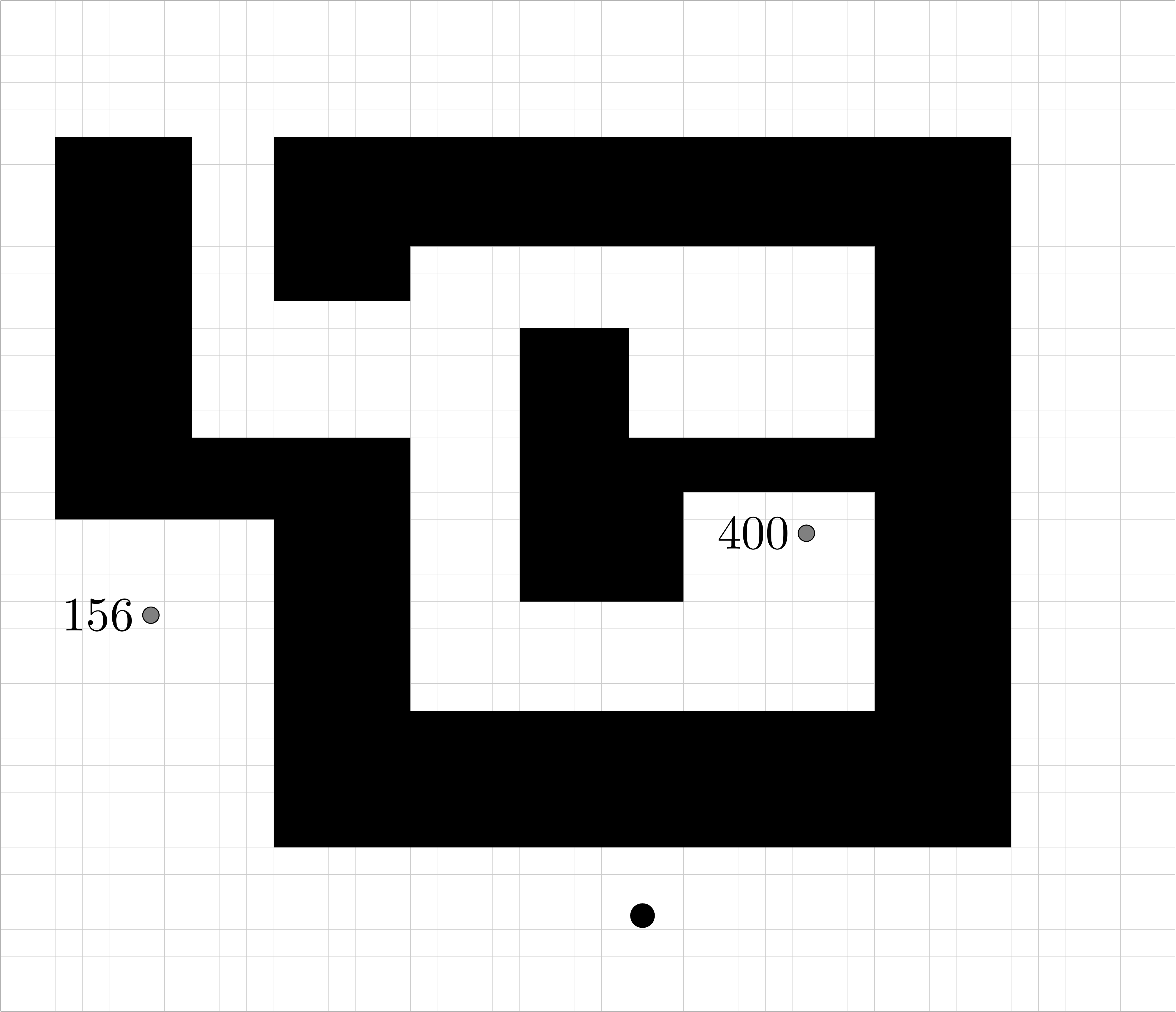} &  
    \includegraphics[width=\wOne]{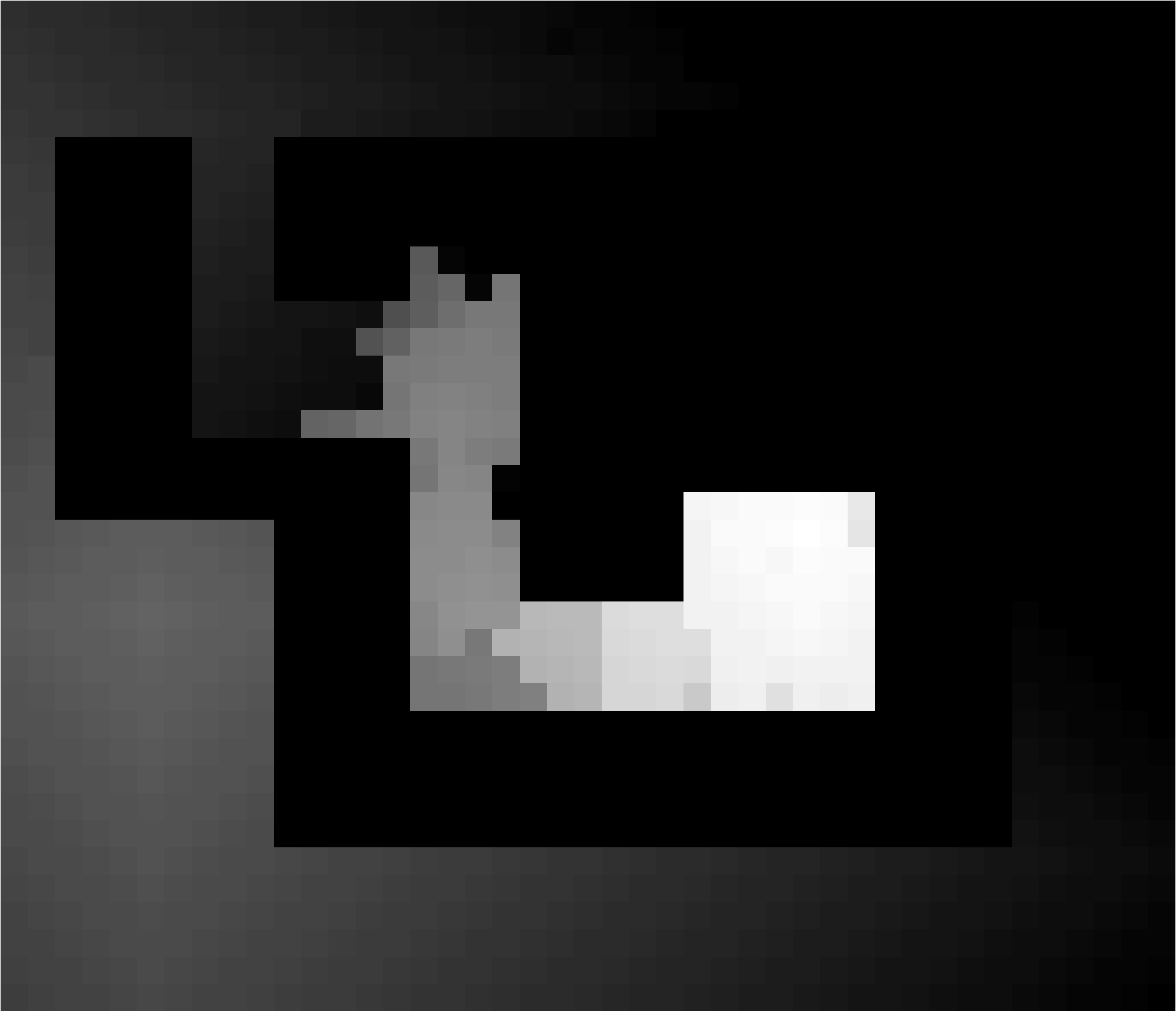} &
    \includegraphics[width=\wOne]{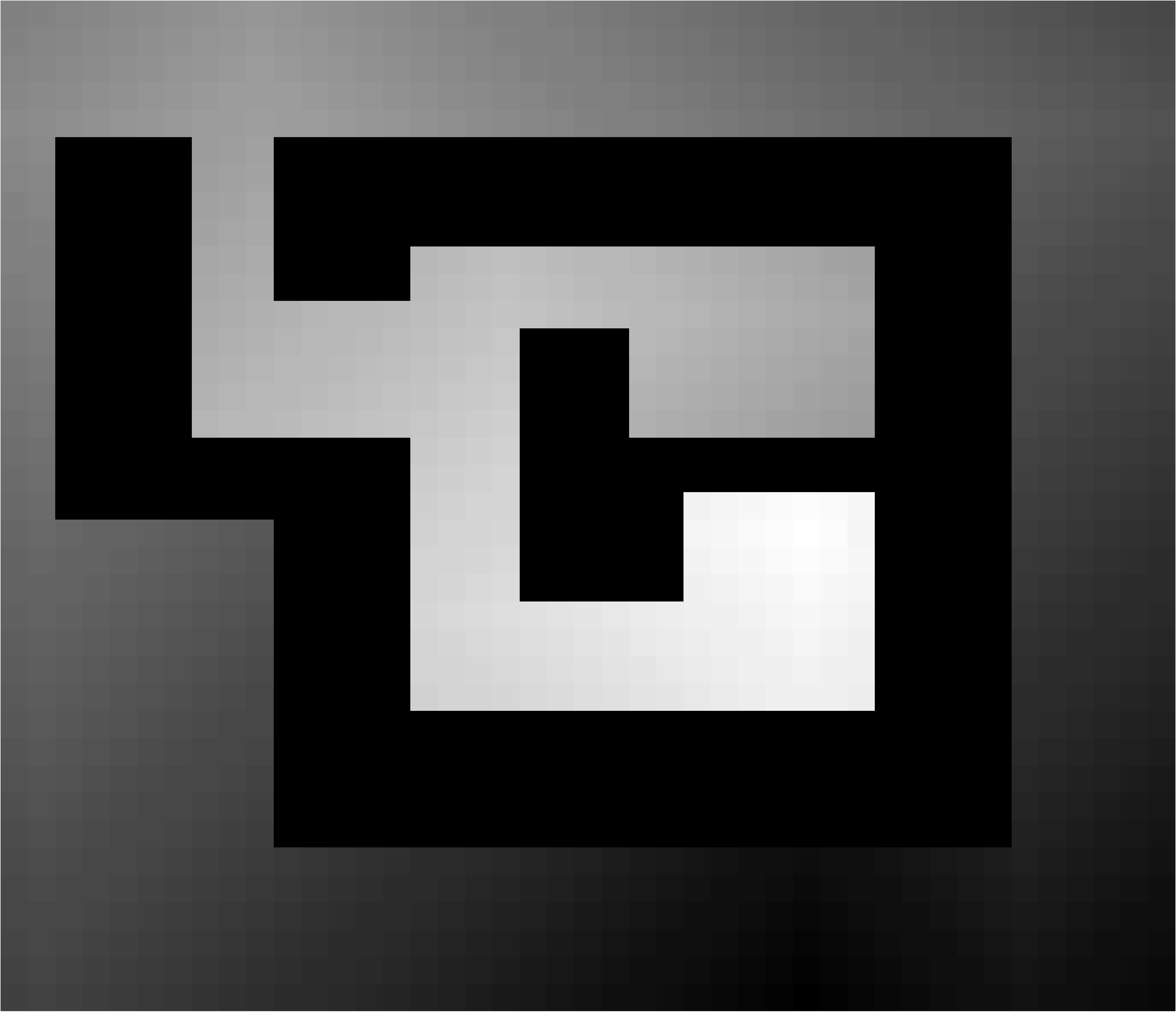} \\
    
    $\step=3$. Initial values: $[0,400]$ & & \\
    
    \hline
\end{tabular}
\end{center}
\caption[Simulation result of \ramp\ on 2D test maps]{Simulation result of \ramp\ on three 2D grid maps, that are deterministic and connected (see description in Example~\ref{ex:dc-sim}). 
To have faster convergence, we considered $\epsilon=1$.

(A) Column A shows the setup of each map. The first and second map use random initial values (for each state-action pair) in the mentioned interval. The starting location is indicated by a black dot, and the goal cells are marked with their numerical reward quantity.
We also vary $\step$ to test more circumstances.

(B) For a value function $\V$ midway the learning process, column B shows the value $\V\get{c}$ of each cell $c$, computed as the maximum over the action values for that cell. The highest values are shown as the brightest. 

(C) Column C shows the cell value when no visual changes occur anymore. Note that in the second row (with the spiral), some top-right cells converge to zero (optimal) value.}
\label{fig:dc-sim}
\end{figure}

\begin{remark}[Degree of exploration] When relating exploring runs to Algorithm~\ref{alg:full}, we would like to point out that Theorem~\ref{theo:explore} works for any $\epsilon>0$, even for very small (but nonzero) $\epsilon$ values that would make the agent seem almost entirely greedy. Therefore, the theorem might be useful for better understanding settings where a high degree of greediness (and therefore exploitation of knowledge) is preferred.\qed
\end{remark}

\begin{remark}[Liberal initialization]
    Theorem~\ref{theo:explore} applies to exploring runs that start with arbitrarily initialized value functions. This highlights a strength of the \ramp\ algorithm. In particular, the theorem seems to refute immediate simplifications of \ramp\ that would simply remember for each state-action pair the highest value seen so far.
    Such a simplification would in general require that initial values are all zero, which is not needed by Theorem~\ref{theo:explore}.
    \qed
\end{remark}

\begin{remark}[Off-policy learning]
    Theorem~\ref{theo:explore} resembles the viewpoint of Q-learning~\Qrefs\ in that the agent is updating its value estimation without necessarily using its learned knowledge to explore. Essentially, all that we require in the proof is that the system keeps running, and that each state-action pair is visited sufficiently often. This has been called off-policy learning by \citet{sutton-barto_1998}: the agent is trying to find an optimal policy (mapping states to the best actions), independently of the other policy used to explore the state space.
    \qed
\end{remark}

\begin{remark}[Not all equivalent actions]
    In a \DC\ task, in an exploring run, we can expect that the agent only rarely learns two optimal actions for the same state. Once the agent has found one optimal action $\act$ for a state $\st$, it will become more difficult (or impossible) to increase the value for another pair $(\st,\act')$. 
    Indeed, if state $\st$ has reached its optimal value, through the value of $(\st,\act)$, the value of $\st$ has become too high to have positive surprise when trying the pair $(\st,\act')$; positive surprise would correspond to $d> 0$ in Algorithm~\ref{alg:update}.
    \qed
\end{remark}

On nondeterministic tasks, the following example illustrates why exploring runs do not necessarily converge numerically (as in Theorem~\ref{theo:explore}). 
\begin{example}[Nondeterminism causes fluctuations]    \label{ex:explore-fluctuate}
    We consider the nondeterministic task in Figure~\ref{fig:explore-fluctuate}. 
    Note that the state-action pair $(1,a)$ can choose among two successors: $2$ and $3$.
    For states $2$ and $3$, the actions behave deterministically.
    For simplicity, we assume that $\step=1$ and that initial values are zero.
    In an exploring run, starting at state $1$,  the pairs $(3,a)$ and $(3,b)$ both get the value $4-\step=3$; subsequently, the pairs $(2,a)$ and $(2,b)$ get the value $3-\step=2$.
    Since $(1,b)$ always arrives at state $2$, the value of $(1,b)$ will also stabilize at $1$. However, $(1,a)$ will continue to fluctuate in value: if $(1,a)$ arrives at state $2$ then the value will be $1$, and if $(1,a)$ arrives at state $3$ then the value will be $2$.
    \qed
    \begin{figure}
    \begin{center}
        \includegraphics[width=0.7\textwidth]{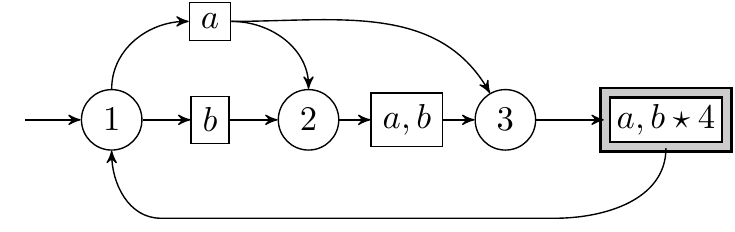}
    \end{center}
        \caption{A task, where states are represented by circles, and action applications are represented by rectangles. Start states have an inbound arrow without origin. The nonzero rewards are indicated by a shaded box. In this case, $(3,a)$ and $(3,b)$ are assigned a reward quantity of $4$.}
    \label{fig:explore-fluctuate}
    \end{figure}
\end{example}

Later, in Section~\ref{sec:greedy}, we will approach nondeterministic tasks with greedy runs instead of exploring runs; and, we will restrict attention to so-called navigation problems, that are introduced in Section~\ref{sub:shortest}.

\subsection{Optimal path following}\label{sub:sprint}
The previous Section~\ref{sub:opt} was about learning optimal values.
Here, we study the effect of value optimization on the actual behavior of the agent. 
 
\begin{definition}\label{def:sprint}
For a given task, we say that a run fragment 
\[
    \cnftupX 1
        \jump{\actX 1,\,\stX{2}}
        \ldots
        \jump{\actX{n-1},\,\stX{n}}
    \cnftupX n
        \jump{\actX n,\,\stX{n+1}}
    \cnftupX{n+1},
\] 
where $n\geq 1$, is a value-sprint if 
\begin{enumerate}
    \item \label{enu:sprint-jump} $\VX{i}\get{\stX i} \leq \VX{i}\get{\stX{i+1}}-\step$ for each $i\in\set{1,\ldots,n-1}$; and,
    \item \label{enu:sprint-end} $\VX{n}\get{\stX n}>\VX{n}\get{\stX{n+1}}-\step$.    
\end{enumerate}
\end{definition}
In a value-sprint, each transition witnesses increasingly larger values, separated by at least $\step$, except the last transition.
A value-sprint could occur anywhere in the run (not necessarily at the beginning).
We allow $n=1$, in which case condition (\ref{enu:sprint-jump}) is vacuously true.
Note that we can not split a value-sprint in smaller value-sprints: the first part would not be a value-sprint because condition (\ref{enu:sprint-end}) is not satisfied.

The following lemma relates runs and value-sprints:
\begin{lemma}\label{lem:sprint-sequence}For each task, every run is always an infinite sequence of value-sprints.
\end{lemma}
\begin{proof}Let $\task=\tasktup$ be the task.
    Suppose towards a contradiction that there is a run $\run$ that is not an infinite sequence of value-sprints. Then $\run$ is a finite sequence of value-sprints followed by an infinite tail
    \[
    \cnftupX i
        \jump{\actX i,\,\stX{i+1}}
    \cnftupX{i+1}
        \jump{\actX{i+1},\,\stX{i+2}}
    \ldots,
    \]
    where $\VX{j}\get{\stX j} \leq \VX{j}\get{\stX{j+1}}-\step$ for each $j\geq i$.
    
    Let $j\geq i$. We note that $\VX j\get{\st}\leq\VX{j+1}\get{\st}$ for each $\st\in\states$: in Algorithm~\ref{alg:update}, we can use the assumption $\VX{j}\get{\stX j} \leq \VX{j}\get{\stX{j+1}}-\step$ to obtain
    \begin{align*}
        d &= \max(\VX j\get{\stX{j+1}},\R(\stX j,\actX j)) - \step - \VX j\get{\stX j}\\
          &\geq \VX j\get{\stX{j+1}} - \step - \VX j\get{\stX j}\\
          &\geq 0.
    \end{align*}
    We observe that $\VX j\get{\stX j}<\VX{j+1}\get{\stX{j+1}}$: everything combined, we have $\VX j\get{\stX j}<\VX j\get{\stX j}+\step\leq\VX j\get{\stX{j+1}}\leq\VX{j+1}\get{\stX{j+1}}$.
    By transitivity, for any indices $j$ and $k$ with $i\leq j<k$, we have $\VX j\get{\stX j}<\VX k\get{\stX k}$. But then we would encounter infinitely many (state) values, and thus infinitely many configurations, contradicting Lemma~\ref{lem:finite-cnf}.
\end{proof}

The above value-sprint describes the following action-path, where we omit the last state $\stX{n+1}$:
\[
    p = (\stX 1,\actX 1),\ldots,(\stX n,\actX n).
\]
We say that $p$ is optimal (for $\stX 1$) if $\pval p=\optval{\stX 1}$.
We are now ready to express the effect of value optimization on the behavior of the agent:
\begin{theorem}[Follow optimal paths]
    \label{theo:sprint}
    For each \DC\ task, for each greedy run that starts with an optimal and consistent value function, each value-sprint describes an optimal action-path. 
\end{theorem}
\noindent \proofmain{theo:sprint--PROOF}

\begin{remark}[Increasingly better]    
    At moments when the agent is not exploring, and is greedily applying preferred actions, the agent is following its best guess about optimal paths.
    Although we do not know the precise moment when the agent has complete knowledge about optimal paths, we can imagine that the agent is increasingly getting better at following them.
    Theorem~\ref{theo:explore} tells us that the value function will eventually contain the knowledge about optimal paths. Then, by Theorem~\ref{theo:sprint}, any subsequent greedy fragments of the run follow optimal paths.
    Note that parameter $\epsilon$ determines the amount of time that the agent exploits its knowledge; high values for $\epsilon$ could make the agent still seem random, even if the agent has knowledge of optimal paths.
    \qed
\end{remark}

\subsubsection{Shortest paths}
\label{sub:shortest}

As an application and further explanation of Theorem~\ref{theo:explore} and Theorem~\ref{theo:sprint}, we discuss a relationship between path value and shortest paths.
See \cite{algorithms_2009} for an introduction to the shortest path problem and related algorithms.
    The standard shortest-path algorithms process graph data in bulk fashion, e.g., they can iterate over vertices and edges. A reinforcement learning system, on the other hand, builds its belief by (repeatedly) following trajectories through the transition function of the task.

We first consider the following definition.
\begin{definition}\label{def:nav}
    We say that a task $\task=\tasktup$ is a navigation problem if there is exactly one nonzero reward quantity $\M$, and with $\M > \ssize\states \step$. More precisely, for all $(\st,\act)\in\states\times\actions$ we have either $\R(\st,\act)=\M$ or $\R(\st,\act)=0$, and there is at least one $(\st,\act)$ with $\R(\st,\act)=\M$. 
\end{definition}
In a navigation problem, the intention behind the sufficiently large reward quantity is to allow the agent to learn a (cycle-free) path between any two states.

\paragraph*{\DC\ navigation problems}
A \DC\ navigation problem is a navigation problem that is also deterministic and connected.
Let $\task=\tasktup$ be a \DC\ navigation problem. Let $p=(\stX 1,\actX 1),\ldots,(\stX n,\actX n)$ be an action-path. 
We say that $p$ is rewarding if there is at least one $i\in\set{1,\ldots,n}$ with $\R(\stX i,\actX i)=\M$.
If $p$ is rewarding then we define the length of $p$, denoted $\plen p$, as the smallest $i\in\set{1,\ldots,n}$ with $\R(\stX i,\actX i)=\M$. 
Note that if a rewarding action-path contains cycles then we can transform it into a rewarding action-path without cycles, using the procedure at the beginning of Section~\ref{sub:opt}. 

Thanks to connectedness, there is a cycle-free action-path from each state $\st$ to a state-action pair $(\st',\act')$ with $\R(\st',\act')=\M$. Recalling the definition of path value from Equation~\eqref{eq:pval}, cycle-free rewarding action-paths have a strictly positive value because $\M>\ssize\states\step$: 
    on a cycle-free path, each pair $(\stX i,\actX i)$ contributes at least $\clamp{\R(\stX i,\actX i)-\ssize\states\step}$ to the overall path value; if one of the pairs is rewarding then the overall path value is strictly positive.
Hence, each state $\st$ has $\optval\st > 0$.
We therefore consider \DC\ navigation problems to be solvable from a path finding viewpoint.

\paragraph*{Shortest path following}
On \DC\ navigation problems, Theorem~\ref{theo:explore} tells us that every exploring run will eventually find an optimal and consistent value function, containing for each state the knowledge of the optimal paths.
Theorem~\ref{theo:sprint} has the following corollary:
\begin{corollary}[Follow shortest paths]
    \label{corol:shortest}
    For all \DC\ navigation problems, for each greedy run that starts with an optimal and consistent value function,     
    each value-sprint follows a shortest path to reward.
\end{corollary}
\begin{proof}    
    Consider a greedy run starting with an optimal and consistent value function. By Lemma~\ref{lem:sprint-sequence}, the greedy run is an infinite sequence of value-sprints. Moreover, by Theorem~\ref{theo:sprint}, each value-sprint follows an optimal path. We are left to argue that those paths are actually the shortest (to reward).
    Since all nonzero reward occurrences have amplitude $\M$, for any two cycle-free rewarding action-paths $p_1$ and $p_2$, if $\plen{p_1}<\plen{p_2}$ then $\pval{p_1}>\pval{p_2}$. This is illustrated in Figure~\ref{fig:shortest-path}.
    Therefore, whenever we follow an optimal action-path $p$ from a state $\st$, we know that $p$ has the shortest length among all those paths that lead from $\st$ to reward.
\end{proof}

\begin{figure}
    \begin{center}
    \includegraphics[width=0.9\textwidth]{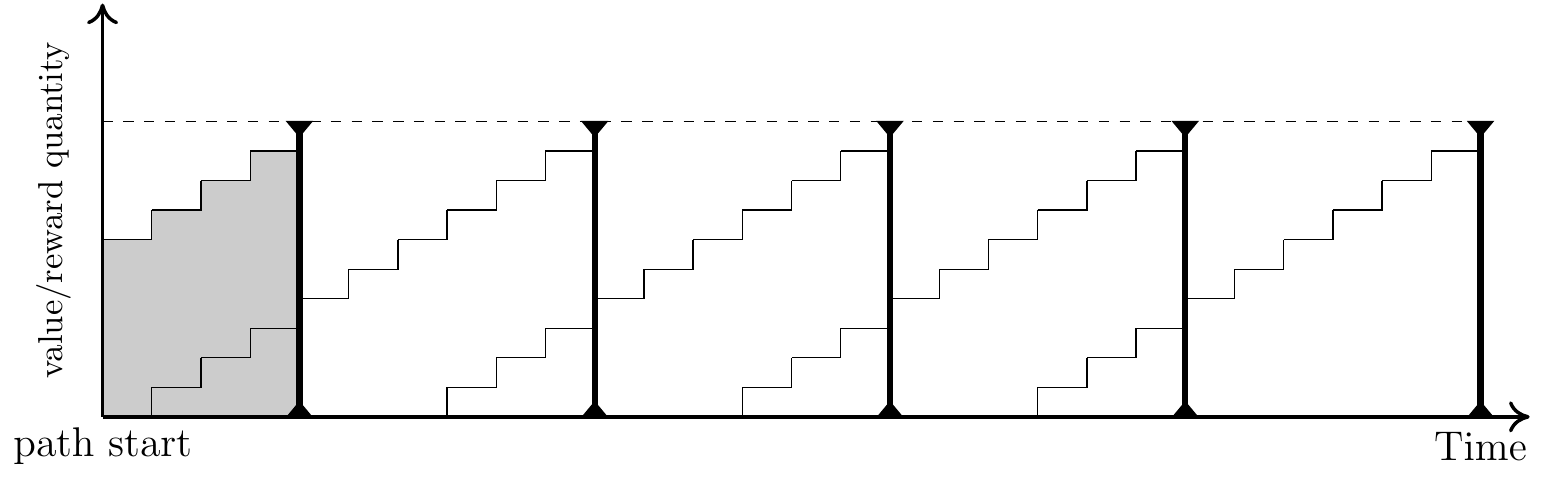}
    \end{center}
    \caption{When all rewards have the same magnitude, the action-path value is determined only by the distance (or time) to reward. In the figure, we have shaded the ramp of the reward occurrence that determines the path value; it must be the first reward occurrence.}
    \label{fig:shortest-path}
\end{figure}

After the above introduction to navigation problems, we may proceed to the topic of greedy navigation in Section~\ref{sec:greedy}.

\subsection{Proof of Theorem~\ref{theo:explore}}
\label{theo:explore--PROOF}

Let us fix some \DC\ task $\task=\tasktup$. 

\subsubsection{Approach}
We first define an auxiliary notion.
Let $\V$ be a value function. We say that $\V$ is valid when for each $(\st,\act)\in\states\times\actions$, we have
\[
    \V(\st,\act) \leq \clamp{\max(\V\get{\st'},\R(\st,\act)) - \step}.
\]
Intuitively, this means that the values are not overestimating the true reward.

We will show in Section~\ref{sub:explore-rest} that every exploring run $\run$ has an infinite suffix $\run'$ in which each value function is both valid and optimal. By Property~\ref{prop:explore-opt-valid-implies-consistent} (below), all value functions in $\run'$ are also consistent, as desired.

\begin{property}\label{prop:explore-opt-valid-implies-consistent}
    Let $\V$ be a value function. If $\V$ is valid and optimal then $\V$ is consistent.
\end{property}
\begin{proof}
    Let $\st\in\states$ and $\act\in\pref\st\V$. Denote $\Det\st\act{\st'}$.
    We show $\V\get\st \geq \clamp{\max(\V\get{\st'},\R(\st,\act))-\step}$. Then the validity assumption $\V(\st,\act)\leq \clamp{\max(\V\get{\st'},\R(\st,\act))-\step}$, combined with $\V\get\st=\V(\st,\act)$ by $\act\in\pref\st\V$, implies the desired consistency
    \[
        \V\get\st = \clamp{\max(\V\get{\st'},\R(\st,\act))-\step}.
    \]    
    By Lemma~\ref{lem:opt} (below),
    \begin{align*}
        \optval\st &\geq \max(\clamp{\optval{\st'}-\step},\clamp{\R(\st,\act)-\step})\\
                &=\clamp{\max(\optval{\st'},\R(\st,\act))-\step}.
    \end{align*}
    Subsequently, by the optimality assumption on $\V$, which gives $\V\get{\st''}=\optval{\st''}$ for each $\st''\in\states$, we have $\V\get\st\geq\clamp{\max(\V\get{\st'},\R(\st,\act))-\step}$.
\end{proof}

\begin{lemma}\label{lem:opt}
    Let $\task=\tasktup$ be a deterministic task. 
    Let $(\st,\act)\in\states\times\actions$, and denote $\Det\st\act{\st'}$. We always have
    \[
        \optval\st \geq \max(\clamp{\optval{\st'}-\step},\clamp{\R(\st,\act)-\step}).
    \]
\end{lemma}
\begin{proof}
    We have $\optval\st\geq\clamp{\R(\st,\act)-\step}$ because $(\st,\act)$ is an action-path (of length one) for $\st$. Also, $\optval\st\geq\clamp{\optval{\st'}-\step}$ because any optimal action-path $p'$ for $\st'$ can be extended to an action-path for $\st$ by adding $(\st,\act)$ to the front; adding $(\st,\act)$ to the front pushes the state-action pairs of $p'$ one step further into the future, leading to an overall value decrease with $\step$.%
\end{proof}

\subsubsection{Obtain validity and optimality}\label{sub:explore-rest}
Henceforth, we fix an exploring run $\run$:
\[
    \cnftupX 1
        \jump{\actX 1,\,\stX{2}}
    \cnftupX 2
        \jump{\actX 2,\,\stX{3}}
    \ldots
\]
We show the existence of an infinite suffix $\run'$ in which all value functions are both valid and optimal.

By Property~\ref{prop:explore-obtain-valid} (below), we know that in $\run$ we eventually encounter a configuration $\cnftupX j$ where $\VX j$ is valid and all subsequent value functions are also valid. 
\begin{property}\label{prop:explore-obtain-valid}
    In run $\run$, eventually all encountered value functions are valid.
    \proofapp{prop:explore-obtain-valid--PROOF}
\end{property}

Subsequently, Property~\ref{prop:explore-preserve-valid} (below) tells us that after configuration $\cnftupX j$, state values do not decrease.

\begin{property}\label{prop:explore-preserve-valid}
    Consider a transition $\cnftupX i\jump{\actX i,\,\stX{i+1}}\cnftupX{i+1}$. If $\VX i$ is valid then for each $\st\in\states$ we have $\VX{i}\get{\st}\leq\VX{i+1}\get{\st}$.
\end{property}
\begin{proof}
    Let $\st\in\states$. If $\st\neq\stX i$ then $\VX{i+1}\get\st=\VX i\get\st$.
    Henceforth, suppose $\st=\stX i$.
    If $\actX i\notin\pref{\stX i}{\VX i}$ then there is some $\act'\in\pref{\stX i}{\VX i}$ with 
    \[
        \VX i\get{\stX i}=\VX i(\stX i,\act')=\VX{i+1}(\stX i,\act')\leq\VX{i+1}\get{\stX i}.
    \]
    
    Suppose $\actX i\in\pref{\stX i}{\VX i}$, giving $\VX i\get{\stX i} = \VX i(\stX i,\actX i)$. By Algorithm~\ref{alg:update},
    \begin{align*}
        \VX{i+1}(\stX i,\actX i)
                & = \clamp{\VX{i}(\stX i,\actX i) + \max(\VX i\get{\stX{i+1}},\R(\stX i,\actX i)) - \step- \VX i\get{\stX i}}\\
                &=\clamp{\VX{i}\get{\stX i} + \max(\VX i\get{\stX{i+1}},\R(\stX i,\actX i)) - \step- \VX i\get{\stX i}}\\
                &=\clamp{\max(\VX i\get{\stX{i+1}},\R(\stX i,\actX i)) - \step}.
    \end{align*}
    By validity, $\clamp{\max(\VX i\get{\stX{i+1}},\R(\stX i,\actX i)) - \step}\geq \VX i(\stX i,\actX i)$.
    Overall, 
    \[
        \VX i\get{\stX i}=\VX i(\stX i,\actX i)\leq\VX{i+1}(\stX i,\actX i)\leq\VX{i+1}\get{\stX i}.
    \]
\end{proof}

We summarize what we have so far: 
\begin{itemize}
    \item We eventually reach a configuration $\cnftupX j$ where $\VX j$ is valid.
    \item After $\cnftupX j$, value functions remain valid and state values do not decrease.
\end{itemize}
After $\cnftupX j$, each state must eventually stop changing its value. Otherwise, since the only change to a state value would be a strict increment, we would see infinitely many state values, and thus infinitely many configurations (contradicting Lemma~\ref{lem:finite-cnf}).
Therefore, somewhere after $\cnftupX j$, there is an infinite suffix $\run'$ in which state values no longer change.
Let $\cnftupX k$ denote the first configuration of $\run'$.
In the rest of the proof, we show that $\VX k$ is optimal. Hence, all value functions in $\run'$ turn out to be optimal. 
Overall, all value functions in $\run'$ are both valid and optimal, as desired.

Abbreviate $\V=\VX k$.
Towards a contradiction, suppose $\V$ is not optimal. Since $\V$ is valid, by Property~\ref{prop:explore-valid-leq-opt} (below) we know that $\V\get\st\leq\optval\st$ for each $\st\in\states$. So, if $\V$ is not optimal, there is at least one $\st\in\states$ with $\V\get\st<\optval\st$.

\begin{property}\label{prop:explore-valid-leq-opt}
    Let $\V$ be a value function. If $\V$ is valid then for each $\st\in\states$ we have
    $
        \V\get\st \leq \optval\st.
    $
    \proofapp{prop:explore-valid-leq-opt--PROOF}
\end{property}

Consider the set
\[
    \sub\V = \set{\st\in\states\mid\V\get\st<\optval\st}.
\]
We select one state $\st\in\sub\V$ with the highest optimal value, i.e., $\st$ satisfies
\[
    \optval\st = \max\set{\optval{\st'}\mid\st'\in\sub\V}.
\]
We show that after $\cnftupX k$, which is the first configuration of suffix $\run'$, the value of $\st$ strictly increases; this would be a contradiction by choice of $\run'$.

By definition of optimal value, there is an action-path $p$ starting at $\st$, with $\optval\st = \pval p$. 
Let $(\st,\act)$ be the first pair of $p$, and denote $\Det\st\act{\st'}$.
By Property~\ref{prop:inf-visit} (below) we execute $(\st,\act)$ infinitely often in run $\run$.
\begin{property}\label{prop:inf-visit}
    In each exploring run, for each $(\st,\act)\in\states\times\actions$ there are infinitely many transitions in which we execute the pair $(\st,\act)$.
\end{property}
\begin{proof}
    Since there are a finite number of configurations (by Lemma~\ref{lem:finite-cnf}), we can consider a configuration $(\alt{\stX 1},\alt{\VX 1})$ that occurs infinitely often in the run.
    By connectedness of the task, there is a path in the state space
    \[
        \alt{\stX 1}\edge{\alt{\actX 1}}\alt{\stX 2}\ldots\alt{\stX n}\edge{\alt{\actX{n}}}\alt{\stX{n+1}},
    \]
    where $n\geq 1$ and $\alt{\stX{n+1}}=\st$. By the built-in fairness assumption of exploring runs (see Section~\ref{sub:runs}), we infinitely often follow $(\alt{\actX 1},\alt{\stX 2})$ from configuration $(\alt{\stX 1},\alt{\VX 1})$. This results in a configuration $(\alt{\stX 2},\alt{\VX 2})$ that also occurs infinitely often. The reasoning can be repeated to arrive at a configuration $(\alt{\stX{n+1}},\alt{\VX{n+1}})$ with $\alt{\stX{n+1}}=\st$ that occurs infinitely often. The reasoning can now be applied one more time. Denoting $\Det\st\act{\st'}$, from configuration $(\alt{\stX{n+1}},\alt{\VX{n+1}})$ we infinitely often follow $(\act,\st')$.
\end{proof}
    
Because the part of $\run$ before $\run'$ is finite, the pair $(\st,\act)$ is executed infinitely often in $\run'$. So, there are infinitely many transitions in $\run'$ of the following form:
\[
    \transX h,
\]
where $(\stX h,\actX h)=(\st,\act)$. We show below that $\VX h(\st,\act)<\VX{h+1}(\st,\act)$. So, as long as the value of $\st$ stays strictly below $\optval\st$ we can strictly increase the value of $(\st,\act)$. This always leads to a moment where the value of $\st$ in its entirety is strictly increased. This is the sought contradiction.

We are left to show $\VX h(\st,\act)<\VX{h+1}(\st,\act)$. Note that $\stX{h+1}=\st'$ by determinism. Also, by assumption on the unchanging values in $\run'$, we have (1) $\VX h\get\st=\V\get\st<\optval\st$ and (2) $\VX h\get{\st'}=\V\get{\st'}$. 
In Algorithm~\ref{alg:update}, the value change during the above transition is
\begin{align}
    d   &= \max(\VX h\get{\stX{h+1}},\R(\st,\act))-\step-\VX h\get\st \nonumber\\
        &= \max(\V\get{\st'},\R(\st,\act))-\step-\V\get\st.\label{eq:d}
\end{align}    
It suffices to show that $d> 0$.
We recall from earlier the action path $p$ starting at $\st$, with $\optval\st = \pval p$.
By Property~\ref{prop:explore-path-cases} (below) we have two cases: either $\optval\st = \clamp{\R(\st,\act)-\step}$ or $\optval\st=\optval{\st'}-\step$.

\begin{property}\label{prop:explore-path-cases}
    Let $\st\in\states$ and let $p$ be an action-path for $\st$ with $\pval p=\optval\st$. Let $(\st,\act)$ be the first pair of $p$, and denote $\Det\st\act{\st'}$.
    We have either 
    \begin{itemize}
        \item $\optval\st=\clamp{\R(\st,\act)-\step}$; or,
        \item $\optval\st=\optval{\st'}-\step$.
    \end{itemize}
    \proofapp{prop:explore-path-cases--PROOF}
\end{property}
\noindent
We consider each case in turn.
    \paragraph{First case}
    Suppose $\optval\st = \clamp{\R(\st,\act)-\step}$. Since $0\leq\V\get\st<\optval\st$, we have $\optval\st>0$, so we write more simply $\optval\st = \R(\st,\act)-\step$.
    Using Equation~\eqref{eq:d}, 
    \begin{align*}
        d   &= \max(\V\get{\st'},\R(\st,\act))-\step-\V\get\st \\
            &\geq \R(\st,\act)-\step-\V\get\st\\
            &=\optval\st - \V\get\st.
    \end{align*}   
    Since $\V\get\st<\optval\st$ by assumption, we obtain $d> 0$.
    
    \paragraph{Second case}
    Suppose $\optval\st=\optval{\st'}-\step$. Since $\step\geq 1$, we have $\optval\st<\optval{\st'}$. Necessarily $\st'\notin\sub\V$, because otherwise $\optval\st<\max\set{\optval{\st''}\mid\st''\in\sub\V}$, which is false by choice of $\st$. Therefore $\V\get{\st'}=\optval{\st'}$. We now complete the reasoning, continuing from Equation~\eqref{eq:d}:
    \begin{align*}
        d   &= \max(\V\get{\st'},\R(\st,\act))-\step-\V\get\st \\
            &\geq \V\get{\st'}-\step-\V\get\st\\
            &=\optval{\st'}-\step - \V\get\st\\
            &=\optval\st - \V\get\st.
    \end{align*}   
    Like in the previous case, since $\V\get\st<\optval\st$ by assumption, we obtain $d> 0$.

\subsection{Proof of Corollary~\ref{corol:fixed}}
\label{corol:fixed--PROOF}

\newcommand{\preff}[1]{\fname{pref}(#1)}


Let $\task=\tasktup$ be a \DC\ task. Let $\run$ be an exploring run, and let $\run'$ be the infinite suffix where all value functions are both optimal and consistent, as given by Theorem~\ref{theo:explore}.
We show that in $\run'$ the value function eventually becomes fixed. 

By Property~\ref{prop:final} (below), we know that for each state the set of preferred actions is fixed throughout $\run'$. For each $\st\in\states$, let $\preff\st\subseteq\actions$ denote the final set of actions preferred by $\st$ in $\run'$.
Also by Property~\ref{prop:final}, for each $(\st,\act)\in\states\times\actions$, if $\act\notin\preff\st$, we know that the value of the non-preferred pair $(\st,\act)$ can never be increased in $\run'$.
Therefore, the value of non-preferred state-action pairs becomes constant.

Hence, there is a suffix $\run''$ of $\run'$ in which states always prefer the same actions and in which the value of non-preferred state-action pairs is constant. We show that all value functions in $\run''$ are the same. Thereto, let us consider two configurations $\cnftupX i$ and $\cnftupX j$ in $\run''$. We show for each $(\st,\act)\in\states\times\actions$ that $\VX i(\st,\act)=\VX j(\st,\act)$. We distinguish between the following cases:
\begin{itemize}
    \item Suppose $\act\notin\preff{\st}$. Then $\VX i(\st,\act)=\VX j(\st,\act)$ by choice of $\run''$.
    
    \item Suppose $\act\in\preff{\st}$.
    Then $\act\in\pref{\st}{\VX i}$ and $\act\in\pref{\st}{\VX j}$. Subsequently, $\VX i(\st,\act)=\VX i\get\st$ and $\VX j(\st,\act)=\VX j\get\st$. Moreover, by optimality of $\VX i$ and $\VX j$, 
    \[
        \VX i\get\st = \optval\st =\VX j\get\st.
    \]
    Overall, $\VX i(\st,\act)=\VX j(\st,\act)$.
\end{itemize}
 
\begin{property}\label{prop:final}
Consider a transition in $\run'$,
\[
    \transX i.
\]
We have
\begin{enumerate}
    \item $\pref{\stX i}{\VX{i+1}}=\pref{\stX i}{\VX i}$;
    \item If $\actX i\notin\pref{\stX i}{\VX i}$ then $\VX{i+1}(\stX i,\actX i)\leq\VX i(\stX i,\actX i)$.
\end{enumerate}
Note: only the preferred actions of $\stX i$ could change; hence, for each $\st\neq\stX i$ we have $\pref{\st}{\VX{i+1}}=\pref{\st}{\VX i}$.
\end{property}
\begin{proof} We distinguish between two cases, depending on whether $\actX i$ is preferred or not.

\paragraph*{Preferred action}
If $\actX i\in\pref{\stX i}{\VX i}$ then consistency of $\VX i$ allows us to apply Lemma~\ref{lem:consistent} (below) to know $\VX{i+1}=\VX{i}$, i.e., executing preferred actions does not modify the value function. Hence, $\pref{\stX i}{\VX{i+1}}=\pref{\stX i}{\VX i}$. 

\paragraph*{Non-preferred action}
Suppose $\actX i\notin\pref{\stX i}{\VX i}$. By Algorithm~\ref{alg:update}, we have $\VX{i+1}(\stX i,\actX i)=\clamp{\VX i(\stX i,\actX i) + d}$ where 
\[
    d=\max(\VX{i}\get{\stX{i+1}},\R(\stX i,\actX i))-\step -\VX{i}\get{\stX i}.
\]
We show below that $d\leq 0$, which implies $\VX{i+1}(\stX i,\actX i)\leq\VX i(\stX i,\actX i)$. Moreover, $\pref{\stX i}{\VX{i+1}}=\pref{\stX i}{\VX i}$: since $\actX i\notin\pref{\stX i}{\VX i}$, the only way to change the set of preferred actions would be to make $\actX i$ into a preferred action by a strict value increase (which does not happen).

We are left to show $d\leq 0$. To start, we use that $x\leq\clamp{x}$ for each integer $x$; hence,
\[
    d \leq \clamp{\max(\VX{i}\get{\stX{i+1}},\R(\stX i,\actX i))-\step} -\VX{i}\get{\stX i}.
\]
Next, by optimality of $\VX i$,
\[
    d \leq \clamp{\max(\optval{\stX{i+1}},\R(\stX i,\actX i))-\step} -\optval{\stX i}.
\]
Subsequently, the clamped part on the right-hand side can be simplified with Lemma~\ref{lem:opt}, to obtain
\[
    d \leq \optval{\stX i} -\optval{\stX i} = 0.
\]
\end{proof}

\begin{lemma}\label{lem:consistent}
    For each deterministic task, for each transition    
    \[
        (\st,\V)
            \jump{\act,\st'}
        (\st',\V'),
    \]
    if $\V$ is consistent and $\act\in\pref\st\V$ then $\V'=\V$.
\end{lemma}
\begin{proof}
    By consistency of $\V$, we have
    \[
        \V\get\st=\clamp{\max(\V\get{\st'},\R(\st,\act)) - \step}.
    \]
    Next, we look at Algorithm~\ref{alg:update}. There are two cases:
    \begin{itemize}
    \item Suppose $\max(\V\get{\st'},\R(\st,\act)) - \step\geq 0$. Then $\V\get\st=\max(\V\get{\st'},\R(\st,\act)) - \step$. Hence, $d=0$ and surely $\V'=\V$.
    
    \item Suppose $\max(\V\get{\st'},\R(\st,\act)) - \step< 0$. Then $\V\get\st=0$, making $\V(\st,\act)=0$. Also, $d<0$, causing 
    \[
        \V'(\st,\act)=\clamp{\V(\st,\act)+d}=\clamp{d}=0=\V(\st,\act).
    \]        
    Overall, $\V'=\V$.
    \end{itemize}
\end{proof}

\subsection{Proof of Theorem~\ref{theo:sprint}}
\label{theo:sprint--PROOF}

    Let $\task=\tasktup$ be the \DC\ task. Consider a greedy run that starts with a value function $\V$ that is both optimal and consistent.
    We recall by Lemma~\ref{lem:sprint-sequence} that the run is an infinite sequence of value-sprints.    
    Moreover, by Lemma~\ref{lem:consistent}, since the run is greedy, every configuration in the run uses the value function $\V$.
    
    Consider an arbitrary value-sprint in the run:    
    \[
    (\stX 1, \V)
        \jump{\actX 1,\,\stX{2}}
        \ldots
        \jump{\actX{n-1},\,\stX{n}}
    (\stX n, \V)
        \jump{\actX n,\,\stX{n+1}}
    (\stX{n+1}, \V),
    \]    
    where $n\geq 1$.
    Here, $(\stX 1,\V)$ is not necessarily the first configuration of the run; it could be anywhere in the run.
    The corresponding action-path is
    \[
        p = (\stX 1,\actX 1),\ldots,(\stX n,\actX n).
    \]
    We show $\pval p=\optval{\stX 1}$.
    For each $i\in\set{1,\ldots,n}$, we define the suffix 
    \[
        p_i=(\stX i,\actX i),\ldots,(\stX n,\actX n).
    \]
    Note that $p_1=p$.
    Below, we show by induction on $i=n,\ldots,1$ that $\V\get{\stX i}=\pval{p_i}$.
    This eventually gives $\V\get{\stX 1}=\pval{p_1}=\pval p$. Subsequently, since $\V\get{\stX 1}=\optval{\stX 1}$ by optimality of $\V$, we obtain $\pval p=\optval{\stX 1}$, as desired.
           
    \paragraph*{Base case}
    By definition, $\pval{p_n}=\clamp{\R(\stX n,\actX n)-\step}$.
    If $\V\get{\stX n}=0$ then $\clamp{\max(\V\get{\stX{n+1}}, \R(\stX n,\actX n)) - \step} = 0$ by consistency of $\V$, enforcing $\R(\stX n,\actX n) \leq \step$. In that case, $\pval{p_n}=0=\V\get{\stX n}$.        
    
    Henceforth, we assume $\V\get{\stX n}> 0$. The consistency property for $\stX n$ may now be written as
    \begin{equation}
        \V\get{\stX n}=\max(\V\get{\stX{n+1}}, \R(\stX n,\actX n)) - \step.\label{eq:basecase-sprint-end}
    \end{equation}
    Necessarily $\V\get{\stX{n+1}}<\R(\stX n,\actX n)$; otherwise, Equation~\eqref{eq:basecase-sprint-end} becomes $\V\get{\stX n}=\V\get{\stX{n+1}} - \step$, which violates condition~\ref{enu:sprint-end} in the definition of value-sprint (Definition~\ref{def:sprint}), namely, $\V\get{\stX{n}}>\V\get{\stX{n+1}}-\step$.
    Hence, Equation~\eqref{eq:basecase-sprint-end} becomes
    \[
        \V\get{\stX n}= \R(\stX n,\actX n) - \step.
    \]
    Since $\V\get{\stX n}>0$, we may write $\V\get{\stX n}= \clamp{\R(\stX n,\actX n) - \step}$. Thus $\V\get{\stX n}=\pval{p_n}$.
    
    \paragraph*{Inductive step}
    If $n=1$ then no inductive step is needed. Henceforth, assume $n\geq 2$.
    Let $i\in\set{2,\ldots,n}$. We assume as induction hypothesis that $\V\get{\stX i}=\pval{p_i}$. 
    By Lemma~\ref{lem:pval} (below), we have
    \[
        \pval{p_{i-1}}=
            \max(\clamp{\R(\stX{i-1},\actX{i-1})-\step}, \clamp{\pval{p_{i}}-\step}).
    \]
    By subsequently applying the induction hypothesis $\V\get{\stX{i}}=\pval{p_i}$, we get
    \begin{align*}
        \pval{p_{i-1}} 
            &= \max(\clamp{\R(\stX{i-1},\actX{i-1})-\step},\clamp{\V\get{\stX i}-\step}
            ) \\
            & = \clamp{\max(\V\get{\stX i},\R(\stX{i-1},\actX{i-1})) - \step}.
    \end{align*}
    Since $\actX{i-1}\in\pref{\stX{i-1}}\V$ by greediness of the value-sprint, the last line equals $\V\get{\stX{i-1}}$ by consistency.
    Hence, $\V\get{\stX{i-1}}=\pval{p_{i-1}}$.

\begin{lemma}\label{lem:pval}
    Consider a deterministic task with reward function $\R$. Let $p=(\stX 1,\actX 1),(\stX 2,\actX 2),\ldots,(\stX n,\actX n)$ be an action-path with $n\geq 2$.
    Let $p'=(\stX 2,\actX 2),\ldots,(\stX n,\actX n)$ be the suffix of $p$ after removing the first pair $(\stX 1,\actX 1)$. We have
    \[
        \pval{p} = \max(\clamp{\R(\stX 1,\actX 1)-\step}, \clamp{\pval{p'}-\step}).
    \]
    \proofapp{lem:pval--PROOF}
\end{lemma}

\section{Greedy navigation under nondeterminism}
\label{sec:greedy}

As suggested by Example~\ref{ex:explore-fluctuate}, optimality is not well-defined for nondeterministic tasks, due to persistent value fluctuations.
Therefore, as a measure of agent quality in nondeterministic tasks, we propose to avoid rewardless cycles in the state space.
Avoiding cycles is a constraint on the time budget to reach reward.
    This could be useful, for example in animals, when reward is associated with survival.
Exploring runs, however, might repeatedly lead the agent into rewardless cycles.
In this section, we show the usefulness of a purely greedy approach to avoid rewardless cycles in nondeterministic navigation problems. 

We recall the definition of navigation problems from Definition~\ref{def:nav} in Section~\ref{sub:shortest}.

\subsection{Greedy navigation}
We recall from Section~\ref{sub:runs} that in a greedy run the agent is constantly following preferred actions, without exploring other possibilities. This corresponds to setting $\epsilon=0$ in Algorithm~\ref{alg:full}. We emphasize that a random action is chosen from the preferred actions. This reflects that the agent deems all preferred actions as equally desirable. The agent can only behave purely randomly on a state when it prefers no actions on that state.

The following Example~\ref{ex:cycle} motivates the use of greedy runs.
\begin{example}\label{ex:cycle}    
    Consider the task in Figure~\ref{fig:cycle}.
    The reward function assigns a nonzero reward only to the pair $(2,b)$. 
    Exploring runs do not try to avoid cycles, and therefore the agent could witness very long cycles without reward if at state $2$ the action $a$ is selected several times in succession. This suggests to use greedy runs as a possible way to eventually avoid cycles.
    \qed 
\end{example}
    
    \begin{figure}
    \begin{center}
        \includegraphics[width=0.7\textwidth]{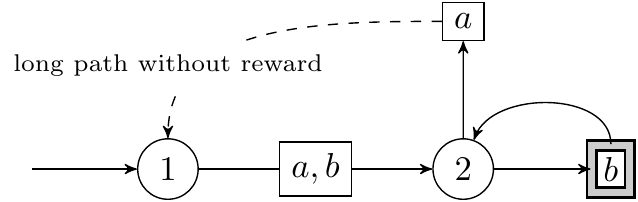}
    \end{center}
        \caption{A task with start state $1$, and at least one other state $2$. The available actions are $a$ and $b$. The only nonzero reward is assigned to $(2,b)$, indicated by a shaded box.}
    \label{fig:cycle}
    \end{figure}

An important assumption in navigation problems, as defined in Section~\ref{sub:shortest}, is that the reward is large enough to bridge large distances in the state space. The following Example~\ref{ex:small-reward} illustrates why the greedy approach sometimes fails to avoid cycles when reward is too small. We will therefore restrict attention to navigation problems, where the issue of small reward does not occur.

\begin{example}\label{ex:small-reward}
    We consider the task in Figure~\ref{fig:small-reward}.
    Suppose for simplicity that $\step=1$.
    If a greedy run starts with a zero value function, and we would perform $(3,a)$ before $(3,b)$ then the value of $(3,a)$ will become $1$.
    Subsequently, by greediness, action $a$ will be executed whenever the agent visits state $3$. Value $1$ is however too small to be propagated towards $(2,b)$, and the agent will remain stuck with a value of zero for both $(2,a)$ and $(2,b)$. This way, actions $a$ and $b$ are both preferred in state $2$, possibly causing the greedy run to witness long cycles without reward if action $a$ would be chosen at state $2$ several times in succession.
    
    Some greedy runs, however, will perform at least twice $(3,b)$ before $(3,a)$. This causes $(3,b)$ to be assigned a value of $2$, which in turn causes $(2,b)$ to be assigned value $1$. In that scenario, the agent will henceforth never witness cycles without reward.
    \qed
\end{example}    

    \begin{figure}
    \begin{center}
    \includegraphics[width=0.7\textwidth]{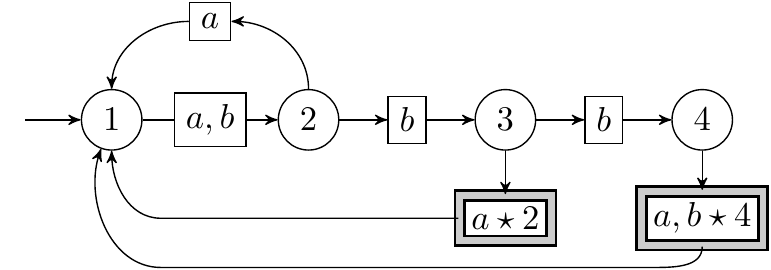}
    \end{center}
        \caption{Using the same graphical notation as in Figure~\ref{fig:cycle}, this task has the following nonzero rewards: $(3,a)\mapsto 2$, $(4,a)\mapsto 4$, and $(4,b)\mapsto 4$. The rewarding state-action pairs are indicated by a shaded box, and the reward quantity is indicated by a star.}
    \label{fig:small-reward}
    \end{figure}

\subsection{Reducibility}
\label{sub:reduce}
Let $\task=\tasktup$ be a navigation problem, with nonzero reward $\M$.
To make assumptions about nondeterminism, we formalize a notion called  reducibility, generalizing solvability mentioned for \DC\ navigation problems in Section~\ref{sub:shortest}.
We first define
\[
    \rewards\task =\set{(\st,\act)\in\states\times\actions\mid\R(\st,\act)=\M},
\]
and,
\[
    \goals\task=\set{\st\in\states\mid\exists\act\in\actions\text{ with }(\st,\act)\in\rewards\task}.
\]
Intuitively, $\goals\task$ contains the states where immediate reward can be obtained.
Now, we define the set of reducible states of $\task$ as follows,
\[
    \reduce\task = \bigcup_{i=1}^\infty \reducelayer i,
\]
where 
\begin{itemize}
    \item $\reducelayer 1=\goals\task$; and,
    \item for each $i\geq 2$, 
    \[
    \reducelayer i = \reducelayer{i-1} \cup \set{\st\in\states\mid\exists\act\in\actions\text{ with }\tr(\st,\act)\subseteq \reducelayer{i-1}}.
    \]
\end{itemize}
Intuitively, $\reduce\task$ represents a stack of layers, as illustrated in Figure~\ref{fig:reduce}. Set $\reducelayer 1$ is the base layer, containing the goal states. Set $\reducelayer 2$ adds those states that have an action leading into $\reducelayer 1$, closer to reward. We keep stacking layers until we can add no more states.
Despite the nondeterminism in each state-action application, each state in $\reduce\task$ can still approach reward. 
Since $\states$ is finite, there is always a smallest index $n\geq 1$ such that $\reducelayer n=\reducelayer{n+1}$, i.e., $\reducelayer n$ is the fixpoint of the sequence.

\begin{figure}
    \begin{center}
    \includegraphics[height=0.2\textheight]{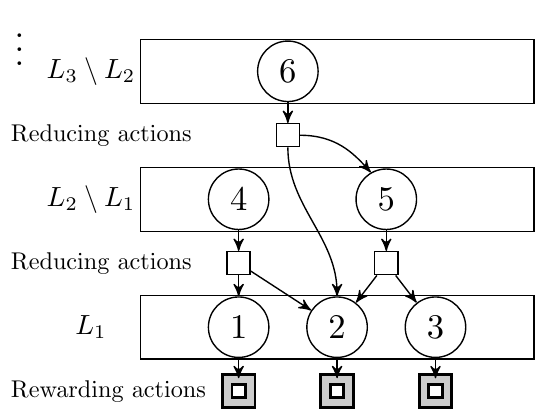}
    \caption{Illustration of the layer structure in reducibility.}
    \label{fig:reduce}
    \end{center}
\end{figure}

We abbreviate $\nonreduce\task=\states\setminus\reduce\task$.
Note that for each $\st\in\nonreduce\task$, for each $\act\in\actions$, we must have $\tr(\st,\act)\cap\nonreduce\task\neq\emptyset$; otherwise $\st\in\reduce\task$. This means that once the agent enters a non-reducible state, the nondeterminism can keep the agent inside the non-reducible states for arbitrary amounts of time. There is no reward in non-reducible states: in a state $\st$, if there would be a rewarding action then $\st\in\goals\task\subseteq\reduce\task$.

\begin{definition}\label{def:reduce}
We say that a task $\task=\tasktup$ is reducible if the following conditions are satisfied:
\begin{enumerate}
    \item\label{enu:reduce-start} $\startstates\subseteq\reduce\task$; and,
    
    \item\label{enu:reduce-escape} for each $\st\in\nonreduce\task$ and each $\st'\in\startstates$ there is a path 
    \[
        \stX1\edge{\actX 1}\ldots\edge{\actX{n-1}}\stX n,
    \]
    where $(\stX 1,\stX n)=(\st,\st')$ and $\set{\stX 2,\ldots,\stX{n-1}}\subseteq\nonreduce\task$.%
        \footnote{We allow $n=2$, in which case the path consists of a single jump.}
\end{enumerate}
\end{definition}
The first condition says that all start states should have a strategy to reward. Whenever the agent would stumble onto a non-reducible state, the second condition provides an escape route back to any start state, entirely tunneled through non-reducible states.

To motivate the assumption of reducibility, the following example gives a non-reducible navigation problem that could forever cause cycles without reward, even with greedy runs.
\begin{example}
    We consider the task shown in Figure~\ref{fig:non-reducible}. For simplicity, we assume $\step=1$ and that initial values are zero. 
    State $2$ is the only non-reducible state.
    When inside state $2$, nondeterminism can keep the agent inside state $2$ for arbitrary amounts of time, leading to cycles without reward. 
    \ramp\ can however not always learn to avoid entering state $2$.
    In a greedy run, if the agent would perform $(1,a)$ three times before $(1,b)$ then the following happens: 
        $(3,a)$ or $(3,b)$ gets value $3$; next,
        $(2,a)$ or $(2,b)$ gets value $2$; and,
        $(1,a)$ gets value $1$.%
        \footnote{Recall that greedy runs, as defined in Section~\ref{sub:runs}, have a built-in fairness condition that would prevent the agent from being stuck inside state $2$ forever.}
    But then the agent would keep running into state $2$; in that case, the greedy run could witness many cycles without reward.
    
    Note that if there would have been an additional escape option from state $2$ back to state $1$, say $1\in\tr(2,a)$, then condition~\ref{enu:reduce-escape} of reducibility would be satisfied. 
    In that case, if we are stuck in state $2$ long enough, both actions $a$ and $b$ are repeatedly tried, whose value is diminished to zero; that is the right moment to jump back to state $1$. The subsequent visit from $1$ to $2$, through action $a$, will make the value of $(1,a)$ also zero. To try action $b$ at $1$, we should however return from $2$ to $1$ before witnessing the rewarding exit to state $3$. In a greedy run, the built-in fairness condition ensures that this right sequence of events can not be postponed indefinitely.
    \qed    
\end{example}

\begin{figure}
    \begin{center}
    \includegraphics[height=0.15\textheight]{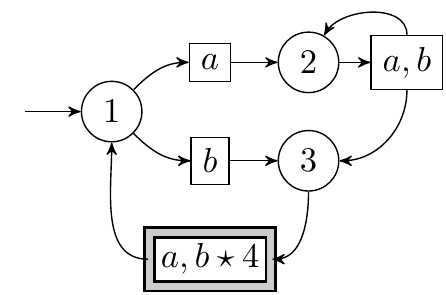}
    \end{center}
        \caption{Using the same graphical notation as in Figure~\ref{fig:cycle}, this task has the following nonzero rewards: $(3,a)\mapsto 4$, $(3,b)\mapsto 4$.}
    \label{fig:non-reducible}
    \end{figure}

\begin{remark}[Relationship to relocations]
    For a reducible state, the agent can trust that certain actions will bring the agent (gradually) closer to reward.     
    On reducible states, we rule out that an  external observer could intervene at arbitrary moments to send the agent to specific places in the state space. 
    This ensures that the agent can in principle follow nice ramp shapes that peak at reward.
    By contrast, if we would suddenly relocate the agent without reward to a low-value state, or without reward to a high-value state, then the ramp-shape of the values might be locally damaged. This is suggested in Figure~\ref{fig:ramp-damage}.
    \qed
\end{remark}

\begin{figure}
\newcommand{\wOne}{0.4\textwidth}
\begin{center}
\begin{tabular}{cc}
    \includegraphics[width=\wOne]{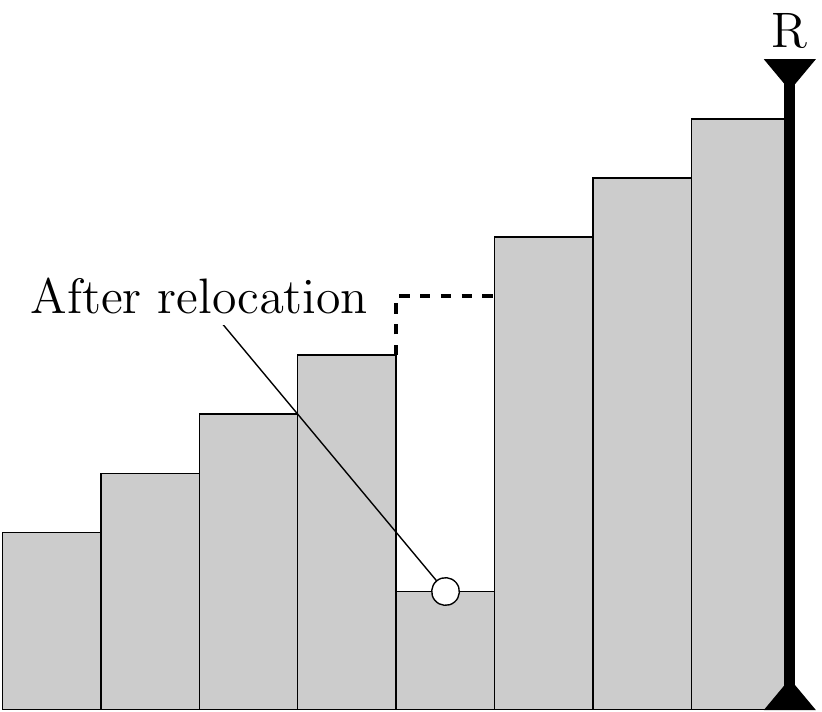}&
    \includegraphics[width=\wOne]{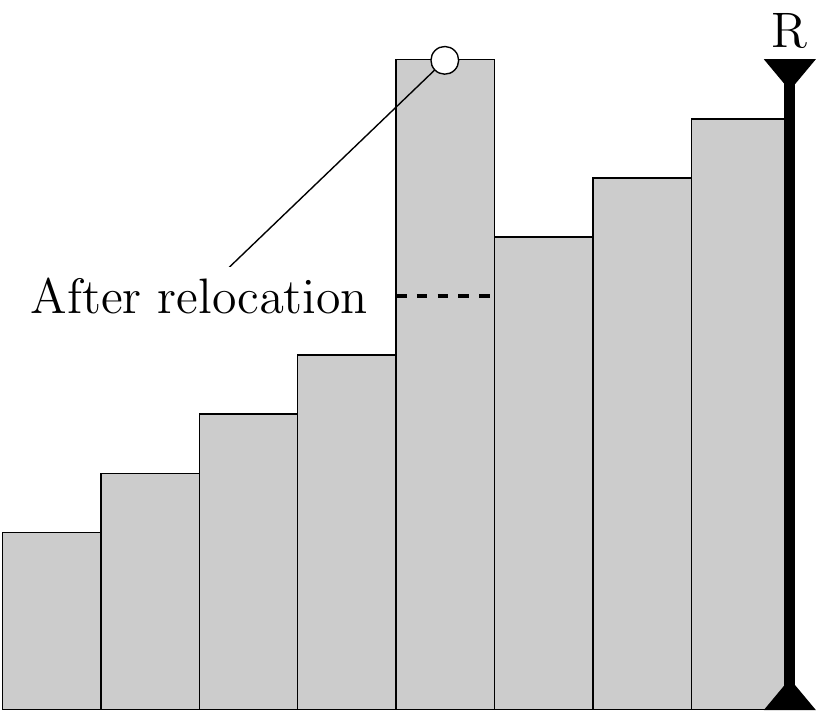}
    \\
    (A) & (B) 
\end{tabular}
\end{center}
\caption{When the agent is relocated at wrong moments, an otherwise good ramp shape might be disrupted. The value could be artificially decreased (A) or increased (B).}
\label{fig:ramp-damage}
\end{figure}

\subsection{Restartability}

We also consider an additional technical assumption on tasks:
\begin{definition}\label{def:restartable}
    We say that a navigation problem $\task=\tasktup$ is restartable if for each $(\st,\act)\in\rewards\task$ we have $\tr(\st,\act)=\startstates$. 
\end{definition}
\noindent This assumption allows the environment to put the agent at another start state after reaching a goal. Possibly such start states are very near to the recently obtained reward, making the movement of the agent sometimes appear seamless in the state space.
    This observation indicates that restartable navigation problems encompass some practical navigation cases on a map, where sometimes we want to simulate that the agent is simply staying at a certain location after obtaining a reward.
    
In combination with reducibility, the assumption of restartability ensures that we remain inside reducible states once we obtain reward. This way, the agent can in principle continually navigate towards reward without getting trapped in non-reducible states.
The following example shows that reducibility by itself does not ensure that the agent can learn to avoid rewardless cycles, but that the combination with restartability is useful.
\begin{example}
    We consider the navigation problem shown in Figure~\ref{fig:non-restartable}(A). There are two states $1$ and $2$, and one possible action $a$.
    The task is not restartable.
    After obtaining reward through the pair $(1,a)$, the agent could be trapped inside the non-reducible state $2$ for arbitrary amounts of time, leading to rewardless cycles.
    
    In Figure~\ref{fig:non-restartable}(B), we consider a modification of subfigure (A) to a reducible and restartable task. Note that state $2$ has now become a start state. We have also removed state $2$ as a successor state of the pair $(2,a)$, ensuring that $2$ is reducible (which is a property demanded for start states by reducibility).
    In this simple example it is immediately clear that every cycle contains reward. More generally, in Theorem~\ref{theo:greedy} (below), we will see the useful effect of combining reducibility and restartability on learning in greedy runs.
    \qed
\end{example}

\begin{figure}
\newcommand{\wOne}{0.4\textwidth}
\begin{center}
\begin{tabular}{cc}
    \includegraphics[width=\wOne]{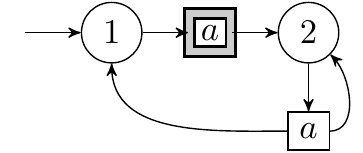}&
    \includegraphics[width=\wOne]{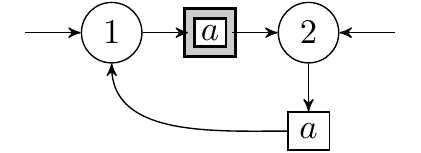}
    \\
    (A) & (B) 
\end{tabular}
\end{center}
\caption{(A) reducible but non-restartable navigation problem. (B) reducible and restartable navigation problem. Rewarding actions are indicated by a shaded box.}
\label{fig:non-restartable}
\end{figure}

\subsection{Navigation result}
As an abbreviation, we say that a navigation problem is \RR\ if the problem is both reducible and restartable.
On \RR\ navigation problems, \ramp\ successfully learns to avoid rewardless cycles in every greedy run:
\begin{theorem}\label{theo:greedy}
    On each \RR\ navigation problem with reward quantity $\M$, when initial values are below $\M$, in each greedy run, eventually all state cycles contain reward.
\end{theorem}
\noindent \proofmain{theo:greedy--PROOF}

\begin{remark}[Assumptions]\label{remark:greedy-assumptions}
    Removing the assumption on initial values in Theorem~\ref{theo:greedy} could be an item for future work. Without the assumption, the agent requires additional time to unlearn high violating values (see Figure~\ref{fig:ramp-violation}), before it could learn paths towards the true reward. 
    The assumption fortunately does not seem too severe, because a practical simulation might initialize the value function to satisfy the assumption.
    
    Moreover, the notions of reducibility and restartability might perhaps be combined into a more tight concept, where we assume that after obtaining reward we do not necessarily end up at a start state but just at a reducible state.
    Formally, letting $\task=\tasktup$ be a navigation problem, for each $(\st,\act)\in\rewards\task$, we could assume $\tr(\st,\act)\subseteq\reduce\task$.
    However, not giving a special role to start states might make it more difficult to assume a structure on the non-reducible states; the current assumption in Definition~\ref{def:reduce} is anchored on start states.
    \qed
\end{remark}

\begin{remark}[Applicability]
    Theorem~\ref{theo:greedy} works in particular for deterministic \RR\ navigation problems, where, necessarily, there could be only one start state.
    Theorem~\ref{theo:greedy} also applies to \RR\ navigation problems that are completely deterministic on non-rewarding state-action pairs, but where the rewarding state-action pairs are non-deterministic (in selecting the next start state).
    \qed
\end{remark}

\begin{example}[\RR\ grid navigation]\label{ex:swamp}
    We extend the simulation of Example~\ref{ex:dc-sim} with nondeterministic effects. We again consider a navigable 2D grid, with actions left, right, up, down, and finish. 
    This time we allow multiple start cells. 
    We keep using the goal cells from earlier: 
    when performing the finish action at goal cells, reward is obtained and the agent is transported back to a randomly selected start cell; this corresponds to the restartability assumption.
    We add two additional types of cell: swamp cells and jump cells.
    In a swamp cell, for each action, we nondeterministically (1) send the agent back to a (random) start state, or (2) we apply a random offset from the set 
    \[
    \set{(0,0),(-1,0),(1,0),(0,1),(0,-1)}.
    \]    
    The ability to restart the task from the swap cells is needed for reducibility (condition~\ref{enu:reduce-escape} in Definition~\ref{def:reduce}).%
        \footnote{Satisfaction of condition~\ref{enu:reduce-start} in Definition~\ref{def:reduce} depends on the specific 2D map.}
    Due to the offset $(0,0)$, the agent could become stuck for arbitrary amounts of time when it enters a swamp cell.
        The other offsets let swamp cells unpredictably move the agent; but that is not crucial for this example.
    Clearly, swamp cells have no action that is guaranteed to reach a goal cell, even if some start cells are goal cells. Hence, all swamp cells are non-reducible.
    
    Second, we have special jump cells: for any movement action, the jump cell takes the direction of the action and nondeterministically multiplies it by either 2 or by 4. For example, on a jump cell, if the agent performs the action `right', with direction $(1,0)$, then the effective offset could be $(2,0)$ or $(4,0)$. No movement is performed if the resulting offset would end in an obstacle.
    
    We show a concrete example situation in Figure~\ref{fig:swamp}. Figure~\ref{fig:swamp-path} shows results from one simulation, where the agent eventually avoids rewardless cycles, but where the agent forever fluctuates between different effective paths due to the jump cells.
    \qed
\end{example}

\begin{figure}
\newcommand{\wOne}{0.45\textwidth}
\begin{center}
\begin{tabular}{cc}
    \includegraphics[width=\wOne]{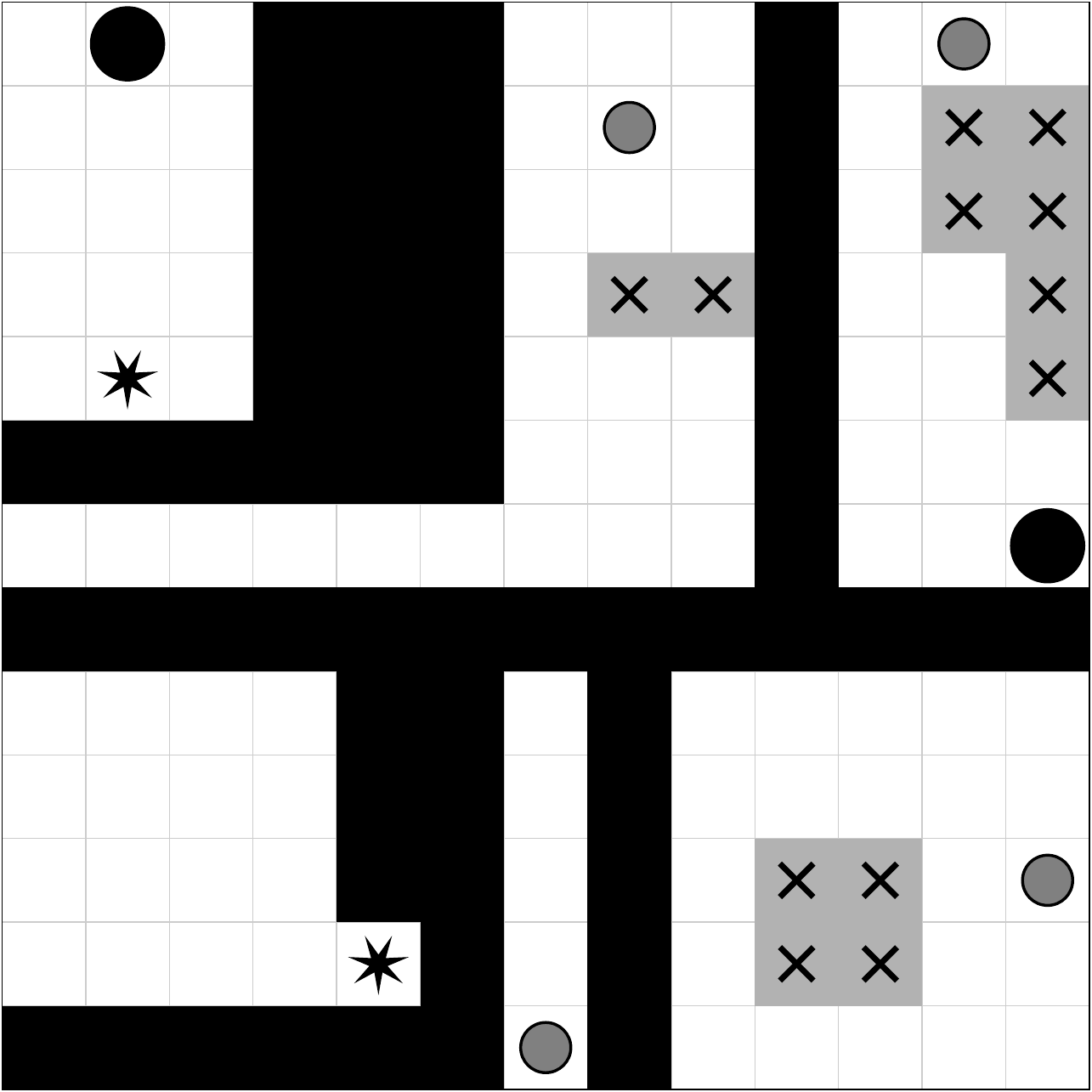}&
    \includegraphics[width=\wOne]{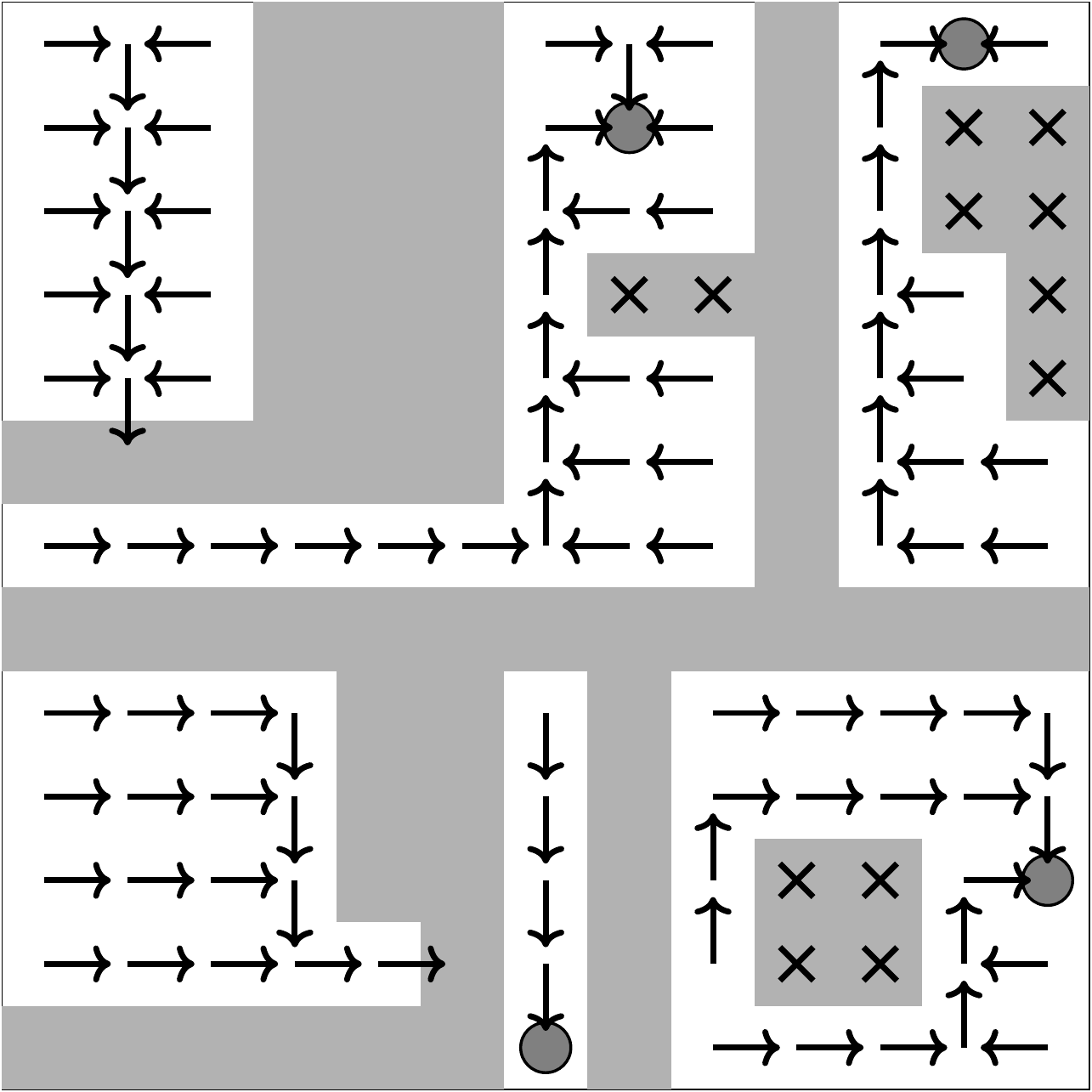}
    \\
    (A) & (B) 
\end{tabular}
\end{center}
\caption{In the context of Example~\ref{ex:swamp}, we give an \RR\ navigation problem in 2D. (A) We use the following notational convention: start cells $\to$ black circle; goal cells $\to$ gray circle; swamp cells $\to$ ``X''; jump cells $\to$ star. Reducibility is shown in subfigure~(B). Note that the jump cells prefer an action that seemingly jumps through a wall; that is possible because the direction is multiplied by either 2 or 4.}
\label{fig:swamp}
\end{figure}

\begin{figure}
    \centering
    \begin{subfigure}[b]{0.3\textwidth}
        \includegraphics[width=\textwidth]{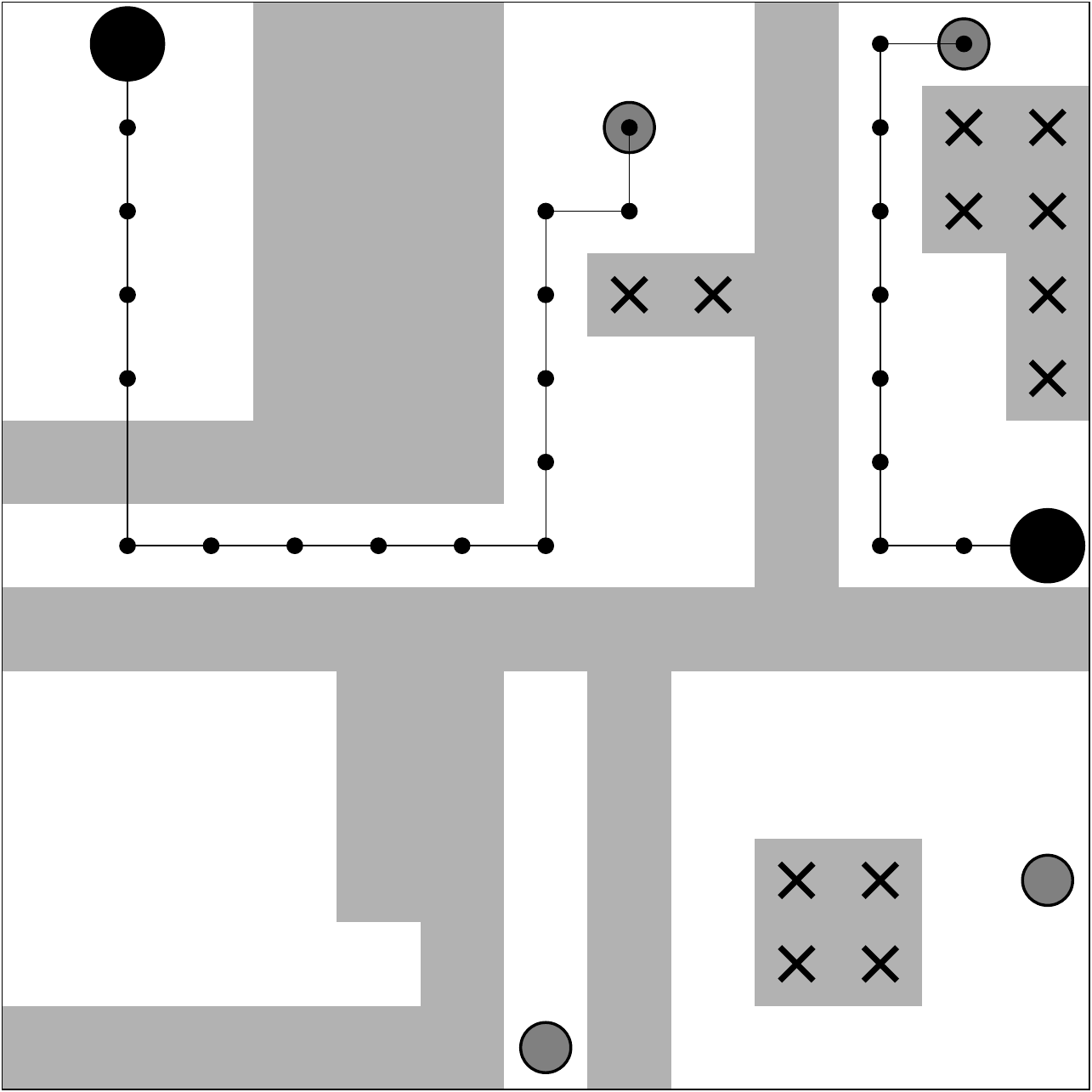}
        \caption{Path 1}        
    \end{subfigure}   
    \quad
    \begin{subfigure}[b]{0.3\textwidth}
        \includegraphics[width=\textwidth]{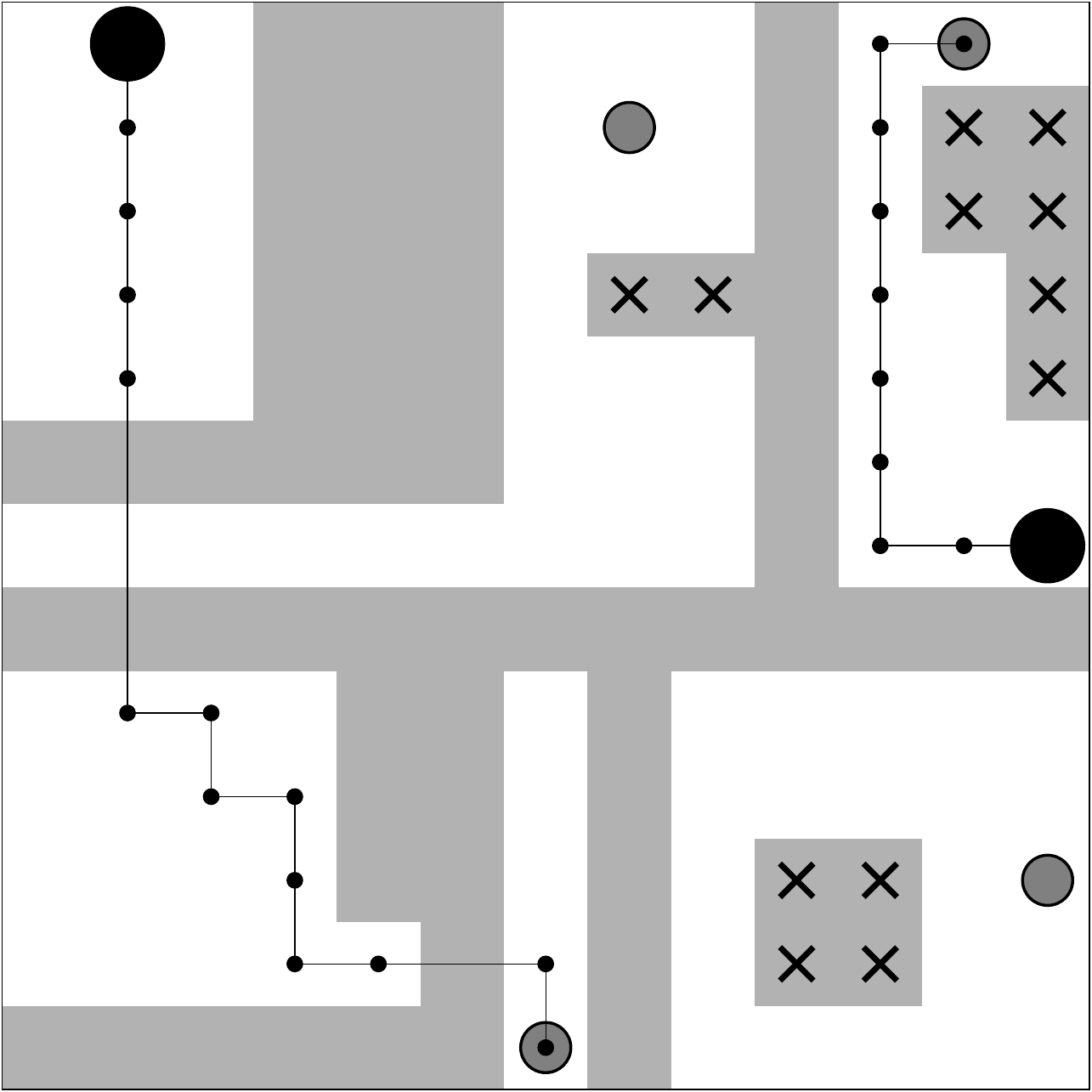}
        \caption{Path 2}        
    \end{subfigure}
    \quad
    \begin{subfigure}[b]{0.3\textwidth}
        \includegraphics[width=\textwidth]{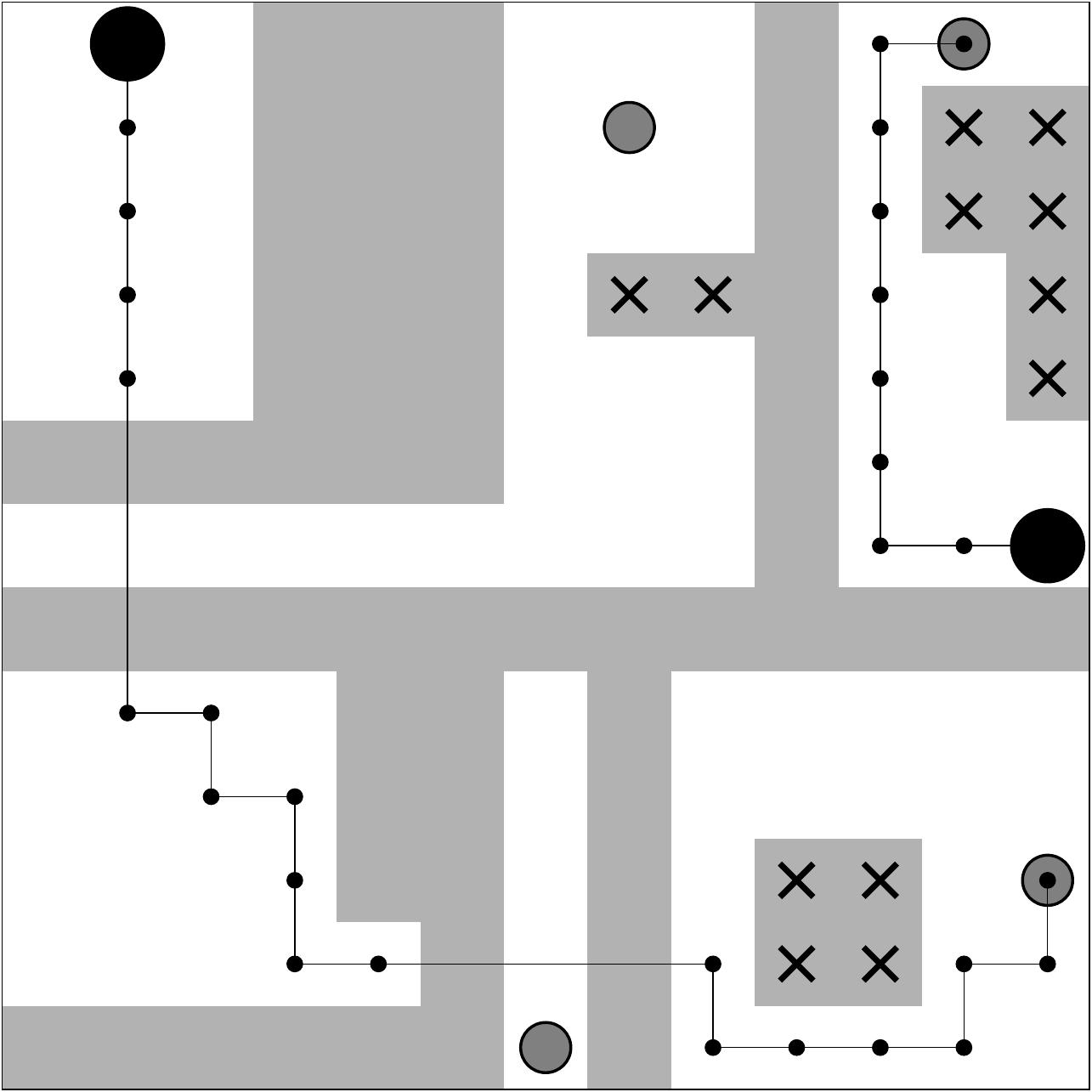}
        \caption{Path 3}       
    \end{subfigure}
    \caption{Consider the task in Figure~\ref{fig:swamp} (A). For the top-left start cell, inside one simulation of a greedy run, the path of the agent to reward strongly depends on the nondeterministic outcome of the jump cells. For the rightmost start cell there is a deterministic path to its local reward.}\label{fig:swamp-path}
\end{figure}

\subsection{Proof of Theorem~\ref{theo:greedy}}
\label{theo:greedy--PROOF}

Let $\task=\tasktup$ be an \RR\ navigation problem, and let $\M$ denote the nonzero reward quantity. 
We fix $\task$ throughout this subsection.

\subsubsection{Proof intuition}

The agent repeatedly begins in the start states and finds for each start state some path to reward. The formed paths are not necessarily the shortest. The agent first remembers some high-value actions nearest to reward; such state-action associations form an initial reward strategy, very localized near the reward. Any other states can be gradually added to that initial reward strategy, as suggested in Figure~\ref{fig:greedy-grow}.
There is a growth process of the reward strategy that propagates back to the start states, until the agent has a reward strategy from every start state. At that point, the action preferences strongly restrict what part of the state space is visited by the agent.

We recall that the considered navigation problem is restartable, implying that after obtaining reward we go back to a start state.

We also recall that the navigation problem may be nondeterministic, meaning that the successor state can vary between different applications of the same state-action pair. In general, the agent can not avoid nondeterminism, so the agent should find a reward strategy that steadily moves towards reward despite the nondeterminism; this is illustrated in Figure~\ref{fig:greedy-nondet}. Reducibility (Definition~\ref{def:reduce}) ensures the existence of such a reward strategy.

Also, the agent can be distracted by the initial value function, or by early estimated values that are misleading due to nondeterminism. But, eventually, the agent will learn where the reward is. The agent can be thought of as digging through layers of violating values to reach the reward. This requires that wrong values must be modified, either increased or decreased, to match the values emanating from the reward. In the proof, we use the built-in fairness assumption of exploring runs to confront the agent with wrong values.

\begin{figure}    
    \begin{center}
        \includegraphics[width=0.5\textwidth]{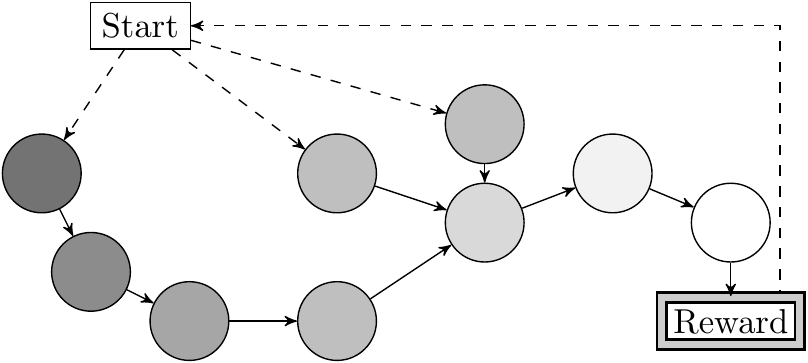}
    \end{center}
    \caption{The agent gradually builds longer paths to reward by adding new state-action pairs in front of existing paths. In a greedy run, we do not necessarily find the shortest paths to reward. In this figure, states with higher value have a brighter shade.}
    \label{fig:greedy-grow}
\end{figure}

\begin{figure}
    \begin{center}
        \includegraphics[height=0.2\textheight]{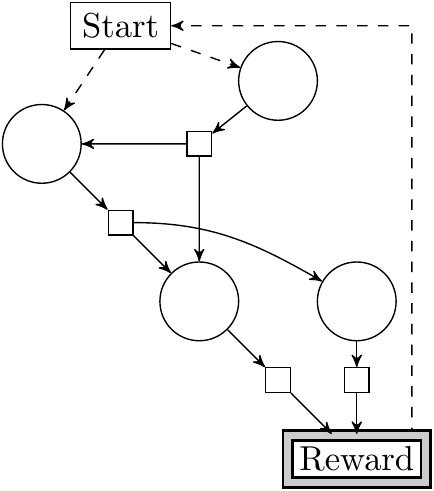}
    \end{center}
    \caption{Illustration of nondeterminism in the final reward strategy. States are represented by circles, and their chosen actions by boxes. Note that in this figure, each action gets strictly closer to reward, despite the nondeterminism. Reducible tasks always have such a reward strategy.}
    \label{fig:greedy-nondet}
\end{figure}

\subsubsection{Strategy}\label{sub:greedy-strategy}
Let $\V$ be a value function. Similar to reducibility (Definition~\ref{def:reduce}), we collect states that have a reward strategy under $\V$. Formally, 
\[
    \strategy\V = \bigcup_{i=1}^\infty \stratlayer i,
\]
where
\begin{itemize}
    \item the set $\stratlayer 1$ contains all states $\st$ satisfying $\V\get\st=M-\step$, and, $\forall\act\in\pref\st\V$ we have
    \[
      (\st,\act)\in\rewards\task.
    \]
    
    \item for each $i\geq 2$, the set $\stratlayer i$ extends $\stratlayer{i-1}$ with all states $\st$ satisfying $\V\get\st=M-i\step$, and, $\forall\act\in\pref\st\V$:
    \begin{enumerate}
        \item $(\st,\act)\notin\rewards\task$,
        \item $\tr(\st,\act)\subseteq\stratlayer{i-1}$,
        \item $\exists\st'\in\tr(\st,\act)$ with $\V\get{\st'}=M-(i-1)\step$.
    \end{enumerate}
\end{itemize}
We call the sets $\stratlayer i$ strategy layers.
Similar to reducibility, the strategy starts with $\stratlayer 1$, containing those states that prefer only rewarding actions and have correct value estimation. Then we add layers of states that prefer non-rewarding actions but whose successor states end up closer to $\stratlayer 1$. Increasing index $i$ corresponds to following value ramps downhill.

We emphasize the following conservative value estimation: for $i\geq 2$, letting $\st\in\stratlayer i$ and $\act\in\pref\st\V$, some successors in $\tr(\st,\act)$ could have value strictly larger than $\M-(i-1)\step$, yet we demand $\V\get\st=\M-i\step$. This will be important in Section~\ref{sub:proof-greedy-next}, where we would like strategies to be preserved under some transformation function.

\subsubsection{Good configurations}\label{sub:proof-greedy-good}

We say that a configuration $\cnf=\cnftup$ is good if the following conditions are satisfied:
\begin{enumerate}
    \item $\startstates\subseteq\strategy\V$; and,
    \item $\st\in\strategy\V$.
\end{enumerate}

By Property~\ref{prop:good-reward} (below), we know that once a greedy run encounters a good configuration, all cycles contain reward. 
The intuition, is that the strategy brings the agent to reward in an acyclic manner. Once all start states and the current state belong to the strategy, the agent is bound inside the strategy forever. In particular, during a reward transition, the agent is sent to another start state, which is inside the strategy.

Our goal is to show that each greedy run eventually encounters a good configuration. 

\begin{property}\label{prop:good-reward}
    In every greedy run, after reaching a good configuration, all state cycles contain reward.    
    \proofapp{prop:good-reward--PROOF}
\end{property}

\subsubsection{Transforming configurations}
\label{sub:proof-greedy-next}

We define a deterministic function $\nextname$ to transform any configuration into a good configuration.
More specifically, when given a configuration, function $\nextname$ tells us (1) which action should be taken, and (2) which successor state should be visited.
As we will see in Section~\ref{sub:proof-greedy-eventually-good}, the fairness assumption of greedy runs allows $\nextname$ to be called sufficiently often.
To specify $\nextname$ as fully deterministic, we assume a total order on the finite sets $\states$ and $\actions$.
    Usage of the order is indicated by the ``$\min$''-operator on some sets of states and actions.

\paragraph*{Restarts from non-reducible states}
Before defining $\nextname$, we define an acyclic movement strategy from non-reducible states to start states.
Fixing a start state $\st$, we specify a function $\go\st:\nonreduce\task\to\actions\times\states$ below.

We introduce some convenience notation. Let $X\subseteq\states\times\actions\times\states$ be a set.
We define $\dom X = \set{\st\in\states\mid\exists\act,\st':\,(\st,\act,\st')\in X}$. The set $X$ can be viewed as a nondeterministic function $f:\dom X\to\powset{\actions\times\states}$, defined for each $\st\in\dom X$ as $f(\st)=\set{(\act,\st')\mid(\st,\act,\st')\in X}$.

Now, we define a set $\G\subseteq\nonreduce\task\times\actions\times\states$,
\[
    \G = \bigcup_{i=1}^\infty \Glayer i,
\]
where
\begin{itemize}
    \item $\Glayer 1 = \set{(\st',\act,\st)\mid \st'\in\nonreduce\task,\act\in\actions,\st\in\tr(\st',\act)}$,
    
    \item for each $i\geq 2$,    
    \begin{align*}
        \Glayer i = \Glayer{i-1} \cup 
            \Big\{(\st',\act,\st'') \mid\, 
            & \st'\in\nonreduce\task\setminus\dom{\Glayer{i-1}},\\
            & \act\in\actions,\st''\in\tr(\st',\act)\cap\dom{\Glayer{i-1}}\Big\}
    \end{align*}    
\end{itemize}
Intuitively, $\G$ specifies how non-reducible states could choose action-successor pairs to move closer to the fixed start state $\st$. Note that this movement solution is acyclic by definition of $\Glayer i$ for each $i\geq 2$.
By Property~\ref{prop:nonreduce-restart} (below) we know $\nonreduce\task\subseteq\dom{\G}$. 

We convert $\G$ into a deterministic function $\go\st:\nonreduce\task\to\actions\times\states$. Using the assumed order on $\states$ and $\actions$, for each $\st'\in\nonreduce\task$, we define $\go\st(\st')$ as the lexicographically smallest pair in the set $\G(\st')$.

\begin{property}\label{prop:nonreduce-restart}
    For each start state $\st$, we have $\nonreduce\task\subseteq\dom{\G}$. 
    \proofapp{prop:nonreduce-restart--PROOF}
\end{property}

\paragraph*{The function $\nextname$}
Let $\cnf=\cnftup$ be an input configuration for $\nextname$.

To jump outside $\strategy\V$ during reward, we define
\[
\restart\cnf = \begin{cases}
                \min(\startstates\setminus\strategy\V) & \text{if }\startstates\not\subseteq\strategy\V \\
                \min(\startstates) &\text{otherwise}.
            \end{cases}
\]
Function $\nextname$ produces an action-successor pair based on the following nested case analysis. The main idea is to gradually decrease wrong values, to arrive at value zero; at that moment we can pull the agent into any desirable direction, in particular towards reward.
\begin{enumerate}[label*=\arabic*]
\newcommand{\mydef}[1]{$\Rightarrow$ Define #1.}

    \item \label{next-zero} Suppose $\V\get\st=0$. Then all actions are equally preferable in $\st$. We choose an action and successor to gradually bring us closer to reward. There are three mutually disjoint cases.
        \begin{enumerate}[label*=.\arabic*]
            
            \item \label{next-zero-nonreduce} Suppose $\st\in\nonreduce\task$.%
                \footnote{We have $\st\notin\startstates$ because $\startstates\subseteq\reduce\task$ by assumption (Definition~\ref{def:reduce}).}
            Let $\st_0=\restart\cnf$.
            
            \mydef{$\next\cnf=\go{\st_0}(\st)$}
            
                
                \item \label{next-zero-reduce-reward} Suppose $\st\in\goals\task$.%
                  \footnote{Recall that $\goals\task\subseteq\reduce\task$ (Section~\ref{sub:reduce}).}
                Let $\act=\min\set{\act'\in\actions\mid(\st,\act')\in\rewards\task}$ and $\st_0=\restart\cnf$.%
                    \footnote{The restartability assumption on the task tells us that $\st_0\in\tr(\st,\act)$.}
                
                \mydef{$\next\cnf=(\act,\st_0)$}
                
                \item \label{next-zero-reduce-reduce} Suppose $\st\in\reduce\task\setminus\goals\task$. We move one layer down into the reducibility structure.
                Recalling $\reduce\task=\bigcup_{i=1}^\infty\reducelayer i$ from Section~\ref{sub:reduce}, we write $\layer\st$ to denote the smallest index $i$ for which $\st\in \reducelayer i$.                 
                Let $\act=\min\set{\act'\in\actions\mid\tr(\st,\act')\subseteq \reducelayer{\layer\st-1}}$, and
                \[
                    \st' = 
                        \begin{cases}
                            \min(\tr(\st,\act)\setminus\strategy\V) & \text{if }\tr(\st,\act)\not\subseteq\strategy\V \\
                            \min(\tr(\st,\act)) & \text{otherwise}.
                        \end{cases}
                \]
                \mydef{$\next\cnf=(\act,\st')$}
        \end{enumerate}
        
    \item \label{next-nonzero} Otherwise $\V\get\st>0$. We confront the agent with wrong values, if any, that lead away from reward. Importantly, because the agent is always greedy in the context of this proof, we may only choose actions preferred by the agent, i.e., actions from $\pref\st\V$.
    
    \begin{enumerate}[label*=.\arabic*]
        
        \item \label{next-nonzero-incorrectActions} Suppose there is an action $\act\in\pref\st\V$ for which $(\st,\act)\notin\rewards\task$ and $\tr(\st,\act)\not\subseteq\strategy\V$.            
        Let $\act$ be the smallest from such actions.
        Let $\st'=\min(\tr(\st,\act)\setminus\strategy\V)$.
        
        \mydef{$\next\cnf=(\act,\st')$}
        
        \item \label{next-nonzero-correctActions} Otherwise, for all $\act\in\pref\st\V$ we have either $(\st,\act)\in\rewards\task$ or $\tr(\st,\act)\subseteq\strategy\V$. There could still be errors in the value estimation.
        Regarding notation, for any $\act\in\actions$ with $\tr(\st,\act)\subseteq\strategy\V$, we define 
        \[
            \expect\V\st\act = \min\set{\V\get{\st'}\mid\st'\in\tr(\st,\act)}.
        \]
        Intuitively, $\expect\V\st\act$ is a conservative value expectation.
        
        \begin{enumerate}[label*=.\arabic*]
        
            \item \label{next-nonzero-incorrectReward} Suppose there is an action $\act\in\pref\st\V$ with $(\st,\act)\in\rewards\task$ but $\V(\st,\act)\neq \M-\step$.       
            Let $\act$ be the smallest such action, and let $\st_0=\restart\cnf$.
            
            \mydef{$\next\cnf=(\act,\st_0)$}                        
                       
            \item \label{next-nonzero-incorrectReduce} Suppose there is an action $\act\in\pref\st\V$ with $(\st,\act)\notin\rewards\task$, and therefore $\tr(\st,\act)\subseteq\strategy\V$, with
            \[
                \V(\st,\act) \neq \expect\V\st\act - \step.
            \]
            Let $\act$ be the smallest such action, and let 
            \[
                \st' = \min\set{\st''\in\tr(\st,\act)\mid \V\get{\st''} = \expect\V\st\act}.
            \]
            
            \mydef{$\next\cnf=(\act,\st')$}
                
            \item \label{next-nonzero-correct} Otherwise the value estimation is correct. We choose an action-successor pair to proceed. Let $\act=\min(\pref\st\V)$.
            
            \begin{enumerate}[label*=.\arabic*]
                \item \label{next-nonzero-correct-reward} Suppose $(\st,\act)\in\rewards\task$. Let $\st_0=\restart\cnf$.
                
                \mydef{$\next\cnf = (\act,\st_0)$}
                
                \item \label{next-nonzero-correct-reduce} Otherwise, $(\st,\act)\notin\rewards\task$, but we still know $\tr(\st,\act)\subseteq\strategy\V$. Let 
                \[
                    \st'=\min\set{\st''\in\tr(\st,\act)\mid\V\get{\st''}=\expect\V\st\act}.
                \]
                
                \mydef{$\next\cnf = (\act,\st')$}
            \end{enumerate}
        \end{enumerate}
        
    \end{enumerate}
\end{enumerate}

\subsubsection{Eventually good configurations}
\label{sub:proof-greedy-eventually-good}

We fix a greedy run. By means of function $\nextname$, we show that the greedy run eventually encounters a good configuration. Intuitively, $\nextname$ represents the useful learning opportunities that are witnessed by the agent.    

\paragraph*{Bring start states in strategy}
First, we show the existence of a configuration $(\st,\V)$ that occurs infinitely often in the run and with $\startstates\subseteq\strategy\V$.

Because there are only a finite number of configurations (Lemma~\ref{lem:finite-cnf}), there is at least one configuration $\cnf=\cnftup$ that occurs infinitely often. 
Note that function $\nextname$, by design, proposes action-successor pairs allowed by a greedy transition. Since $\cnf$ occurs infinitely often, the fairness assumption of greedy runs (see Section~\ref{sub:runs}) tells us that we perform the following transition infinitely often:
\[
    \cnf \jump{\act,\st'}(\st',\V'),
\]
where $(\act,\st')=\next\cnf$. So, $(\st',\V')$ too occurs infinitely often, and therefore we can also apply $\nextname$ to $(\st',\V')$, and so on. 
We see that $\nextname$ can be applied arbitrarily many times;
    this process does not necessarily happen as a contiguous sequence of transitions in the run.
If $\startstates\not\subseteq\strategy{\V'}$ then Property~\ref{prop:extend-strategy} (below) tells us that we eventually discover a configuration $(\st'',\V'')$ with $\strategy{\V'}\subsetneq\strategy{\V''}$, i.e., with the strategy strictly extended, that occurs infinitely often. 
As long as the configurations encountered by $\nextname$ have a start state outside the strategy, we can repeat Property~\ref{prop:extend-strategy} to strictly extend the strategy. But the strategy can not keep growing because there are a finite number of states. By repeated application of $\nextname$, we eventually arrive at a configuration $(\st'',\V'')$ with $\startstates\subseteq\strategy{\V''}$ that occurs infinitely often in the greedy run.

\begin{property}\label{prop:extend-strategy}
    Beginning at a configuration $\cnftup$ with $\startstates\not\subseteq\strategy\V$, by repeatedly applying $\nextname$ we eventually reach a configuration $(\st',\V')$ with $\strategy\V\subsetneq\strategy{\V'}$, i.e., we have strictly extended the strategy.
    \proofapp{prop:extend-strategy--PROOF}
\end{property}

\paragraph*{Bring current state in strategy}
At this point, we have shown that there is a configuration $\cnf=\cnftup$ occurring infinitely often in the run and with $\startstates\subseteq\strategy\V$. We proceed to showing the existence of a good configuration, where additionally the current state is in the strategy.
Again, by the fairness assumption on the run, we can apply $\nextname$ an arbitrary number of times, starting at configuration $\cnf$. 
By Property~\ref{prop:next-preserve-strategy} (below) we know that the strategy is preserved.
Moreover, by Property~\ref{prop:next-inf-reward} (below) there is at least one occurrence of reward. Since the task is restartable, during the reward transition we arrive at a start state, which is inside the strategy. At that moment, we have reached a good configuration, as desired.

\begin{property}\label{prop:next-preserve-strategy}
    Function $\nextname$ always preserves the strategy. More formally, for each transition generated by $\nextname$,
    \[
        \transX i,
    \]     
    we have $\strategy{\VX i}\subseteq\strategy{\VX{i+1}}$.
    \proofapp{prop:next-preserve-strategy--PROOF}
\end{property}

\begin{property}\label{prop:next-inf-reward}
    Beginning at any configuration, by repeatedly applying $\nextname$ we encounter infinitely many reward transitions.
    \proofapp{prop:next-inf-reward--PROOF}
\end{property}

\section{Conclusion and further work}
\label{sec:conclusion}

By means of formal theorems, we have given concrete insights into the operation of \ramp\ on well-defined classes of tasks. We now discuss interesting items for further work.

\paragraph*{Practical case studies}
In this paper we have been occupied with the search for general yet nontrivial descriptions of the agent behavior generated by \ramp. A complementary study could focus on testing \ramp\ on various practical problems, to observe agent behavior on more concrete circumstances, and meanwhile to judge the practical viability of the technique.
It appears likely that \ramp\ can be used for much more problems than the 2D grid examples that we have given. For example, each state could be a sequence of sensory cues, to represent agent conceptualization in a complex environment~\cite{mnih_2015}.

In the usage of \ramp, a concrete proposal could be to set the reward quantities rather high and to take $\step = 1$, because then the ramps are longer and the agent can subsequently learn long strategies to rewarding events.
Moreover, our intuition from the proof of Theorem~\ref{theo:greedy} is that nondeterministic tasks could in general be very slow to learn, because the formation of an acyclic rewarding strategy seems to require rare learning opportunities to be (eventually) witnessed. Practical studies might therefore benefit from introducing sufficiently specific concepts inside the agent, so that tasks are rendered approximately deterministic.
As suggested by \citet{fremaux_2013},  specific concepts might correspond to place cells in the brain, see e.g.~\cite{moser_2008}.

\paragraph*{Generalized exploration property}
Perhaps Theorem~\ref{theo:explore} and Theorem~\ref{theo:sprint} can be generalized to particular kinds of nondeterministic tasks. Likely, in such a generalization, we should not seek numerical stability of the values, but rather a stability of the knowledge of the highest value paths. This can be likened to Theorem~\ref{theo:greedy}, of greedy navigation on nondeterministic tasks, where we sought a behavioral stability property instead of a numerical (value) stability property. Of course, it could be that, even on simple tasks, continued exploration leads to continued fluctuations in agent behavior, as suggested by the simple Example~\ref{ex:explore-fluctuate}.

\paragraph*{More navigation problems}
Possibly, \ramp\ can learn to avoid rewardless cycles on more navigation problems than the \RR\ ones of Section~\ref{sec:greedy}. More work is needed to understand the form of learnable navigation problems. Some further suggestions on relaxing Theorem~\ref{theo:greedy} are mentioned in Remark~\ref{remark:greedy-assumptions}.

\paragraph*{Negative reward and avoidance}
Reward in this paper is always a nonnegative quantity. Negative quantities could be introduced to study avoidance learning. Or, one could consider a dual value-ramp principle for estimating the aversiveness of state-action pairs. In the aversive value-ramp, the values increase as the agent approaches an aversive stimulus. Whereas greediness in a rewarding value-ramp selects actions to maximize value, greediness in an aversive value-ramp selects actions to minimize value.

\paragraph*{Partial observability and features}
Towards better understanding \ramp\ on more practical problems, it might be useful to formalize how the task structure is derived from various practical constraints. For example, the agent might have sensors with limited range, leading to perceived states that deviate from the true environment states. This leads to structural assumptions on the transition function $\tr$. It appears interesting to make concrete insights similar to the ones we have presented when more structural assumptions about the tasks are taken into account.

A brain consists of multiple neurons, and each neuron might represent a feature, i.e., a piece of state information. 
Each encountered task state is projected to a set of features.
It appears interesting to extend our framework to learn value for feature-action pairs instead of state-action pairs. In each state, the feature-action pair with the highest value could determine the action for the state.

\paragraph*{Relationship with reward discounting}
Many algorithms in reinforcement learning are based on reward discounting~\cite{sutton-barto_1998}. An important observation is that reward discounting is based on multiplying values with a rational number $\gamma$ between zero and one, whereas the value ramp is based on subtracting a strictly positive constant. In further work, it could be interesting to clarify the relationship between reward discounting and the value-ramp principle. The notions could be complementary, but they could be equivalent on certain classes of tasks and reward definitions.

\paragraph*{Biological plausibility}
The \ramp\ algorithm is inspired by simulations of biologically plausible learning models~\cite{fremaux_2013}, that could correspond to observations in biology~\cite{van-der-meer_2011}. 
Possibly, further work could elicit whether suitable variations of \ramp\ accurately model biological learning. In the current \ramp\ algorithm, negative updates to value utilize an arbitrary range, i.e., the $d$-value in Algorithm~\ref{alg:update} has no constraints (in particular for the negative range).
That might not be biologically realistic: if biological learning is based on dopamine, the negative value updates are likely caused by suppressing dopamine; but the dopamine baseline (in neutral circumstances) is already relative low~\cite{schultz_2013}. One might suspect that multiple iterations of dopamine suppression are needed to unlearn wrong value expectations. To obtain this effect, we could redefine the learning rule of Equation~\eqref{eq:ramp} to a rule of the following kind:
\begin{equation*}
    \rl(v,v',r) = \clamp{\max(v',r)-v}-\step.
\end{equation*}
The effect is that at most $\step$ is subtracted when value expectation is not met by successor value or by reward.
This could model a limited but noticeable erosion effect on value, in particular on neuronal connections (representing value) during periods of dopamine suppression.%
    
    One concrete hypothesis could be that on a ramp-like value experience, the steps of size $\step$ represent small dopamine releases that sustain useful concept-action connections in the brain. In absence of such a dopamine release, there could be a net erosion effect on the recently triggered neuronal connections, to unlearn wrong actions.

\bibliographystyle{apalike} 

\newpage\appendix\section*{Appendix}
\addcontentsline{toc}{section}{Appendix}
\addtocontents{toc}{\protect\setcounter{tocdepth}{0}}

\section{Proof details of Lemma~\ref{lem:finite-cnf}}
\label{lem:finite-cnf--PROOF}

\begin{lemma}\label{lem:ceil}
    For each transition
    $
        \cnftup
            \jump{\act,\st'}
        (\st',\V'),
    $
    we have $\ceil\V\geq\ceil{\V'}$.
\end{lemma}
\begin{proof}
    The reward quantities never change, and therefore $\R(\st'',\act'')\leq\ceil\V$ for each $(\st'',\act'')\in\states\times\actions$.
    Regarding values, only the value of $(\st,\act)$ can change during the transition. Therefore $\V'(\st'',\act'')=\V(\st'',\act'')\leq\ceil\V$ for each $(\st'',\act'')\in\states\times\actions\setminus\set{(\st,\act)}$.
    
    There are two cases for $(\st,\act)$:
    \begin{itemize}
        
        \item Suppose $\V'(\st,\act)\leq\V(\st,\act)$. Since always $\V(\st,\act)\leq\ceil\V$, we have $\V'(\st,\act)\leq\ceil\V$.
        
        \item Suppose $\V'(\st,\act)>\V(\st,\act)$. By Algorithm~\ref{alg:update}, we have 
        \begin{align*}
            \V'(\st,\act) &= \clamp{\V(\st,\act) + \big(\max(\V\get{\st'},\R(\st,\act))-\step-\V\get\st\big)}\\
                        &\leq \clamp{\V\get\st +\big(\max(\V\get{\st'},\R(\st,\act))-\step-\V\get\st\big)}\\
                        &=\clamp{\max(\V\get{\st'},\R(\st,\act))-\step}\\
                        &\leq\clamp{\max(\ceil\V,\ceil\V)-\step}\\
                        &\leq\clamp{\ceil\V-\step}\\
                        &\leq\clamp{\ceil\V}\\
                        &=\ceil\V.
        \end{align*}
        In the last step we use that always $\ceil\V\geq 0$. Overall, $\V'(\st,\act)\leq\ceil\V$.        
    \end{itemize}    
\end{proof}

\section{Proof details of Theorem~\ref{theo:explore}}

\subsection{\propproof{prop:explore-obtain-valid}}
\label{prop:explore-obtain-valid--PROOF}

Let $\V$ be a value function. Letting $(\st,\act)\in\states\times\actions$, and denoting $\Det\st\act{\st'}$, we say that $(\st,\act)$ is a violation in $\V$ if
\[
    \V(\st,\act) > \clamp{\max(\V\get{\st'},\R(\st,\act)) - \step}.
\]
We define the highest violation value in $\V$, denoted $\violmax\V$, as follows:
\[
    \violmax\V = 
        \begin{cases}
        \max\set{\V(\st,\act)\mid (\st,\act)\in\viol\V} & \text{if }\viol\V\neq\emptyset\\
        0 & \text{otherwise}.
        \end{cases}
\]
Always $\violmax\V\geq 0$. Also note that $\viol\V=\emptyset\iff\violmax\V=0$:
\begin{itemize}
    \item If $\viol\V=\emptyset$ then $\violmax\V=0$ by definition. 
    \item Suppose $\viol\V\neq\emptyset$. Each violation $(\st,\act)$ in $\V$ satisfies  $\V(\st,\act) > \clamp{\max(\V\get{\st'},\R(\st,\act)) - \step}\geq 0$, implying $\violmax\V>0$.
\end{itemize}
The following property will be useful:
\begin{property}\label{prop:violmax-top}
    For each transition $\transX i$ on $\task$ we have
    \[
        \violmax{\VX i}\geq\violmax{\VX{i+1}}.
    \]
    \proofapp{prop:violmax-top--PROOF}
\end{property}

Recall that the exploring run $\run$ is denoted as
\[
    \cnftupX 1
        \jump{\actX 1,\,\stX{2}}
    \cnftupX 2
        \jump{\actX 2,\,\stX{3}}
    \ldots
\]
We gradually remove all violations.
As long as there are violations in $\run$, the highest violation value is strictly positive. So, while there are violations, Property~\ref{prop:decrease-violmax} (below) tells us that the highest violation value can be strictly decreased. There can only be a finite number of such strict decrements because values are at least zero. Hence, eventually the highest violation value becomes zero. 
Thereafter, all value functions are valid, because for each transition
\[
    \transX i,
\]
if $\VX i$ is valid then $\violmax{\VX i}=0$ and therefore $\violmax{\VX{i+1}}=0$ by Property~\ref{prop:violmax-top}, implying $\viol{\VX{i+1}}=\emptyset$.

\begin{property}\label{prop:decrease-violmax}
    For each configuration index $i\geq 1$, if $\violmax{\VX i}>0$ then there is a configuration index $j>i$ with
    \[
        \violmax{\VX i} > \violmax{\VX j},
    \]
    i.e., the highest violation value has been strictly decreased.
    \proofapp{prop:decrease-violmax--PROOF}
\end{property}

\subsubsection{\propproof{prop:violmax-top}}
\label{prop:violmax-top--PROOF}
Consider a transition,
\[
    \transX i.
\]
We show $\violmax{\VX i}\geq\violmax{\VX{i+1}}$.
To start, by Algorithm~\ref{alg:update}, we have
\[
    \VX{i+1}(\stX i,\actX i) = \clamp{\VX i(\stX i,\actX i) + d},
\]
where $d=\max(\VX i\get{\stX{i+1}},\R(\stX i,\actX i))-\step - \VX i\get{\stX i}$.

Let $(\st,\act)\in\viol{\VX{i+1}}$, and denote $\Det\st\act{\st'}$. We show $\VX{i+1}(\st,\act)\leq\violmax{\VX i}$; overall, this implies $\violmax{\VX{i+1}}\leq\violmax{\VX i}$.
We distinguish between the following cases: $d \geq 0$ and $d < 0$.

\paragraph*{First case ($d\geq 0$)}
If $d\geq 0$ then values are not decreased during the transition, implying $\VX{i+1}\get{\st'}\geq\VX i\get{\st'}$.
If $(\st,\act)\neq(\stX i,\actX i)$ then, using $(\st,\act)\in\viol{\VX{i+1}}$, we have
\begin{align*}
    \VX i(\st,\act) &= \VX{i+1}(\st,\act)\\
                    &> \clamp{\max(\VX{i+1}\get{\st'},\R(\st,\act))-\step}\\
                    &\geq \clamp{\max(\VX i\get{\st'},\R(\st,\act))-\step},
\end{align*}
which implies $(\st,\act)\in\viol{\VX i}$, and therefore $\VX{i+1}(\st,\act)=\VX i(\st,\act)\leq\violmax{\VX i}$.

We show that the other case, $(\st,\act)=(\stX i,\actX i)$, is impossible. Indeed, if $(\st,\act)=(\stX i,\actX i)$ then, based on the above equation for $\VX{i+1}(\stX i,\actX i)$, and using $\VX i(\stX i,\actX i)\leq\VX i\get{\stX i}$ (which is always true),
\begin{align*}
    \VX{i+1}(\stX i,\actX i) &= \clamp{\VX i(\stX i,\actX i) + d}\\
                             &\leq \clamp{\VX i\get{\stX i} + d}\\
                             &= \clamp{\max(\VX i\get{\stX{i+1}},\R(\stX i,\actX i))-\step}\\
                             &\leq\clamp{\max(\VX{i+1}\get{\stX{i+1}},\R(\stX i,\actX i))-\step},                     
\end{align*}
which implies $(\stX i,\actX i)\notin\viol{\VX{i+1}}$.

\paragraph*{Second case ($d< 0$)}
Note that $d<0$ implies $\VX{i+1}(\st,\act)\leq\VX i(\st,\act)$. 

First, if $(\st,\act)\in\viol{\VX i}$ then
\[
    \VX{i+1}(\st,\act)\leq\VX i(\st,\act)\leq\violmax{\VX i}.
\]

Henceforth, suppose $(\st,\act)\notin\viol{\VX i}$, i.e., $(\st,\act)\in\viol{\VX{i+1}}$ is a violation newly created during $\transX i$.
We observe
\begin{align}
    \clamp{\max(\VX{i+1}\get{\st'},\R(\st,\act))-\step} &< \VX{i+1}(\st,\act)\nonumber\\
           &\leq \VX i(\st,\act)\nonumber\\
           &\leq \clamp{\max(\VX i\get{\st'},\R(\st,\act))-\step}.\label{eq:clamp-and-clamp}
\end{align}
Hence, $\VX{i+1}\get{\st'}<\VX i\get{\st'}$.%
    \footnote{Otherwise, if $\VX{i+1}\get{\st'}\geq\VX i\get{\st'}$ then actually $\VX{i+1}\get{\st'}=\VX i\get{\st'}$ by $d<0$; subsequently $\clamp{\max(\VX{i+1}\get{\st'},\R(\st,\act))-\step}=\clamp{\max(\VX i\get{\st'},\R(\st,\act))-\step}$, which is false.}
Therefore $\st'=\stX i$.

Subsequently, we have $\actX i\in\pref{\st'}{\VX i}$; otherwise there would be some $\act'\in\pref{\st'}{\VX i}$ with $\actX i\neq\act'$ and $\VX{i+1}(\st',\act')=\VX i(\st',\act')=\VX i\get{\st'}$, implying $\VX i\get{\st'}\leq\VX{i+1}\get{\st'}$, which is false.

Hence, $\VX i(\st',\actX i)=\VX i\get{\st'}$. Now, using the definition of $d$ above, the inequality $d<0$ implies
\[
    \VX i\get{\st'}=\VX i\get{\stX i} > \max(\VX i\get{\stX{i+1}},\R(\stX i,\actX i))-\step.
\]
Combined, 
\[
    \VX i(\st',\actX i) > \max(\VX i\get{\stX{i+1}},\R(\stX i,\actX i))-\step,
\]
which implies $(\st',\actX i)\in\viol{\VX i}$. Thus $\VX i(\st',\actX i)\leq\violmax{\VX i}$.

Lastly, we have $\clamp{\violmax{\VX i}-\step}>\clamp{\R(\st,\act)-\step}$. Otherwise, when considering $\clamp{\violmax{\VX i}-\step}\leq\clamp{\R(\st,\act)-\step}$, Equation~\eqref{eq:clamp-and-clamp} would imply the following contradiction, using $\VX i\get{\st'}=\VX i(\st',\actX i)\leq\violmax{\VX i}$ (from above):
\begin{align*}
  \clamp{\R(\st,\act)-\step} \leq \clamp{\max(\VX{i+1}\get{\st'},\R(\st,\act))-\step} 
                             &< \clamp{\max(\VX i\get{\st'},\R(\st,\act))-\step} \\
                             &=\clamp{\max(\VX i(\st',\actX i),\R(\st,\act))-\step} \\
                             &\leq\clamp{\max(\violmax{\VX i},\R(\st,\act))-\step}\\ 
                             &=\max(\clamp{\violmax{\VX i}-\step},\clamp{\R(\st,\act)-\step})\\
                             &\leq \clamp{\R(\st,\act)-\step}.
\end{align*}
So, if $\clamp{\violmax{\VX i}-\step}>\clamp{\R(\st,\act)-\step}$ then, using $(\st,\act)\notin\viol{\VX i}$, we observe
\begin{align*}
 \VX{i+1}(\st,\act) &\leq \VX i(\st,\act)\\
                    &\leq\clamp{\max(\VX i\get{\st'},\R(\st,\act))-\step} \\
                             &=\clamp{\max(\VX i(\st',\actX i),\R(\st,\act))-\step} \\
                    &\leq\clamp{\max(\violmax{\VX i},\R(\st,\act))-\step}\\ 
                    &\leq \clamp{\violmax{\VX i}-\step}.
\end{align*}
Moreover, $\clamp{\violmax{\VX i}-\step}\leq\violmax{\VX i}$.%
    \footnote{If $\violmax{\VX i}-\step<0$ then $ \clamp{\violmax{\VX i}-\step}=0\leq\violmax{\VX i}$. If $\violmax{\VX i}-\step\geq 0$ then $ \clamp{\violmax{\VX i}-\step}=\violmax{\VX i}-\step<\violmax{\VX i}$.}
Everything combined, we have $\VX{i+1}(\st,\act)\leq\violmax{\VX i}$, as desired.

\subsubsection{\propproof{prop:decrease-violmax}}\label{prop:decrease-violmax--PROOF}

Let $i$ be a configuration index, denoting the corresponding configuration as $\cnftupX i$, where $\violmax{\VX i}> 0$.
By Property~\ref{prop:violmax-top}, we know for all subsequent configuration indices $j$ with $j\geq i$ that $\violmax{\VX i}\geq\violmax{\VX j}$, i.e., the highest violation value never increases.

Towards a contradiction, suppose that $\violmax{\VX i}=\violmax{\VX j}$ for all $j\geq i$. Because there are only a finite number of configurations by Lemma~\ref{lem:finite-cnf}, there must be a configuration $(\alt{\st},\alt{\V})$ that occurs infinitely often, and with $\violmax{\alt{\V}}=\violmax{\VX i}$.
Since  $\violmax{\VX i}> 0$, we can consider a violation $(\st,\act)\in\viol{\alt\V}$ with $\alt\V(\st,\act)=\violmax{\alt\V}$. 
By Property~\ref{prop:inf-visit}, there are infinitely many transitions where we execute the pair $(\st,\act)$.

To continue with the proof, since  $(\alt{\st},\alt{\V})$ occurs infinitely often, and $(\st,\act)$ is infinitely often executed, after configuration $\cnftupX i$ we can consider a finite run-fragment $F$ of the following form:
\[
     (\alt{\st},\alt{\V}) \rightarrow \ldots \text{ in between execute $(\st,\act)$ at least once} \ldots \rightarrow (\alt{\st},\alt{\V}).
\]

In the fragment $F$, there must be a \emph{last} transition in which we execute $(\st,\act)$, denoted as
\[
  \transX j,
\]
where $(\stX j,\actX j)=(\st,\act)$.
Since this transition is the last transition of $(\st,\act)$ in fragment $F$, 
we must have $\VX{j+1}(\st,\act)=\alt\V(\st,\act)=\violmax{\alt\V}$. 

There are two cases: either $(\st,\act)\in\viol{\VX{j+1}}$ or $(\st,\act)\notin\viol{\VX{j+1}}$.
In each case, we derive a contradiction. Denote $\Det\st\act{\st'}$, and abbreviate $W=\violmax{\alt\V}$.
By assumption at the beginning of this proof, $W>0$.

\paragraph*{First case}
Suppose $(\st,\act)\in\viol{\VX{j+1}}$.
By Algorithm~\ref{alg:update}, we have
\[
    \VX{j+1}(\st,\act) = \clamp{\VX j(\st,\act) + d},
\]
where $d=\max(\VX{j}\get{\st'},\R(\st,\act)) - \step - \VX j\get{\st}$.
It must be $d< 0$. Otherwise, considering $d\geq 0$, since always $\VX j(\st,\act)\leq\VX j\get{\st}$, and additionally $\VX j\get{\st'}\leq\VX{j+1}\get{\st'}$ when $d\geq 0$, we would have
\begin{align*}
    \VX{j+1}(\st,\act)  &\leq \clamp{\VX j\get\st + d}\\
                        &= \clamp{\max(\VX{j}\get{\st'},\R(\st,\act))- \step}\\
                        &\leq \clamp{\max(\VX{j+1}\get{\st'},\R(\st,\act))- \step},                       
\end{align*}
and therefore $(\st,\act)\notin\viol{\VX{j+1}}$, which is false by assumption.

Note that $\VX j(\st,\act) > \VX{j+1}(\st,\act)$, since $d < 0$ and $\VX{j+1}(\st,\act)=W>0$.
Therefore $\VX j(\st,\act) > W$. We will show below that $(\st,\act)\in\viol{\VX j}$, giving 
\[
    \violmax{\VX j}\geq\VX j(\st,\act)> W = \violmax{\alt\V}=\violmax{\VX i},
\]
in particular, $\violmax{\VX j}>\violmax{\VX i}$, which contradicts Property~\ref{prop:violmax-top} (since $i\leq j$).

To show $(\st,\act)\in\viol{\VX j}$, we consider the following cases.
\begin{itemize}
    \item Suppose $\act\in\pref\st{\VX j}$.        
    Hence, $\VX j\get\st=\VX j(\st,\act)$, which we substitute into the equation of $\VX{j+1}(\st,\act)$ given by Algorithm~\ref{alg:update}:
    \begin{align*}
        \VX{j+1}(\st,\act)  &= \clamp{\VX j(\st,\act) + \max(\VX j\get{\st'}, \R(\st,\act)) - \step - \VX j(\st,\act)}\\
                            & = \clamp{\max(\VX j\get{\st'}, \R(\st,\act)) - \step}.
    \end{align*}    
    Combined with $\VX j(\st,\act)>\VX{j+1}(\st,\act)$ (see above), we obtain
    \[
        \VX j(\st,\act) > \clamp{\max(\VX j\get{\st'}, \R(\st,\act)) - \step},
    \]
    and therefore $(\st,\act)\in\viol{\VX j}$.

    \item Suppose $\act\notin\pref\st{\VX j}$ and $\st'=\st$.
    If $\act\notin\pref\st{\VX j}$ then there is some $\act'\in\pref\st{\VX j}$ with $\act\neq\act'$. Note that $\VX{j+1}(\st,\act')=\VX j(\st,\act')=\VX j\get\st$, implying $\VX j\get\st\leq\VX{j+1}\get\st$. Moreover, $d<0$ implies $\VX{j+1}\get{\st}\leq\VX j\get\st$. Overall, $\VX{j+1}\get\st = \VX j\get\st$.
    
    Next, since $(\st,\act)\in\viol{\VX{j+1}}$, we have, substituting $\VX{j+1}\get{\st'}=\VX{j+1}\get{\st}=\VX j\get\st$,
    \begin{align*}
        \VX{j+1}(s,a)   &> \clamp{\max(\VX{j+1}\get{\st'},\R(\st,\act))-\step} \\
                        &= \clamp{\max(\VX{j}\get{\st},\R(\st,\act))-\step}. 
    \end{align*}
    Combined with $\VX{j}(\st,\act)>\VX{j+1}(\st,\act)$ (see above), we obtain
    \[
        \VX{j}(s,a) > \clamp{\max(\VX{j}\get{\st},\R(\st,\act))-\step},
    \]
    and, recalling the assumption $\st'=\st$, therefore $(\st,\act)\in\viol{\VX j}$.
    
    \item Suppose $\act\notin\pref\st{\VX j}$ and $\st'\neq\st$.
    The latter implies $\VX{j+1}\get{\st'}=\VX j\get{\st'}$.
    Since $(\st,\act)\in\viol{\VX{j+1}}$, we have    
    \begin{align*}
        \VX{j+1}(s,a) &> \clamp{\max(\VX{j+1}\get{\st'},\R(\st,\act))-\step} \\
                      &=\clamp{\max(\VX j\get{\st'},\R(\st,\act))-\step}.
    \end{align*}
    Combined with $\VX{j}(\st,\act)>\VX{j+1}(\st,\act)$ (see above), we obtain 
    \[
        \VX j(\st,\act) > \clamp{\max(\VX j\get{\st'},\R(\st,\act))-\step},
    \]
    and therefore $(\st,\act)\in\viol{\VX j}$.
\end{itemize}

\paragraph*{Second case}
Suppose $(\st,\act)\notin\viol{\VX{j+1}}$. 
Hence,
\[
    \VX{j+1}(\st,\act) \leq \clamp{\max(\VX{j+1}\get{\st'},\R(\st,\act))-\step}.
\]

In order for $(\st,\act)\in\viol{\alt\V}$, which we assumed to be true, it is necessary that the value of $\st'$ is strictly decreased before the end of fragment $F$. Otherwise, for all configuration indices $k\geq j+1$ in fragment $F$, we would have $\VX{j+1}\get{\st'}\leq\VX k\get{\st'}$; and, combined with the assumption that transition $\transX j$ is the last transition of $F$ in which $(\st,\act)$ is updated, we obtain
\begin{align*}
    \VX k(\st,\act)
        &=\VX{j+1}(\st,\act)\\
        &\leq \clamp{\max(\VX{j+1}\get{\st'},\R(\st,\act))-\step}\\
        &\leq\clamp{\max(\VX k\get{\st'},\R(\st,\act))-\step},
\end{align*}
implying $(\st,\act)\notin\viol{\alt\V}$, which is false.

So, still inside fragment $F$, we can consider the \emph{first} transition after configuration $j+1$ where the value of $\st'$ is strictly decreased:
\[
    \transX{k},
\]
where $k\geq j+1$ and $\VX{k+1}\get{\st'}<\VX k\get{\st'}$. This implies $\stX k=\st'$.%
    \footnote{If $\stX k\neq\st'$ then always $\VX{k+1}\get{\st'}=\VX k\get{\st'}$.}
Denote $\Det{\st'}{\actX k}{\st''}$. Now, by Algorithm~\ref{alg:update}, 
\[
    \VX{k+1}(\st',\actX k) = \clamp{\VX k(\st',\actX k) +\max(\VX k\get{\st''},\R(\st',\actX k)) - \step - \VX k\get{\st'}}.
\]

Also, we have $\actX k\in\pref{\st'}{\VX k}$; otherwise there would be some action $\act'\in\pref{\st'}{\VX k}$ with $\act'\neq\actX k$ and $\VX k\get{\st'}=\VX{k}(\st',\act')=\VX{k+1}(\st',\act')$, implying $\VX k\get{\st'}\leq\VX{k+1}\get{\st'}$, which is false.
Since $\actX k\in\pref{\st'}{\VX k}$, we have $\VX k\get{\st'}=\VX k(\st',\actX k)$. This can be used to simplify the above equation for $\VX{k+1}(\st',\actX k)$, as follows:
\[
    \VX{k+1}(\st',\actX k) = \clamp{\max(\VX k\get{\st''},\R(\st',\actX k)) - \step}.
\]

Next, since $\VX k\get{\st'}> \VX{k+1}\get{\st'}$, we observe 
\[
    \VX k(\st',\actX k)=\VX k\get{\st'} > \VX{k+1}\get{\st'}\geq\VX{k+1}(\st',\actX k).
\]
In combination with the simplified equation for $\VX{k+1}(\st',\actX k)$, we obtain
\[
    \VX k(\st',\actX k) > \clamp{\max(\VX k\get{\st''},\R(\st',\actX k)) - \step}.
\]
Therefore, $(\st',\actX k)\in\viol{\VX k}$.

Now, Property~\ref{prop:lowerbound-w} (below) gives us $W < \VX{j+1}\get{\st'}$. Since $\transX k$ is the first transition after configuration $j+1$ with a value decrement on state $\st'$, we have $\VX{j+1}\get{\st'}\leq\VX k\get{\st'}$. Combined, $W < \VX k\get{\st'}$. Since $\VX k\get{\st'}=\VX k(\st',\actX k)$ (see above), we obtain $W < \VX k(\st',\actX k)$.

Overall, we obtain $\violmax{\VX k}> W=\violmax{\alt\V}=\violmax{\VX i}$, which contradicts Property~\ref{prop:violmax-top} (since $i\leq k$). This is the desired contradiction.

\begin{property}\label{prop:lowerbound-w}
    We have $W < \VX{j+1}\get{\st'}$.
\end{property}
\begin{proof}
    Below we show that $\VX{j+1}(\st,\act)\leq\VX{j+1}\get{\st'}-\step$. Therefore,
    \[
        W = \VX{j+1}(\st,\act) < \VX{j+1}(\st,\act)+\step\leq\VX{j+1}\get{\st'},
    \]
    giving $W<\VX{j+1}\get{\st'}$, as desired. 

    We are left to show $\VX{j+1}(\st,\act)\leq\VX{j+1}\get{\st'}-\step$.
    First, it must be $\VX{j+1}(\st,\act)>\clamp{\R(\st,\act)-\step}$. Otherwise, considering $\VX{j+1}(\st,\act)\leq\clamp{\R(\st,\act)-\step}$, since the transition $\transX{j}$ is the last transition in the run-fragment $F$ where $(\st,\act)$ is executed, for all configuration indices $k\geq j+1$ in $F$, we would have 
    \begin{align*}
        \VX k(\st,\act) 
            &= \VX{j+1}(\st,\act)\\
            &\leq \clamp{\R(\st,\act)-\step}\\
            &\leq \max(\clamp{\VX k\get{\st'}-\step},\clamp{\R(\st,\act)-\step})\\
            &=\clamp{\max(\VX k\get{\st'},\R(\st,\act))-\step}. 
    \end{align*}
    In particular, $\alt\V(\st,\act)\leq\clamp{\max(\alt\V\get{\st'},\R(\st,\act)) -\step}$, implying $(\st,\act)\notin\viol{\alt\V}$, which is false.

    Subsequently, $\VX{j+1}(\st,\act)>\clamp{\R(\st,\act)-\step}$ implies $\VX{j+1}(\st,\act)\leq\clamp{\VX{j+1}\get{\st'}-\step}$. Otherwise,
    \begin{align*}
        \VX{j+1}(\st,\act) &> \max(\clamp{\VX{j+1}\get{\st'}-\step},\clamp{\R(\st,\act)-\step}) \\
                           &=\clamp{\max(\VX{j+1}\get{\st'},\R(\st,\act))-\step},
    \end{align*}
    implying $(\st,\act)\in\viol{\VX{j+1}}$, which we assumed to be false.
    
    Lastly, since $\VX{j+1}(\st,\act)=W>0$, we know $\clamp{\VX{j+1}\get{\st'}-\step}>0$. Therefore, $\VX{j+1}\get{\st'}-\step>0$. Hence, we may write
    $\VX{j+1}(\st,\act)\leq\VX{j+1}\get{\st'}-\step$.
\end{proof}

\subsection{Proof of Property~\ref{prop:explore-valid-leq-opt}}
\label{prop:explore-valid-leq-opt--PROOF}

Let $\V$ be a valid value function. Towards a contradiction, suppose there is some $\st\in\states$ with $\V\get\st>\optval\st$. By Property~\ref{prop:valid-leq-opt-helper} (below), there is an action $\act\in\pref\st\V$, denoting $\Det\st\act{\st'}$, with $\V\get{\st'}>\optval{\st'}$ and $\V\get{\st}<\V\get{\st'}$. Property~\ref{prop:valid-leq-opt-helper} can subsequently be applied to $\st'$.
By repeatedly applying Property~\ref{prop:valid-leq-opt-helper}, we can establish an infinite sequence of the following form:
\[
    \stX 1\edge{\actX 1}\stX2\edge{\actX 2}\ldots,
\]
where
    $\V\get{\stX i}<\V\get{\stX{i+1}}$ for each $i\geq 1$.
But since there are a finite number of states, there must be two indices $j$ and $k$ with $j<k$ and $\stX j=\stX k$. Then $\V\get{\stX j}<\V\get{\stX k}$ is the desired contradiction.

\begin{property}\label{prop:valid-leq-opt-helper}
    Let $\V$ be the considered valid value function. Let $\st\in\states$. If $\V\get\st>\optval\st$ then $\exists\act\in\pref\st\V$, denoting $\Det\st\act{\st'}$, with 
    \begin{itemize}
        \item $\V\get{\st'}>\optval{\st'}$; and,
        \item $\V\get{\st}<\V\get{\st'}$.
    \end{itemize}
\end{property}
\begin{proof}
    Let $\act\in\pref\st\V$ be arbitrary, and denote $\Det\st\act{\st'}$.
    
    First we show that $\V\get{\st'}>\optval{\st'}$.
    Always,%
        \footnote{Note that $\optval\st\geq\clamp{\R(\st,\act)-\step}$ because $(\st,\act)$ is an action-path for $\st$. Also, $\optval\st\geq\clamp{\optval{\st'}-\step}$ because any optimal action-path for $\st'$ can be extended to an action-path for $\st$ by adding $(\st,\act)$ to the front. Formally, letting $p'$ be an action-path for $\st'$ with $\pval{p'}=\optval{\st'}$, and letting $p$ be the extension of $p'$ by adding $(\st,\act)$ to the front; Lemma~\ref{lem:pval} implies $\optval{\st}\geq\pval{p}=\max(\clamp{\R(\st,\act)-\step},\clamp{\optval{\st'}-\step})$.}
    \begin{align}
        \optval{\st} &\geq \max(\clamp{\optval{\st'}-\step},\clamp{\R(\st,\act)-\step})\label{eq:optval-geq-1}\\
                     &= \clamp{\max(\optval{\st'},\R(\st,\act))-\step}\label{eq:optval-geq-2}.
    \end{align}
    Towards a contradiction, suppose $\V\get{\st'}\leq\optval{\st'}$. Then, using all assumptions (including validity of $\V$), and the equality $\V\get{\st}=\V(\st,\act)$ (by $\act\in\pref\st\V$), we have
    \begin{align*}
        \optval\st &< \V\get\st\\
                   &=\V(\st,\act)\\
                   &\leq \clamp{\max(\V\get{\st'},\R(\st,\act))-\step}\\
                   &\leq \clamp{\max(\optval{\st'},\R(\st,\act))-\step},
    \end{align*}    
    which contradicts Equation~\eqref{eq:optval-geq-2}. Therefore $\V\get{\st'}>\optval{\st'}$.
    
    Now we show that $\V\get\st<\V\get{\st'}$.
    Since $\V\get\st>\optval\st$ by assumption, Equation~\eqref{eq:optval-geq-1} implies $\V\get\st > \clamp{\R(\st,\act)-\step}$. Together with $\V\get\st=\V(\st,\act)$ (by $\act\in\pref\st\V$) and validity, we have
    \begin{align}
        \clamp{\R(\st,\act)-\step} 
            &< \V\get\st \nonumber\\
            &=\V(\st,\act)\nonumber\\
            &\leq \clamp{\max(\V\get{\st'},\R(\st,\act))-\step} \nonumber\\
            & = \max(\clamp{\V\get{\st'}-\step},\clamp{\R(\st,\act)-\step})\label{eq:validity-implies}.
    \end{align}
    We have $\clamp{\V\get{\st'}-\step}>\clamp{\R(\st,\act)-\step}$ because otherwise Equation~\eqref{eq:validity-implies} would imply the contradiction $\clamp{\R(\st,\act)-\step}<\clamp{\R(\st,\act)-\step}$.
    Thus $\clamp{\V\get{\st'}-\step}>0$, causing $\V\get{\st'}-\step>0$, and therefore $\clamp{\V\get{\st'}-\step}=\V\get{\st'}-\step$. Validity now implies,
    \begin{align*}
        \V\get\st &= \V(\st,\act)\\
                &\leq \max(\clamp{\V\get{\st'}-\step},\clamp{\R(\st,\act)-\step})\\
                  &=\clamp{\V\get{\st'}-\step}\\
                  &=\V\get{\st'}-\step.
    \end{align*}
    Hence, $\V\get\st\leq\V\get{\st'}-\step$ and therefore $\V\get\st<\V\get{\st'}$ (using that $\step\geq 1$): $\V\get\st<\V\get\st+\step\leq\V\get{\st'}$.
\end{proof}

\subsection{Proof of Property~\ref{prop:explore-path-cases}}
\label{prop:explore-path-cases--PROOF}

    If $\length p=1$ then necessarily $\optval\st=\pval p=\clamp{\R(\st,\act)-\step}$.
    
    Henceforth, we assume $\length p \geq 2$. Let $p'$ denote the suffix of $p$ after omitting the first pair $(\st,\act)$.
    By Lemma~\ref{lem:pval},
    \[
        \pval p=\max(\clamp{\R(\st,\act)-\step},\clamp{\pval{p'}-\step}).    
    \]
    If $\clamp{\R(\st,\act)-\step}\geq\clamp{\pval{p'}-\step}$ then again $\optval\st = \clamp{\R(\st,\act)-\step}$.
    Henceforth we assume $\clamp{\R(\st,\act)-\step}<\clamp{\pval{p'}-\step}$. This implies $\clamp{\pval{p'}-\step}> 0$, causing $\pval{p'}-\step> 0$, so we write more simply
    \[
        \pval p=\pval{p'}-\step.
    \]
    We now show concretely that $\pval{p'}=\optval{\st'}$, giving, as desired
    \[
        \optval\st=\optval{\st'}-\step.
    \]
    We separately show $\pval{p'}\leq\optval{\st'}$ and $\optval{\st'}\leq\pval{p'}$.
    
    \textbf{Direction 1.} Since $p'$ is an action-path for $\st'$ we observe
    \[
        \pval{p'}\leq\optval{\st'}.
    \]
    
    \textbf{Direction 2.}
    Next, let $p''$ be an action-path for $\st'$ with $\pval{p''}=\optval{\st'}$. We can add the pair $(\st,\act)$ to the front of $p''$, resulting in a path $p'''$. By Lemma~\ref{lem:pval},
    \begin{align*}
        \pval{p'''} 
            &= \max(\clamp{\R(\st,\act)-\step},\clamp{\pval{p''}-\step})\\
            &\geq\clamp{\pval{p''}-\step}.
    \end{align*}
    Also, by definition of $\optval\st$, we have $\pval{p'''}\leq\optval\st=\pval p=\pval{p'}-\step$.
    Everything combined, we have
    \[
        \clamp{\pval{p''}-\step} \leq \pval{p'}-\step.
    \]
    We have $\pval{p''}\leq\pval{p'}$: otherwise, considering $\pval{p''}>\pval{p'}$, since $\pval{p'}-\step>0$ (see above), we would have $\pval{p''}-\step>0$; and subsequently $\clamp{\pval{p''}-\step}=\pval{p''}-\step>\pval{p'}-\step$, which is false.

    Now, $\pval{p''}\leq\pval{p'}$, combined with $\pval{p''}=\optval{\st'}$, implies the second direction that was sought:
    \[
        \optval{\st'}\leq\pval{p'}.
    \]    

\section{Proof of Lemma~\ref{lem:pval}}
\label{lem:pval--PROOF}

    By definition of path-value (Equation~\eqref{eq:pval}),
    \[
        \pval{p} = \max\set{\clamp{\R(\stX 1,\actX 1)-\step},\clamp{\R(\stX 2,\actX 2)-2\step},\ldots,\clamp{\R(\stX n,\actX n)-n\step}}.
    \]
    We may rewrite this as follows:
    \[
        \pval{p} = \max(\clamp{\R(\stX 1,\actX 1)-\step},m),
    \]
    where 
    \begin{align*}
        m 
            &= \max\set{\clamp{\R(\stX 2,\actX 2)-2\step},\ldots,\clamp{\R(\stX n,\actX n)-n\step}}\\
            & = \clamp{\max\set{\R(\stX 2,\actX 2)-2\step,\ldots,\R(\stX n,\actX n)-n\step}}.
    \end{align*}
    Subsequently,
    \begin{align*}
        m 
            &= \clamp{\max\set{\R(\stX 2,\actX 2)-\step,\ldots,\R(\stX n,\actX n)-(n-1)\step}-\step}.\\
            &= \clamp{\clamp{\max\set{\R(\stX 2,\actX 2)-\step,\ldots,\R(\stX n,\actX n)-(n-1)\step}}-\step}\\
            &= \clamp{\max\set{\clamp{\R(\stX 2,\actX 2)-\step},\ldots,\clamp{\R(\stX n,\actX n)-(n-1)\step}}-\step}\\
            &= \clamp{\pval{p'} - \step}.
    \end{align*}


\section{Proof details of Theorem~\ref{theo:greedy}}

\subsection{Auxiliary general properties}

Theorem~\ref{theo:greedy} assumes that all initial values are below $\M$. For a value function $\V$, we define
\[
    \highest\V = \max\set{\V(\st,\act)\mid(\st,\act)\in\states\times\actions}.
\]
The following property will be useful:
\begin{property}\label{prop:greedy-value-limit}
    For any run on the task, for any encountered value function $\V$, we have $\highest\V<\M$.
\end{property}
\begin{proof}
    We show the property by induction on the transitions of the run. By assumption, the property is true for the initial value function.
    Now, consider a transition
    \[
        (\st,\V)
            \jump{\act,\,\st'}
        (\st',\V').
    \]
    Assume $\highest\V < \M$. We show $\highest{\V'}<\M$. For each $(\st'',\act'')\in\states\times\actions$ with $(\st'',\act'')\neq(\st,\act)$ we have $\V'(\st'',\act'')=\V(\st'',\act'')\leq\highest\V<\M$. For the pair $(\st,\act)$ itself we have, by Algorithm~\ref{alg:update},
    \begin{align*}
        \V'(\st,\act) &= \clamp{\V(\st,\act) + \big(\max(\V\get{\st'},\R(\st,\act))- \step - \V\get\st\big)}\\
                      &\leq \clamp{\V\get\st + \big(\max(\V\get{\st'},\R(\st,\act))- \step - \V\get\st\big)}\\
                      &=\clamp{\max(\V\get{\st'},\R(\st,\act))- \step}.
    \end{align*}
    By subsequently using $\V\get{\st'}\leq\highest\V<\M$ and $\R(\st,\act)\leq\M$, we have
    \[
        \V'(\st,\act) \leq\clamp{\M - \step}.
    \]    
    Lastly, we use $\M-\step>0$ by Property~\ref{prop:greedy-m-step} (below), which implies $\clamp{\M-\step}=\M-\step$, to obtain
    \[
        \V'(\st,\act) \leq\M - \step<\M.
    \]
\end{proof}

\begin{property}\label{prop:greedy-m-step}
    We have $\M-\step>0$.
\end{property}
\begin{proof}
     Since $\ssize\states\geq 1$ we have $\M-\step\geq\M-\ssize\states\step$. And $\M -\ssize\states\step>0$ because the task is a navigation problem.
\end{proof}

\subsection{Auxiliary properties of strategies}

Let $\V$ be a value function.
For uniformity, we define $\stratlayer 0=\emptyset$.
We define $\fixp\V$ as the smallest index $n\in\nat$ for which $\stratlayer n=\stratlayer k$ for all $k\geq n$, i.e., $\fixp\V$ is the fixpoint index. Possibly $\fixp\V=0$, when no states can be added to the strategy.

\begin{property}\label{prop:strategy-atleast}
    Let $\V$ be a value function. For each $\st\in\strategy\V$, we have $\V\get\st\geq\M-\fixp\V\step$.
\end{property}
\begin{proof}
    Denote $n=\fixp\V$. Let $\st\in\strategy\V$. There is a smallest index $j\geq 1$ with $\st\in\stratlayer j$, implying $\V\get\st=\M-j\step$. 
    We have $j\leq n$: otherwise, considering $j>n$, we would have $\st\in\stratlayer j\setminus\stratlayer n$, which is not possible because $n=\fixp\V$. 
    
    Now, $j\leq n$ implies $\V\get\st\geq\M-n\step$.
\end{proof}

\begin{property}\label{prop:strategy-fixp}
    Let $\V$ be a value function. We have
    \[
        \fixp\V \leq \ssize{\strategy\V}.
    \]
\end{property}
\begin{proof}
    Abbreviate $n=\fixp\V$. If $n=0$ then the property is immediately true. Henceforth, suppose $n\geq 1$. We show by induction on $j=n,\ldots,1$ that
    \[
        \exists\st\in\stratlayer{j}\text{ with }\V\get\st=\M-j\step.
    \] 
    For any two states $\st$ and $\st'$, if $\V\get\st\neq\V\get{\st'}$ then $\st\neq\st'$; hence the inductive property implies $n\leq\ssize{\strategy\V}$, as desired.
    \begin{itemize}
        \item Base case: $j=n$. By choice of $n$ as the smallest index after which no more states are added to the strategy, we have $\stratlayer n\neq\stratlayer{n-1}$. Hence, $\stratlayer n$ extends $\stratlayer{n-1}$ with at least one state $\st$ satisfying $\V\get\st=\M-n\step$.
        
        \item Inductive step. Let $j\geq 2$, with the assumption that $\stratlayer j$ contains a state $\st$ with $\V\get\st=\M-j\step$.
        By definition of $\stratlayer{j}$, for each $\act\in\pref\st\V$ there must be some state $\st'\in\tr(\st,\act)\subseteq\stratlayer{j-1}$ with $\V\get{\st'}=\M-(j-1)\step$. Hence, there is at least one state $\st'\in\stratlayer{j-1}$ with $\V\get{\st'}=\M-(j-1)\step$.
    \end{itemize}
\end{proof}

\begin{property}\label{prop:strategy-leftover}
    Let $\V$ be a value function. If $\strategy\V\neq\states$ then for each $\st\in\strategy\V$ we have $\V\get\st>\step$.
\end{property}
\begin{proof}
    Denote $n=\fixp\V$. Since $\strategy\V\neq\states$, and yet always $\strategy\V\subseteq\states$, we have $\ssize{\strategy\V}<\ssize\states$. Combined with $n\leq\ssize{\strategy\V}$ (by Property~\ref{prop:strategy-fixp}), we see that
    \[
        n+1\leq\ssize\states.
    \]
    Since $\M-\ssize\states\step > 0$ by assumption on navigation problems, we obtain
    \[
        \M-(n+1)\step \geq \M-\ssize\states\step> 0,
    \]
    resulting in $\M-n\step>\step$.
    
    Now, let $\st\in\strategy\V$. Because $\V\get\st\geq\M-n\step$ by Property~\ref{prop:strategy-atleast}, we now observe, as desired,
    \[
        \V\get\st > \step.
    \]    
\end{proof}

\begin{property}\label{prop:strategy-bounds}
    Let $\V$ be a value function.
    For each $\st\in\strategy\V$ we have 
    \[
        0<\V\get\st\leq\M-\step.
    \]
\end{property}
\begin{proof}
    For the upper bound, we note that for each $\st\in\stratlayer 1$ we have $\V\get\st=\M-\step$, and for each $i\geq 2$, for each $\st\in\stratlayer i\setminus\stratlayer{i-1}$, we have $\V\get\st=\M-i\step<\M-\step$.

    For the lower bound, let $\st\in\strategy\V$. We first recall that $\fixp\V\leq\ssize{\strategy\V}$ by Property~\ref{prop:strategy-fixp}. Combined with $\ssize{\strategy\V}\leq\ssize\states$ (which is always true), we arrive at $\fixp\V\leq\ssize\states$.
    Now, combined with Property~\ref{prop:strategy-atleast}, and the assumption $\M-\ssize\states\step>0$ on navigation problems, we obtain, as desired,
    \[
        \V\get\st \geq \M-\fixp\V\step \geq \M-\ssize\states\step> 0.
    \]
\end{proof}

\begin{property}\label{prop:greedy-strategy-reduce}
    Let $\V$ be a value function. We have
    \[
        \strategy\V \subseteq \reduce\task.
    \]
\end{property}
\begin{proof}
    We show by induction on $j=1,2,\ldots$ that $\stratlayerX{\V}j\subseteq\reduce\task$.
    \begin{itemize}
        \item For the base case, we know for each $\st\in\stratlayerX\V1$ that $\exists\act\in\pref\st\V$ with $(\st,\act)\in\rewards\task$. Therefore $\stratlayerX\V1\subseteq\goals\task\subseteq\reduce\task$.
        
        \item For the inductive step, let $j\geq 1$ and assume $\stratlayerX\V j\subseteq\reduce\task$.
        We show that $\stratlayerX\V{j+1}\subseteq\reduce\task$.
        Suppose $\stratlayerX\V j\subsetneq\stratlayerX\V{j+1}$.
        Let $\st\in\stratlayerX\V{j+1}\setminus\stratlayerX\V j$.
        By definition of $\stratlayerX\V{j+1}$ we know $\exists\act\in\pref\st\V$ with $\tr(\st,\act)\subseteq\stratlayerX\V j$. By applying the induction hypothesis, we know $\tr(\st,\act)\subseteq\reduce\task$. Denoting $\reduce\task=\bigcup_{i=1}^\infty \reducelayer i$, we can consider an index $k$ with $\tr(\st,\act)\subseteq \reducelayer k$. Hence $\st\in \reducelayer{k+1}\subseteq\reduce\task$.
    \end{itemize}
\end{proof}

\subsection{Proof of Property~\ref{prop:good-reward}}
\label{prop:good-reward--PROOF}

    Consider the suffix $\run$ of a greedy run, 
    \[
        \cnftupX 1
            \jump{\actX 1,\stX 2}
        \cnftupX 2
            \jump{\actX 2,\stX 3}
        \ldots
    \]    
    where $\cnftupX 1$ is a good configuration.
    In the suffix, consider a finite sequence of transitions forming a state cycle, denoted as
     \[
        \cnftupX j
            \jump{\actX j,\,\stX{j+1}}
            \ldots
            \jump{\actX{j+(n-1)},\,\stX{j+n}}
        \cnftupX{j+n},
    \]    
    where $n\geq 1$ and $\stX{j+n}=\stX{j}$. Towards a contradiction, if none of the transitions between $\cnftupX j$ and $\cnftupX{j+n}$ has reward then Property~\ref{prop:good-flow} (below) tells us that, inside value function $\VX 1$, 
    \[
        \VX1\get{\stX j}<\ldots<\VX1\get{\stX{j+n}}=\VX1\get{\stX j},
    \]
    which is a contradiction. Hence, all state cycles in $\run$ contain reward.
    
\begin{property}\label{prop:good-flow}
 Consider the suffix of a greedy run, 
    \[
        \cnftupX 1
            \jump{\actX 1,\stX 2}
        \cnftupX 2
            \jump{\actX 2,\stX 3}
        \ldots
    \]    
    where $\cnftupX 1$ is a good configuration.
    For each $i\geq 1$, if $(\stX i,\actX i)\notin\rewards\task$ then $\VX 1\get{\stX i} < \VX 1\get{\stX{i+1}}$.
    Note the special role played by $\VX 1$.
    \proofapp{prop:good-flow--PROOF}
\end{property}

\subsubsection{Proof of Property~\ref{prop:good-flow}}\label{prop:good-flow--PROOF}

Regarding notation, for any two value functions $\V$ and $\V'$, we write $\V\cause\V'$ if for each $\st\in\strategy\V$ the following conditions are satisfied:
\begin{enumerate}
    \item $\forall\act\in\pref\st\V$:
    $
        \V\get\st \leq \V'(\st,\act) \leq \M-\step\text{; and,}
    $
    \item $\forall\act\in\actions\setminus\pref\st\V$:
    $
        \V'(\st,\act) < \V\get\st.
    $        
\end{enumerate}
We note that always $\V\cause\V$.%
        \footnote{Let $\st\in\strategy\V$. For any $\act\in\pref\st\V$, always $\V\get\st=\V(\st,\act)$; and, $\V\get\st\leq\M-\step$ by Property~\ref{prop:strategy-bounds}. For any $\act\in\actions\setminus\pref\st\V$, always $\V(\st,\act)<\V\get\st$.}
One may read $\V\cause\V'$ as \emph{$\V$ causes $\V'$}, because for the states in $\strategy\V$ the action-preference in $\V'$ is strongly related to the action-preference in $\V$; see also Property~\ref{prop:greedy-strategy-action} below.

\begin{property}\label{prop:greedy-strategy-action}
    Let $\V$ and $\V'$ be two value functions with $\V\cause\V'$. For each $\st\in\strategy\V$ we have $\pref\st{\V'}\subseteq\pref\st\V$.
\end{property}
\begin{proof}
    Let $\st\in\strategy{\V}$.
    Abbreviate $N=\actions\setminus\pref\st\V$.
    Below, we show $N\cap\pref\st{\V'}=\emptyset$. Then, 
    \begin{align*}
        \pref\st{\V'}&=\pref\st{\V'} \setminus N \\
                    &\subseteq \actions\setminus N \\
                    & =\pref\st\V.
    \end{align*}
    
    If $N=\emptyset$ then immediately $N\cap\pref\st{\V'}=\emptyset$. Henceforth, suppose $N\neq\emptyset$.
    Let $\actX 1\in N$ and $\actX 2\in\pref\st\V$.%
        \footnote{Note that always $\pref\st\V\neq\emptyset$.}
    Because $\V\cause\V'$, we have 
    \[
        \V'(\st,\actX 1)<\V\get\st\leq\V'(\st,\actX 2)\leq\V'\get{\st}.
    \]
    Hence, $\V'(\st,\actX 1)<\V'\get{\st}$, giving $\actX 1\notin\pref\st{\V'}$.  
\end{proof}

We continue with the proof of Property~\ref{prop:good-flow}. Abbreviate $\V=\VX 1$.
Consider a transition 
\[
    \transX{i},
\]
with $(\stX i,\actX i)\notin\rewards\task$.
By Property~\ref{prop:good-inside-strategy} (below), we know $\stX i\in\strategy{\V}$ and $\V\cause\VX i$. Subsequently, Property~\ref{prop:greedy-strategy-action} gives $\pref{\stX i}{\VX i}\subseteq\pref{\stX i}{\V}$.
Since $\actX i\in\pref{\stX i}{\VX i}$ by greediness of the run, we find $\actX i\in\pref{\stX i}{\V}$.

Next, because $\stX i\in\strategy\V$, we can consider the smallest index $j$ satisfying $\stX i\in\stratlayer j$, which implies $\V\get{\stX i}=\M-j\step$.
Since $\actX i\in\pref{\stX i}\V$ (see above) and $(\stX i,\actX i)\notin\rewards\task$, we have $j\geq 2$.
Then, by definition of $\stratlayer j$, we have $\tr(\stX i,\actX i)\subseteq\stratlayer{j-1}$.
In particular, since $\stX{i+1}\in\tr(\stX i,\actX i)$, we see $\stX{i+1}\in\stratlayer{j-1}$. Therefore $\V\get{\stX{i+1}}\geq\M-(j-1)\step>\M-j\step=\V\get{\stX i}$. Overall, 
$
    \V\get{\stX i}<\V\get{\stX{i+1}},
$
as desired.%
\footnote{For completeness, we note that not necessarily $\VX i\get{\stX i}<\VX i\get{\stX{i+1}}$.}
    
\begin{property}\label{prop:good-inside-strategy}
 Consider the suffix of a greedy run, 
    \[
        \cnftupX 1
            \jump{\actX 1,\stX 2}
        \cnftupX 2
            \jump{\actX 2,\stX 3}
        \ldots
    \]    
    where $\cnftupX 1$ is a good configuration.
    Abbreviating, $\V=\VX 1$, for each $i\geq 1$ we have
    \begin{enumerate}
        \item $\stX i\in\strategy{\V}$;
        \item $\V\cause\VX i$.
    \end{enumerate}
    \proofapp{prop:good-inside-strategy--PROOF}
\end{property}

\subsubsection{Proof of Property~\ref{prop:good-inside-strategy}}\label{prop:good-inside-strategy--PROOF}

    We show these properties by induction on $i=1,2,\ldots$.
    For the base case, $i=1$, we note the following:
    \begin{enumerate}
        \item We have $\stX 1\in\strategy{\VX 1}=\strategy\V$ since $\cnftupX 1$ is a good configuration.
        \item Always $\VX 1\cause\VX 1$.
    \end{enumerate}
    For the inductive step, consider a transition,
    \[
        \transX i,
    \]
    where $i\geq 1$. As induction hypothesis, we assume $\stX i\in\strategy\V$ and $\V\cause\VX i$. We show that the induction properties are also true for $\cnftupX{i+1}$.
    
    \paragraph*{First property}
    We show that $\stX{i+1}\in\strategy\V$.
    If $(\stX i,\actX i)\in\rewards\task$ then $\stX{i+1}\in\startstates$ since task $\task$ is restartable. Moreover, because $\startstates\subseteq\strategy\V$ by goodness of $\cnftupX 1$, we obtain $\stX{i+1}\in\strategy{\VX1}=\strategy\V$.
    
    Suppose $(\stX i,\actX i)\notin\rewards\task$. Since $\stX i\in\strategy\V$ by the induction hypothesis, we can consider the smallest index $j$ with $\stX i\in\stratlayer j$. 
    Also, $\V\cause\VX i$ by the induction hypothesis. Subsequently, $\pref{\stX i}{\VX i}\subseteq\pref{\stX i}{\V}$ by Property~\ref{prop:greedy-strategy-action}. Since $\actX i\in\pref{\stX i}{\VX i}$ by greediness of the run, we find $\actX i\in\pref{\stX i}\V$. 
    The assumption $(\stX i,\actX i)\notin\rewards\task$ now implies $j\geq2$. By definition of $\stratlayer j$ with $j\geq 2$, we know that $\tr(\stX i,\actX i)\subseteq\stratlayer{j-1}\subseteq\strategy\V$. In particular, $\stX{i+1}\in\strategy\V$.
    
    \paragraph*{Second property}
    We show that $\V\cause\VX{i+1}$.
    Let $\st\in\strategy\V$. As above, let $j$ be the smallest index for which $\st\in\stratlayer j$.
    
    \textbf{Preferred actions.} Let $\act\in\pref\st\V$. We have to show that
    \[
        \V\get\st \leq \VX{i+1}(\st,\act) \leq \M-\step.
    \]

    If $(\st,\act)\neq(\stX i,\actX i)$ then $\VX{i+1}(\st,\act)=\VX i(\st,\act)$, and the induction hypothesis $\V\cause\VX i$ implies $\V\get\st\leq\VX{i+1}(\st,\act)\leq \M-\step$.    
    Henceforth, suppose $(\st,\act)=(\stX i,\actX i)$. By Algorithm~\ref{alg:update},
        \[
            \VX{i+1}(\st,\act) = \clamp{\VX i(\st,\act) + \left(\max(\VX i\get{\stX{i+1}}, \R(\st,\act)) - \step - \VX i\get{\st}\right)}.
        \] 
    Since $\act=\actX i\in\pref\st{\VX i}$ by greedy action selection, we have $\VX i\get{\st}=\VX i(\st,\act)$, and the expression simplifies to
    \[
        \VX{i+1}(\st,\act) = \clamp{\max(\VX i\get{\stX{i+1}}, \R(\st,\act)) - \step}.
    \]
    Another general observation, is that $\VX i\get{\stX{i+1}}\leq\M-\step$: since $\stX{i+1}\in\strategy\V$ (see above), and $\V\cause\VX i$ (by the induction hypothesis), we have
    \begin{itemize}
        \item for each $\act'\in\pref{\stX{i+1}}\V$: $\V\get{\stX{i+1}}\leq\VX i(\stX{i+1},\act')\leq\M-\step$;
        \item for each $\act'\in\actions\setminus\pref{\stX{i+1}}\V$: $\VX i(\stX{i+1},\act')<\V\get{\stX{i+1}}\leq\M-\step$ (using Property~\ref{prop:strategy-bounds} for the upper bound).
    \end{itemize}
    
    Next, we distinguish between two cases, as follows.
    \begin{itemize}
        \item Suppose $(\st,\act)\in\rewards\task$.       
            From the proof of the first induction property above, we recall that $\act\in\pref\st\V$. 
            Now, $(\st,\act)\in\rewards\task$ implies $\st\in\stratlayer 1$. Therefore $\V\get\st=\M-\step$. Also, recall that $\M-\step>0$ (Property~\ref{prop:greedy-m-step}).
                
            By applying $\VX i\get{\stX{i+1}}\leq\M-\step$ (see above) and $\R(\st,\act)=\M$ (since $(\st,\act)\in\rewards\task$) to the equation for $\VX{i+1}(\st,\act)$, we obtain:%
                \footnote{Note in particular that $\max(\M-\step,\M)=\M$.}
            \begin{align*}
                \VX{i+1}(\st,\act) &= \clamp{\max(\VX i\get{\stX{i+1}}, \R(\st,\act)) - \step}\\                                  
                                   &= \clamp{\M- \step}\\
                                   &= \M-\step\\
                                   &=\V\get{\st}.
            \end{align*}
            Hence, $\V\get{\st}\leq\VX{i+1}(\st,\act)\leq\M-\step$.
        
        \item Suppose $(\st,\act)\notin\rewards\task$. 
        From the proof of the first induction property above, we recall that $\act\in\pref\st\V$.
        Now, $(\st,\act)\notin\rewards\task$ implies $\st\in\stratlayer j$ with $j\geq 2$.
        By applying $\R(\st,\act)=0$, and using $\VX i\get{\stX{i+1}}\geq 0$, the earlier equation of $\VX{i+1}(\st,\act)$ is simplified as follows:
        \[
            \VX{i+1}(\st,\act) = \clamp{\VX i\get{\stX{i+1}} - \step}.
        \]
        
        Before we continue, we show $\V\get{\stX{i+1}}\leq\VX i\get{\stX{i+1}}$.
        Since $\act\in\pref\st\V$, the definition of $\st\in\stratlayer j$ with $j\geq 2$ implies $\tr(\st,\act)\subseteq\stratlayer{j-1}$. In particular, $\stX{i+1}\in\stratlayer{j-1}\subseteq\strategy{\V}$, which implies $\V\get{\stX{i+1}}\geq\M-(j-1)\step$. 
        Moreover, combining $\stX{i+1}\in\strategy{\V}$ and $\V\cause\VX i$, and letting $\act'\in\pref{\stX{i+1}}\V$, we have 
        \[
            \V\get{\stX{i+1}}\leq\VX i(\stX{i+1},\act')\leq\VX i\get{\stX{i+1}}.
        \]
        
        Next, $\st\in\stratlayerX\V j$ gives $\V\get\st = \M-j\step$, and $j\geq 2$ further implies,
        \begin{align*}
            \V\get\st &= \M-(j-1)\step - \step\\
                    &\leq \V\get{\stX{i+1}} - \step\\
                    &\leq \VX i\get{\stX{i+1}} - \step.
        \end{align*}
        
        Lastly, by applying the deduced inequalities $\V\get\st\leq \VX i\get{\stX{i+1}}-\step$ and $\VX i\get{\stX{i+1}}\leq\M-\step$ (see earlier) to the last simplified equation for $\VX{i+1}(\st,\act)$, and using $\V\get\st\geq 0$, we obtain:
         \[
            \VX{i+1}(\st,\act) = \clamp{\VX i\get{\stX{i+1}} - \step} \geq \clamp{\V\get\st}=\V\get\st, 
        \]
        and, using $\M-\step>0$,
         \[
            \VX{i+1}(\st,\act) = \clamp{\VX i\get{\stX{i+1}} - \step} \leq \clamp{\M-\step-\step}\leq\clamp{\M-\step}=\M-\step.
        \]
    \end{itemize}
    
    \textbf{Non-preferred actions.}
    Let $\act\in\actions\setminus\pref\st\V$. We have to show that
    \[
        \VX{i+1}(\st,\act) < \V\get\st.
    \]
    Recall $\V\cause\VX i$ by the induction hypothesis.
    We distinguish between two cases, as follows:
    \begin{itemize}
        \item Suppose $\st\neq\stX i$. We have $\VX{i+1}(\st,\act) = \VX i(\st,\act)<\V\get\st$, where the inequality is given by $\V\cause\VX i$.
    
        \item Suppose $\st=\stX i$. Since $\stX i\in\strategy\V$, $\V\cause\VX i$, and $\actX i\in\pref{\stX i}{\VX i}$ (by greediness of the run), Property~\ref{prop:greedy-strategy-action} tells us $\actX i\in\pref{\stX i}\V=\pref\st\V$. Therefore, $\act\neq\actX i$, and again $\VX{i+1}(\st,\act) = \VX i(\st,\act)<\V\get\st$ (with the same reasoning as in the previous case).
    \end{itemize}    
    
\subsection{Auxiliary properties of function $\nextname$}

\begin{property}\label{prop:next-nochange}
    Consider a transition generated by function $\nextname$,
    \[
        \transX i.
    \]
    For each $(\st,\act)\in\strategy{\VX i}\times\actions$, we have $\VX{i+1}(\st,\act) = \VX i(\st,\act)$.
    In words: no changes occur to the value of state-action pairs where the state is in the strategy.
\end{property}
\begin{proof}        
    Let $(\st,\act)\in\strategy{\VX i}\times\actions$. If $(\st,\act)\neq(\stX i,\actX i)$ then the value could not have changed during the transition.
   
    Henceforth, suppose $(\st,\act)=(\stX i,\actX i)$. We show that $\VX{i+1}(\stX i,\actX i)=\VX i(\stX i,\actX i)$. 
    Since $\stX i\in\strategy{\VX i}$, we can consider the smallest index $j$ with $\stX i\in\stratlayerX{\VX i}j$. This implies $\VX i\get{\stX i}=\M-j\step$. Also, we have $\actX i\in\pref{\stX i}{\VX i}$ because function $\nextname$ only selects an action that the agent prefers. We distinguish between the following cases:
    \begin{itemize}
        \item Suppose $(\stX i,\actX i)\in\rewards\task$. Therefore $j=1$, and subsequently $\VX i\get{\stX i}=\M-\step$.
        Next, by Algorithm~\ref{alg:update}, we have
        \begin{align*}
            \VX{i+1}(\stX i,\actX i) &= \clamp{\VX i(\stX i,\actX i) + (\max(\VX i\get{\stX{i+1}}, \R(\stX i,\actX i)) - \step - \VX i\get{\stX i})}.
        \end{align*}
        We have $\VX i\get{\stX{i+1}} < \M$ by Property~\ref{prop:greedy-value-limit}, and $\R(\stX i,\actX i)=\M$ since $(\stX i,\actX i)\in\rewards\task$. Overall,
        \begin{align*}
            \VX{i+1}(\stX i,\actX i) &= \clamp{\VX i(\stX i,\actX i) + (\M - \step) - (\M - \step)} \\
                                    &= \clamp{\VX i(\stX i,\actX i)}\\
                                    &=\VX i(\stX i,\actX i),
        \end{align*}
        where the last step uses $\VX i(\stX i,\actX i)\geq 0$.
        
        \item Suppose $(\stX i,\actX i)\notin\rewards\task$. This implies $j\geq 2$.
        By definition of $\stratlayerX{\VX i}j$, we know 
        \begin{enumerate}
            \item $\tr(\stX i,\actX i)\subseteq\stratlayerX{\VX i}{j-1}$, which gives $\VX i\get{\st'}\geq\M-(j-1)\step$ for each $\st'\in\tr(\stX i,\actX i)$; and,
            \item $\exists\st'\in\tr(\stX i,\actX i)$ with $\VX i\get{\st'} = \M-(j-1)\step$.
        \end{enumerate}
        Therefore, $\expect{\VX i}{\stX i}{\actX i}=\M-(j-1)\step$.
        By subsequently using that $\VX i(\stX i,\actX i)=\VX i\get{\stX i}$ since $\actX i\in\pref{\stX i}{\VX i}$, and using $\VX i\get{\stX i}=\M-j\step$, we see
        \[
            \VX i(\stX i,\actX i) = \expect{\VX i}{\stX i}{\actX i} - \step.
        \]
        We can now see that in the specification of function $\nextname$, we have to exclude all cases except Case~\ref{next-nonzero-correct-reduce}:
        \begin{itemize}
            \item Case~\ref{next-zero} is not possible because $\VX i\get{\stX i}> 0$, as given by $\stX i\in\strategy{\VX i}$ and Property~\ref{prop:strategy-bounds}.
            
            \item Case~\ref{next-nonzero-incorrectActions} is not possible because $\tr(\stX i,\actX i)\subseteq\stratlayerX{\VX i}{j-1}\subseteq\strategy{\VX i}$.
            
            \item Case~\ref{next-nonzero-incorrectReward} is not possible because $(\stX i,\actX i)\notin\rewards\task$ by assumption.
            
            \item Case~\ref{next-nonzero-incorrectReduce} is not possible because it would demand $\VX i(\stX i,\actX i)\neq\expect{\VX i}{\stX i}{\actX i}-\step$, which is false, as we have shown above.
            
            \item Case~\ref{next-nonzero-correct-reward} is not possible because $(\stX i,\actX i)\notin\rewards\task$ by assumption.
        \end{itemize}
        
        Therefore, only Case~\ref{next-nonzero-correct-reduce} is possible.
        Importantly, Case~\ref{next-nonzero-correct-reduce} chooses $\stX{i+1}\in\tr(\stX i,\actX i)$ to satisfy $\VX i\get{\stX{i+1}} = \expect{\VX i}{\stX i}{\actX i}$. We also have $\R(\stX i,\actX i)=0$. Lastly, since $\actX i\in\pref{\stX i}{\VX i}$, we have $\VX i\get{\stX i}=\VX i(\stX i,\actX i)=\expect{\VX i}{\stX i}{\actX i}-\step$, as shown above. Subsequently, the equation for $\VX{i+1}(\stX i,\actX i)$, as given by Algorithm~\ref{alg:update}, can be simplified in the following manner:
         \begin{align*}       
            \VX{i+1}(\stX i,\actX i) &= \clamp{\VX i(\stX i,\actX i) + (\max(\VX i\get{\stX{i+1}}, \R(\stX i,\actX i)) - \step - \VX i\get{\stX i})}\\        
            &= \clamp{\VX i(\stX i,\actX i) + (\VX i\get{\stX{i+1}} - \step) - \VX i\get{\stX i})}\\    
             &= \clamp{\VX i(\stX i,\actX i) + (\expect{\VX i}{\stX i}{\actX i} - \step) - (\expect{\VX i}{\stX i}{\actX i}-\step)} \\
                    &= \clamp{\VX i(\stX i,\actX i)}\\
                    &=\VX i(\stX i,\actX i).
        \end{align*}
        where the last step uses $\VX i(\stX i,\actX i)\geq 0$.    
    \end{itemize}
\end{proof}

\subsection{Proof of Property~\ref{prop:nonreduce-restart}}\label{prop:nonreduce-restart--PROOF}

Let $\st$ be the fixed start state.
Consider a path 
\[
    \stX n\edge{\actX{n}}\stX{n-1}\ldots\stX 1\edge{\actX 1}\stX 0,
\]
where $n\geq 1$, $\set{\stX n,\ldots,\stX 1}\subseteq\nonreduce\task$, and $\stX 0=\st$.
We show by induction on $j=1,\ldots,n$ that $\stX j\in\dom{\Glayer{j}}$, which eventually implies $\stX n\in\dom{\Glayer n}\subseteq\dom{\G}$. Because the task is reducible, every non-reducible state has a path inside $\nonreduce\task$ towards $\st$. Hence, $\nonreduce\task\subseteq\dom{\G}$, as desired.

For the base case, we see $\stX 1\in\dom{\Glayer 1}$ because $\st=\stX 0\in\tr(\stX 1,\actX 1)$. For the inductive step, with $j\geq 2$ (and $j\leq n$), if not already $\stX j\in\dom{\Glayer{j-1}}$ then surely $\stX j\in\dom{\Glayer j}$ because $\stX{j-1}\in\dom{\Glayer{j-1}}$ (by the induction hypothesis) and $\stX{j-1}\in\tr(\stX j,\actX{j})$.    

\subsection{Proof of Property~\ref{prop:extend-strategy}}\label{prop:extend-strategy--PROOF}

Let $\run$ denote the infinite sequence of transitions obtained by repeatedly applying function $\nextname$ starting at $\cnftup$.
By Property~\ref{prop:next-preserve-strategy}, under $\nextname$, states are never removed from the strategy, i.e., the strategy could in principle only grow.
Towards a contradiction, suppose $\nextname$ is not able to eventually strictly extend the strategy, i.e., we have $\strategy{\V'}=\strategy\V$ for all encountered value functions $\V'$ after $\cnftup$.

By design, $\nextname$ jumps to a start state outside $\strategy\V$ during each reward transition (if possible). 
By Property~\ref{prop:next-inf-reward}, there are infinitely many reward transitions in $\run$, and since $\startstates\not\subseteq\strategy\V$, we arrive infinitely often outside $\strategy\V$.
There are two cases that could occur:
\begin{itemize}
    \item There are infinitely many transitions where reward is obtained at a state outside $\strategy\V$.
    
    \item There are infinitely many transitions where reward is obtained at a state inside $\strategy\V$. 
    This implies there are infinitely many transitions that jump from outside $\strategy\V$ to inside $\strategy\V$.  
\end{itemize}

Because there are a finite number of configurations (Lemma~\ref{lem:finite-cnf}), there are a finite number of possible transitions. Hence, in $\run$ there is either 
\begin{itemize}
    \item one particular transition, occurring infinitely often, where reward is obtained at a state outside $\strategy\V$; or,
    
    \item one particular transition, occurring infinitely often, that jumps from a state outside $\strategy\V$ to a state inside $\strategy\V$.
\end{itemize}    
We distinguish between the two cases.    

\paragraph*{Reward outside $\strategy\V$}
Consider a transition specified by $\nextname$,
\[
    \transX i,
\]
that occurs infinitely often in $\run$, and where $\stX{i}\notin\strategy\V$ and $(\stX i,\actX i)\in\rewards\task$.
We now analyze why function $\nextname$ has chosen $(\actX i,\stX{i+1})$, by looking at the specification of $\nextname$.
\begin{itemize}
    \item Case~\ref{next-zero-nonreduce} is not possible: $\stX{i}\in\nonreduce\task$ would contradict $(\stX i,\actX i)\in\rewards\task$.
    
    \item Case~\ref{next-zero-reduce-reward} is not possible: after the first execution of $(\stX i,\actX i)$, the value of $\stX i$ will be at least $\M-\step$, which is strictly larger than zero (Property~\ref{prop:greedy-m-step}); hence this case can not explain the infinite occurrences of the above transition.
    
    \item Case~\ref{next-zero-reduce-reduce} is not possible: $\stX i\notin\goals\task$ would contradict $(\stX i,\actX i)\in\rewards\task$.
    
    \item Case~\ref{next-nonzero-incorrectActions} is not possible; it would contradict $(\stX i,\actX i)\in\rewards\task$.
    
    \item Case~\ref{next-nonzero-incorrectReward} is not possible, as we now explain. We argue that after the first execution of $(\stX i,\actX i)$, the value of $(\stX i,\actX i)$ will remain $\M-\step$; hence Case~\ref{next-nonzero-incorrectReward} can not explain the infinite occurrences of the above transition.
    Consider a transition 
    \[
        \transX j,
    \]
    where $j\geq i$, and $(\stX j,\actX j)=(\stX i,\actX i)$. By Algorithm~\ref{alg:update},
    \[
        \VX{j+1}(\stX i,\actX i) = \clamp{\VX j(\stX i,\actX i) + \max(\VX j\get{\stX{j+1}},\R(\stX i,\actX i)) - \step - \VX j\get{\stX i}}.
    \]
    The equation can now be simplified as follows. Since $\actX i=\actX j\in\pref{\stX i}{\VX j}$ (since $\nextname$ always performs preferred actions), we have $\VX j(\stX i,\actX i)=\VX j\get{\stX i}$. Moreover, $\VX j\get{\stX{j+1}}<\M$ by Property~\ref{prop:greedy-value-limit}. Lastly, $\R(\stX i,\actX i)=\M$. We obtain the simplification,
    \[
        \VX{j+1}(\stX i,\actX i) = \clamp{\M-\step} = \M-\step,
    \]
    where we also use $\M-\step>0$ (Property~\ref{prop:greedy-m-step}).
    
    \item Case~\ref{next-nonzero-incorrectReduce} is not possible; it would contradict $(\stX i,\actX i)\in\rewards\task$.
    
    \item Case~\ref{next-nonzero-correct-reward} \underline{is possible}. Recall that $\stX i\notin\strategy\V$ by assumption. We now show $\stX i\in\strategy{\VX i}$, implying $\strategy\V\subsetneq\strategy{\VX i}$; this is the desired contradiction.
    
    Concretely, we show $\stX i\in\stratlayerX{\VX i}1$ (see Section~\ref{sub:greedy-strategy}).
    \begin{itemize}
        \item We show $\VX i\get{\stX i}=\M-\step$.
        Because Case~\ref{next-nonzero-incorrectReward} was not applicable, $(\stX i,\actX i)\in\rewards\task$ and $\actX i\in\pref{\stX i}{\VX i}$ (by design of $\nextname$) together imply $\VX i\get{\stX i}=\VX i(\stX i,\actX i)=\M-\step$.
    
        \item Let $\act\in\pref{\stX i}{\VX i}$. We show $(\stX i,\act)\in\rewards\task$. Towards a contradiction, if $(\stX i,\act)\notin\rewards\task$, since Case~\ref{next-nonzero-incorrectReduce} was not applicable, we know $\VX i(\stX i,\act)=\expect{\VX i}{\stX i}{\act}-\step$. But $\expect{\VX i}{\stX i}{\act}<\M$ by Property~\ref{prop:greedy-value-limit}, and thus $\VX i(\stX i,\act)<\M-\step$, resulting in $\actX i\notin\pref{\stX i}{\VX i}$, which is a contradiction.
    \end{itemize}
    
    \item Case~\ref{next-nonzero-correct-reduce} is not possible; it would contradict $(\stX i,\actX i)\in\rewards\task$.
\end{itemize}
    
\paragraph*{Arriving in $\strategy\V$}
Consider a transition
\[
    \transX i,
\]
that occurs infinitely often in $\run$, and where $\stX{i}\notin\strategy\V$ and $\stX{i+1}\in\strategy\V$. 
Note that $(\stX i,\actX i)\notin\rewards\task$: otherwise, we would jump (to a start state) outside $\strategy\V$.%
    \footnote{Also, the situation where $\stX i\notin\strategy\V$ and $(\stX i,\actX i)\in\rewards\task$ was already discussed earlier.}
We now analyze why function $\nextname$ has chosen $(\actX i,\stX{i+1})$, by looking at the specification of $\nextname$.
\begin{itemize}
    \item Case~\ref{next-zero-nonreduce} is not possible. We would either jump to (1) a start state outside $\strategy{\VX i}=\strategy\V$ (which is possible because $\startstates\nsubseteq\strategy\V$); or (2) a non-reducible state, which is also outside $\strategy\V$ by Property~\ref{prop:greedy-strategy-reduce}. Either option would be impossible because $\stX{i+1}\in\strategy\V$.
    
    \item Case~\ref{next-zero-reduce-reward} is not possible because $(\stX i,\actX i)\notin\rewards\task$.
    
    \item Case~\ref{next-zero-reduce-reduce} is not possible, as we now explain.
    The case implies $\VX i\get{\stX i}=0$. Therefore the value of $\stX i$ would have to be zero infinitely often in $\run$; we show this is not possible.
               
    We have $\tr(\stX i,\actX i)\subseteq\strategy{\VX i}=\strategy\V$: otherwise, Case~\ref{next-zero-reduce-reduce} would have chosen a successor state $\stX{i+1}$ outside $\strategy{\VX i}$, which is false.
    
    Subsequently, noting $\strategy\V\neq\states$ (as implied by $\startstates\not\subseteq\strategy\V$), we apply Property~\ref{prop:strategy-leftover} to know $\V\get{\st'} > \step$ for each $\st'\in\tr(\stX i,\actX i)$.
    Hence, by Property~\ref{prop:next-nochange-states} (below), $\VX i\get{\st'}=\V\get{\st'}> \step$ for each $\st'\in\tr(\stX i,\actX i)$.
    During the above transition, from the viewpoint of Algorithm~\ref{alg:update}, we would have (using $\stX{i+1}\in\tr(\stX i,\actX i)$ and $\VX i\get{\stX i}=0$):
    \begin{align*}
        d &= \max(\VX i\get{\stX{i+1}},\R(\stX i,\actX i)) - \step - \VX i\get{\stX i}\\
          &= \max(\VX i\get{\stX{i+1}},\R(\stX i,\actX i)) - \step\\
          & > \step - \step = 0.
    \end{align*}
    So, there is a strict value increase, making the value of $(\stX i,\actX i)$ (and thus the value of $\stX i$) nonzero after the transition. 
    
    The value of $(\stX i,\actX i)$ will remain nonzero after all subsequent executions of $(\stX i,\actX i)$. To see this, consider a transition
    \[
        \transX j,
    \]
    where $j\geq i$ and $(\stX j,\actX j)=(\stX i,\actX i)$. Again, by Property~\ref{prop:next-preserve-strategy} we have $\tr(\stX i,\actX i)\subseteq\strategy{\VX i}\subseteq\strategy{\VX j}=\strategy\V$. Also, we use Property~\ref{prop:strategy-leftover} and Property~\ref{prop:next-nochange-states} to know $\VX j\get{\st'} > \step$ for each $\st'\in\tr(\stX i,\actX i)$. Using Algorithm~\ref{alg:update}, where we substitute $\VX j(\stX i,\actX i)=\VX j\get{\stX i}$ (since $\nextname$ selects only preferred actions), we have
    \begin{align*}
        \VX{j+1}(\stX i,\actX i) 
            &= \clamp{\VX j(\stX i,\actX i) + \max(\VX j\get{\stX{j+1}},\R(\stX i,\actX i)) - \step - \VX j\get{\stX i}}\\
            &= \clamp{\max(\VX j\get{\stX{j+1}},\R(\stX i,\actX i)) - \step}.
    \end{align*}
    Since $\VX j\get{\stX{j+1}}>\step$, the right-hand side is strictly positive.
    
    \item Case~\ref{next-nonzero-incorrectActions} is not possible: the case would imply that $\stX{i+1}\notin\strategy{\VX i}=\strategy\V$, which is false.        
    
    \item Case~\ref{next-nonzero-incorrectReward} is not possible because $(\stX i,\actX i)\notin\rewards\task$.
    
    \item Case~\ref{next-nonzero-incorrectReduce} is not possible, as we now explain. Towards a contradiction, suppose the case were applicable. The specific contradiction will be that although the first application of Case~\ref{next-nonzero-incorrectReduce} leads to a change in the value of $(\stX i,\actX i)$, any subsequent applications of $(\stX i,\actX i)$ will keep the value fixed; so the above transition could occur only a finite number of times, which is false.
        
    To start, we note that the above transition changes the value of $(\stX i,\actX i)$. By Algorithm~\ref{alg:update}, we have:%
        \footnote{We substitute (1) $\VX i(\stX i,\actX i)=\VX i\get{\stX i}$ since $\actX i\in\pref{\stX i}{\VX i}$, and (2) $\VX i\get{\stX{i+1}}=\expect{\VX i}{\stX i}{\actX i}$ by design of Case~\ref{next-nonzero-incorrectReduce}.}
    \begin{align*}
        \VX{i+1}(\stX i,\actX i) 
            &= \clamp{\VX i(\stX i,\actX i) + \max(\VX i\get{\stX{i+1}},\R(\stX i,\actX i)) - \step - \VX i\get{\stX i}}\\            
            &= \clamp{\max(\expect{\VX i}{\stX i}{\actX i}, 0)-\step}\\
            &= \clamp{\expect{\VX i}{\stX i}{\actX i}-\step}\\
            &= \expect{\VX i}{\stX i}{\actX i} - \step.
    \end{align*}
    The last step uses $\expect{\VX i}{\stX i}{\actX i}-\step>0$ by Property~\ref{prop:strategy-leftover} (due to $\startstates\nsubseteq\strategy\V=\strategy{\VX i}$).
    Case~\ref{next-nonzero-incorrectReduce} implies $\VX i(\stX i,\actX i)\neq\expect{\VX i}{\stX i}{\actX i}-\step$. Hence, $\VX i(\stX i,\actX i)\neq\VX{i+1}(\stX i,\actX i)$.

    We show that henceforth the value of $(\stX i,\actX i)$ remains fixed.    
    Suppose we encounter a subsequent transition $t$
    \[
        \transX j,
    \]
    where $(\stX j,\actX j)=(\stX i,\actX i)$, and where still $\VX j(\stX i,\actX i)=\expect{\VX i}{\stX i}{\actX i}-\step$.
    We show concretely that $\VX{j+1}(\stX i,\actX i)=\VX j(\stX i,\actX i)$.
    We analyze why $\nextname$ has decided to perform transition $t$.
    \begin{itemize}    
        \item Case~\ref{next-zero} is not possible because $\VX j\get{\stX j}>0$: indeed, $\VX j\get{\stX j}=\VX j(\stX i,\actX i)$ because $\actX j\in\pref{\stX j}{\VX j}$, and $\VX j(\stX i,\actX i)=\expect{\VX i}{\stX i}{\actX i}-\step>0$ (as seen above).
        
        \item Case~\ref{next-nonzero-incorrectActions} is not possible: we have $\tr(\stX i,\actX i)\subseteq\strategy{\VX i}$ because we are working in Case~\ref{next-nonzero-incorrectReduce}, and $\strategy{\VX i}\subseteq\strategy{\VX j}$ (by Property~\ref{prop:next-preserve-strategy}). 
        
        \item Case~\ref{next-nonzero-incorrectReward} is not possible since $(\stX i,\actX i)\notin\rewards\task$.
        
        \item Importantly, Case~\ref{next-nonzero-incorrectReduce} is also not possible, as we now explain. Recall that $\tr(\stX i,\actX i)\subseteq\strategy{\VX i}$ because we are working in Case~\ref{next-nonzero-incorrectReduce}. Then Property~\ref{prop:next-nochange-states} (below) gives us $\VX j\get{\st'}=\VX i\get{\st'}$ for each $\st'\in\tr(\stX i,\actX i)$. Therefore $\expect{\VX j}{\stX i}{\actX i}=\allowbreak\expect{\VX i}{\stX i}{\actX i}$, and thus $\VX j(\stX i,\actX i)=\allowbreak\expect{\VX j}{\stX i}{\actX i}-\step$.
        
        \item Case~\ref{next-nonzero-correct-reward} is not possible since $(\stX i,\actX i)\notin\rewards\task$.
        
        \item Only Case~\ref{next-nonzero-correct-reduce} is possible. We show that $\VX{j+1}(\stX i,\actX i)=\VX j(\stX i,\actX i)$. Inside Algorithm~\ref{alg:update}, we have
        \begin{align*}
            d &= \max(\VX j\get{\stX{j+1}},\R(\stX i,\actX i)) - \step - \VX j\get{\stX j}.
        \end{align*}
        It suffices to show $d=0$. In the equation for $d$, the following substitutions can be done: 
        \begin{enumerate}
            \item $\VX j\get{\stX j}=\VX j(\stX i,\actX i)=\expect{\VX i}{\stX i}{\actX i}-\step$ since $\actX j=\actX i\in\pref{\stX j}{\VX j}$; 
            
            \item $\VX j\get{\stX{j+1}}=\expect{\VX j}{\stX i}{\actX i}=\expect{\VX i}{\stX i}{\actX i}$ by Case~\ref{next-nonzero-correct-reduce};%
                \footnote{The equality $\expect{\VX j}{\stX i}{\actX i}=\expect{\VX i}{\stX i}{\actX i}$ can be seen with the same reasoning as in the discussion of Case~\ref{next-nonzero-incorrectReduce} just above.}
            
            \item $\R(\stX i,\actX i)=0$. 
        \end{enumerate}
        Hence,
        \begin{align*}
            d &= \expect{\VX i}{\stX i}{\actX i} - \step - (\expect{\VX i}{\stX i}{\actX i} - \step)\\
                &= 0.
        \end{align*}
    \end{itemize}
    
    \item Case~\ref{next-nonzero-correct-reward} is not possible because $(\stX i,\actX i)\notin\rewards\task$.
    
    \item Case~\ref{next-nonzero-correct-reduce} \underline{is possible}. We show $\stX i\in\strategy{\VX i}$, which, combined with $\stX i\notin\strategy\V$, gives the desired contradiction.    
    Concretely, we show $\stX i\in\stratlayerX{\VX i}{j+1}$ for some $j\geq 1$, which implies $\stX i\in\strategy{\VX i}$.
    
    We first show $\VX i\get{\stX i}=\M-(j+1)\step$ for some $j\geq 1$.
    Recall that $(\stX i,\actX i)\notin\rewards\task$. We have $\tr(\stX i,\actX i)\subseteq\strategy{\VX i}$ because Case~\ref{next-nonzero-incorrectActions} was not applicable.
    Since Case~\ref{next-nonzero-incorrectReduce} was not applicable, we know
    \[
        \VX i(\stX i,\actX i) = \expect{\VX i}{\stX i}{\actX i} - \step. 
    \]
    Let $\st'\in\tr(\stX i,\actX i)$ with $\VX i\get{\st'}=\expect{\VX i}{\stX i}{\actX i}$. Since $\st'\in\strategy{\VX i}$, we can consider the smallest index $j$ satisfying $\st'\in\stratlayerX{\VX i}j$; note that $j\geq 1$. Hence, $\VX i\get{\st'}=\M-j\step$.        
    We note the following, where we start with $\VX i\get{\stX i}=\VX i(\stX i,\actX i)$ since $\nextname$ only chooses preferred actions:
        \begin{align*}
            \VX i\get{\stX i}&=\VX i(\stX i,\actX i)\\
                        &=\expect{\VX i}{\stX i}{\actX i}-\step \\
                        &=\VX i\get{\st'}-\step\\
                        &=\M-j\step -\step\\
                        &= \M-(j+1)\step.
        \end{align*}

    Next, we show that the actions in $\pref{\stX i}{\VX i}$ satisfy the desired properties, in the definition of $\stratlayerX{\VX i}{j+1}$ in Section~\ref{sub:greedy-strategy}. Let $\act\in\pref{\stX i}{\VX i}$.
    \begin{enumerate}
        \item Since $\VX i\get{\stX i}=\M-(j+1)\step$ and $j\geq 1$, we have $\VX i\get{\stX i}<\M-\step$. We have $(\stX i,\act)\notin\rewards\task$: otherwise, because Case~\ref{next-nonzero-incorrectReward} was not applicable (as mentioned above), we would have $\VX i\get{\stX i}\geq\VX i(\stX i,\act)=\M-\step$, which is false.
        
        \item 
        Because Case~\ref{next-nonzero-incorrectActions} was not applicable, and $(\stX i,\act)\notin\rewards\task$ (see the previous item), we know $\tr(\stX i,\act)\subseteq\strategy{\VX i}$.
        
        Since $\act\in\pref{\stX i}{\VX i}$, we have
        \[
            \VX i(\stX i,\act) = \VX i\get{\stX i}  =\M-(j+1)\step.
        \]
        Moreover, since Case~\ref{next-nonzero-incorrectReduce} was not applicable, we know 
        \[
             \VX i(\stX i,\act) = \expect{\VX i}{\stX i}{\act} - \step.
        \]
        By combining the above two expressions for $\VX i(\stX i,\act)$, we know 
        \[
            \expect{\VX i}{\stX i}{\act} = \M-j\step.
        \]
        Since $\tr(\stX i,\act)\subseteq\strategy{\VX i}$, we therefore know $\tr(\stX i,\act)\subseteq\stratlayerX{\VX i}j$.%
            \footnote{Suppose there is some $\st''\in\tr(\stX i,\act)\setminus\stratlayerX{\VX i}j$. Then there is some smallest index $k$ with $\st''\in\stratlayerX{\VX i}k$ where $k>j$. Then $\VX i\get{\st''}=\M-k\step<\M-j\step$, implying $\expect{\VX i}{\stX i}{\act}<\M-j\step$, which is false.}
            
        \item Since $\expect{\VX i}{\stX i}{\act} = \M-j\step$ (see previous item), there must be some state $\st''\in\tr(\stX i,\act)$ with $\VX i\get{\st''}=\M-j\step$.
    \end{enumerate}
\end{itemize}

\begin{property}\label{prop:next-nochange-states}
    Function $\nextname$ preserves the value of strategy states. More formally, consider a sequence of transitions generated by function $\nextname$,
    \[        
    \cnftupX 1
        \jump{\actX 1,\,\stX 2}   
        \ldots   
        \jump{\actX{n-1},\,\stX n}
    \cnftupX n.
    \]
    For each $\st\in\strategy{\VX 1}$ we have $\VX 1\get\st = \VX n\get\st$.    
\end{property}
\begin{proof}    
    For each  transition 
    \[
        \transX i,
    \] 
    with $i\in\set{1,\ldots,n-1}$, we know the following:
    \begin{itemize}
        \item By Property~\ref{prop:next-preserve-strategy}, we know $\strategy{\VX i}\subseteq\strategy{\VX{i+1}}$.
    
        \item By Property~\ref{prop:next-nochange}, we  know $\VX{i+1}(\st,\act)=\VX i(\st,\act)$ for each $(\st,\act)\in\strategy{\VX i}\times\actions$.
    \end{itemize}
    Now, we fix some $\st\in\strategy{\VX 1}$. We have
    \begin{itemize}
        \item $\st\in\strategy{\VX 1}\subseteq\strategy{\VX2}$ and $\VX 1\get\st = \VX 2\get\st$;
        \item $\st\in\strategy{\VX 2}\subseteq\strategy{\VX 3}$ and $\VX 2\get\st = \VX 3\get\st$;
        \item $\ldots$
    \end{itemize}
    By transitivity, $\VX 1\get\st = \VX n\get\st$, as desired.
\end{proof}

\subsection{Proof of Property~\ref{prop:next-preserve-strategy}}
\label{prop:next-preserve-strategy--PROOF}

Consider a transition generated by $\nextname$,
\[
    \transX i.
\]
By Property~\ref{prop:next-nochange} we know the following: $\forall\st\in\strategy{\VX i}$, $\forall\act\in\actions$, 
\[
    \VX{i+1}(\st,\act) = \VX i(\st,\act).
\]
This implies for all $\st\in\strategy{\VX i}$ that 
\begin{enumerate}
    \item $\VX{i+1}\get\st = \VX i\get\st$; and,
    \item $\pref\st{\VX{i+1}}=\pref\st{\VX i}$.
\end{enumerate}
We now show by induction on $j=1,2,\ldots$ that $\stratlayerX{\VX i}j\subseteq\stratlayerX{\VX{i+1}}j$, resulting in $\strategy{\VX i}\subseteq\strategy{\VX{i+1}}$, as desired. 
\begin{itemize}
    \item For the base case, let $\st\in\stratlayerX{\VX i}1$. This implies (1) $\VX{i+1}\get\st = \VX i\get\st = \M-\step$, and (2) for each $\act\in\pref\st{\VX{i+1}}=\pref\st{\VX i}$ that $(\st,\act)\in\rewards\task$. Hence, $\st\in\stratlayerX{\VX{i+1}}1$.
    
    \item Let $j\geq 2$. The induction hypothesis is $\stratlayerX{\VX i}{j-1}\subseteq\stratlayerX{\VX{i+1}}{j-1}$. For the inductive step, let $\st\in\stratlayerX{\VX i}j\setminus\stratlayerX{\VX i}{j-1}$. We have (1) $\VX{i+1}\get\st = \VX i\get\st = \M-j\step$ and (2) for all $\act\in\pref\st{\VX{i+1}}=\pref\st{\VX i}$,
    \begin{enumerate}
        \item $(\st,\act)\notin\rewards\task$;
        
        \item $\tr(\st,\act) \subseteq \stratlayerX{\VX i}{j-1}$, which, combined with the induction hypothesis $\stratlayerX{\VX i}{j-1}\subseteq\stratlayerX{\VX{i+1}}{j-1}$, gives $\tr(\st,\act)\subseteq\stratlayerX{\VX{i+1}}{j-1}$;
        
        \item $\exists\st'\in\tr(\st,\act)$ with $\VX{i+1}\get{\st'}=\VX i\get{\st'}=\M-(j-1)\step$.%
            \footnote{Here we use $\st'\in\stratlayerX{\VX i}{j-1}\subseteq\strategy{\VX i}$, which implies $\VX{i+1}\get{\st'}=\VX{i}\get{\st'}$.}
    \end{enumerate}
    Overall, we see $\st\in\stratlayerX{\VX{i+1}}j$.
\end{itemize}

\subsection{Proof of Property~\ref{prop:next-inf-reward}}
\label{prop:next-inf-reward--PROOF}

Suppose we perform $\nextname$ infinitely often, starting at some arbitrary configuration. Towards a contradiction, suppose that after a while we no longer encounter transitions with reward.

Because there are only a finite number of configurations in any infinite transition sequence (Lemma~\ref{lem:finite-cnf}), we encounter a configuration-cycle $\cycle$,
\[
    \cnftupX 1
        \jump{\actX 1,\,\stX 2}
        \ldots
        \jump{\actX{n-1},\,\stX n}
    \cnftupX n,
\]
where $n\geq 2$, and $\cnftupX 1=\cnftupX n$, and where none of the transitions contains reward.

Property~\ref{prop:next-conf-cycle} (below) tells us that $\VX n\get{\stX n}=0$. Now, because $\nextname$ is deterministic, and $\cnftupX 1=\cnftupX n$, the cycle $\cycle$ gives rise to another cycle $\cycle'$ that is shifted one transition into the future:
\[
    \cnftupX 2
        \jump{\actX 2,\,\stX 3}
        \ldots
        \jump{\actX{n-1},\,\stX n}
    \cnftupX n
        \jump{\actX n,\,\stX{n+1}}
    \cnftupX{n+1},        
\]    
where $\cnftupX 2$ is the second configuration in $\cycle$, $(\actX n,\stX{n+1})=(\actX 1,\stX 2)$, and $\cnftupX 2=\cnftupX{n+1}$ (because $\cnftupX 1=\cnftupX n$). All state-action pairs that are executed in cycle $\cycle'$ are also executed in cycle $\cycle$; hence, $\cycle'$ contains no reward either.
Property~\ref{prop:next-conf-cycle} again gives $\VX{n+1}\get{\stX{n+1}}=0$.
The reasoning can now be repeated arbitrarily many times, to establish an infinite (and contiguous) sequence of configurations in which the current state has zero value.%
    \footnote{If $n=2$ then all configurations after $\cnftupX 1$ are also $\cnftupX 1$; the current state therefore has value zero forever.}
This means that all transitions are specifically generated by Case~\ref{next-zero} of function $\nextname$. Intuitively, we have designed $\nextname$ in such a way that if the current state has value zero then we move the agent towards reward. We look at the sub-cases of Case~\ref{next-zero}:
\begin{itemize}
    \item Case~\ref{next-zero-nonreduce}: in that case we follow the acyclic movement strategy defined for the non-reducible states. Eventually we encounter a start state, which is reducible by assumption. We therefore must eventually arrive at one of the following two cases.
    
    \item Case~\ref{next-zero-reduce-reward}. In that case we obtain reward, which would be a contradiction.
    
    \item Case~\ref{next-zero-reduce-reduce}. In that case we move strictly deeper into the reducibility layers. However, this process can not continue forever because there are only a finite number of states. We must eventually arrive at Case~\ref{next-zero-reduce-reward} and obtain reward (again, a contradiction).
\end{itemize}
In this case analysis, we have therefore arrived at the desired contradiction.

\begin{property}\label{prop:next-conf-cycle}
    Suppose we have configuration-cycle $\cycle$ under function $\nextname$,
    \[
        \cnftupX 1
            \jump{\actX 1,\,\stX 2}
            \ldots
            \jump{\actX{n-1},\,\stX n}
        \cnftupX n,
    \]
    where $\cnftupX 1=\cnftupX n$. If none of the transitions contains reward then $\VX n\get{\stX n}=0$.
    Intuitively, this means that value can not be sustained in absence of reward.
\end{property}
\begin{proof}\newcommand{\cystates}{P} 
    Assume that none of the transitions contains reward.
    Abbreviate $\cystates=\set{\stX 1,\ldots,\stX{n-1}}$.
    We consider the highest value in $\VX 1$ among the states in $\cystates$:
    \[
        W = \max\set{\VX 1\get{\st}\mid\st\in\cystates}.
    \]
    We show $W = 0$. This implies in particular that $\VX 1\get{\stX 1}=\VX n\get{\stX n}=0$, as desired.
    
    Assume for now that we already know the following:
    \begin{claim}\label{claim:w}    
        For each $i\in\set{1,\ldots,n-1}$, we have $\VX{i+1}(\stX i,\actX i)< W$. 
    \end{claim}
    
    Now, let $i\in\set{1,\ldots,n-1}$ be the smallest index for which $\VX 1\get{\stX i}=W$, i.e., index $i$ is the first index at which we encounter a state with value $W$ in $\VX 1$.    
    Because $\stX i$ could not have been encountered before in the cycle (by choice of $i$), we have $\VX i\get{\stX i}=\VX 1\get{\stX i}=W$. 
    Moreover, because function $\nextname$ always chooses a preferred action, we have $\actX i\in\pref{\stX i}{\VX i}$, which implies $\VX i(\stX i,\actX i)=\VX i\get{\stX i}$.
    Hence, $\VX i(\stX i,\actX i)=W$. 
    Again, because $\stX i$ could not have been visited before (by choice of $i$), we have $\VX 1(\stX i,\actX i)=W$.
    By Claim~\ref{claim:w}, we know $\VX{i+1}(\stX i,\actX i)< W$. All subsequent transitions for the pair $(\stX i,\actX i)$ result in a value strictly smaller than $W$.
    This results in $\VX n(\stX i,\actX i)<W$. But since $\VX 1(\stX i,\actX i)=W$, we would have $\VX n\neq\VX 1$; this is the desired contradiction.
    
    \medskip
    \textbf{Proof of Claim~\ref{claim:w}.}
    We first consider the following sub-claim:
    \begin{claim}\label{claim:w-sub}
        Let $i\in\set{1,\ldots,n-1}$. Assume for each $\st\in\cystates$ that $\VX i\get\st\leq W$. Then $\VX{i+1}(\stX i,\actX i) < W$.
    \end{claim}    
    To finish the proof of Claim~\ref{claim:w}, we show by induction on $i\in\set{1,\ldots,n-1}$ that,  for each $\st\in\cystates$, we have $\VX i\get\st\leq W$. For the base case, for each $\st\in\cystates$, we have $\VX 1\get{\st}\leq W$ by definition of $W$.
    For the inductive step, let $i\geq 1$, with the assumption $\VX i\get\st\leq W$ for each $\st\in\cystates$. Letting $\st\in\cystates$, we observe that the desired property is satisfied for $\VX{i+1}$:
    \begin{itemize}
        \item If $\st\neq\stX i$ then $\VX{i+1}\get{\st}=\VX i\get{\st}\leq W$ by the induction hypothesis.
        \item Suppose $\st=\stX i$. For any $\act\in\actions$ if $\act\neq\actX i$ then $\VX{i+1}(\st,\act)=\VX i(\st,\act)\leq W$ by the induction hypothesis; if $\act=\actX i$ then $\VX{i+1}(\st,\act) < W$ by Claim~\ref{claim:w-sub} (using also the induction hypothesis).
    \end{itemize}
    To show Claim~\ref{claim:w-sub}, let $i\in\set{1,\ldots,n-1}$. By Algorithm~\ref{alg:update}, we have
    \[
        \VX{i+1}(\stX i,\actX i) = \clamp{\VX i(\stX i,\actX i) + \max(\VX i\get{\stX{i+1}},\R(\stX i,\actX i)) - \step - \VX i\get{\stX i}}.
    \]
    We have $\VX i(\stX i,\actX i)=\VX i\get{\stX i}$ since function $\nextname$ always chooses an action that is preferred by the agent. Also, we have $\R(\stX i,\actX i)=0$ because none of the transitions contains reward by assumption. Moreover, always $\VX i\get{\stX{i+1}}\geq 0$. The equation can now be simplified as follows:
    \[
        \VX{i+1}(\stX i,\actX i) = \clamp{\VX i\get{\stX{i+1}} - \step}.
    \]
    Since $\stX{i+1}\in\cystates$, and $\VX i\get{\st}\leq W$ for each $\st\in\cystates$ (by the assumption in Claim~\ref{claim:w-sub}), we obtain
    \begin{align*}
        \VX{i+1}(\stX i,\actX i) &\leq \clamp{W - \step}\\
                                 &< W.
    \end{align*}
    In the last step, we use $W - \step\leq W$, which implies $\clamp{W-\step}\leq\clamp{W}=W$ (using that $W\geq 0$).
\end{proof}

\end{document}